\documentclass{article}

\usepackage{microtype}
\usepackage{graphicx}
\usepackage{subfigure}
\usepackage{booktabs} 

\usepackage{hyperref}

\usepackage[accepted]{icml2024}

\usepackage{amsmath}
\usepackage{amssymb}
\usepackage{mathtools}
\usepackage{amsthm}

\usepackage[capitalize,noabbrev]{cleveref}

\theoremstyle{plain}
\newtheorem{theorem}{Theorem}[section]
\newtheorem{proposition}[theorem]{Proposition}
\newtheorem{lemma}[theorem]{Lemma}
\newtheorem{corollary}[theorem]{Corollary}
\newtheorem{condition}[theorem]{Condition}
\theoremstyle{definition}
\newtheorem{definition}[theorem]{Definition}

\theoremstyle{remark}
\newtheorem{remark}[theorem]{Remark}
\usepackage[textsize=tiny]{todonotes}

\usepackage{bm}
\usepackage{mathrsfs}
\usepackage{mathbbol}
\usepackage{amssymb}
\usepackage{natbib}
\usepackage{float}
\usepackage{mdframed}
\newenvironment{Msg}[1]
  {\mdfsetup{
    frametitle={\colorbox{white}{\space \large #1\space}},
    innertopmargin=-3pt,
    innerbottommargin=7pt,
    innerrightmargin=7pt,
    innerleftmargin=7pt,
    frametitleaboveskip=-\ht\strutbox,
    frametitlealignment=\center,
    linewidth=1pt
    }
  \begin{mdframed}%
  }
{\end{mdframed}}

\icmltitlerunning{Provably Neural Active Learning Succeeds via Prioritizing Perplexing Samples}

\begin{document}
\twocolumn[
\icmltitle{Provably Neural Active Learning Succeeds via Prioritizing Perplexing Samples}



\begin{icmlauthorlist}
\icmlauthor{Dake Bu}{cityu}
\icmlauthor{Wei Huang*}{riken}
\icmlauthor{Taiji Suzuki}{utokyo,riken}
\icmlauthor{Ji Cheng}{cityu}
\icmlauthor{Qingfu Zhang}{cityu}
\icmlauthor{Zhiqiang Xu}{mbzuai}
\icmlauthor{Hau-San Wong*}{cityu}
\end{icmlauthorlist}

\icmlaffiliation{cityu}{Department of Computer Science, City University of Hong Kong, Hong Kong SAR}
\icmlaffiliation{riken}{Center for Advanced Intelligence Project, RIKEN, Tokyo, Japan}
\icmlaffiliation{utokyo}{Department of Mathematical Informatics, the University of
 Tokyo, Tokyo, Japan}
\icmlaffiliation{mbzuai}{Mohamed bin Zayed University of Artificial Intelligence, Masdar, United Arab Emirates}

\icmlcorrespondingauthor{Wei Huang}{ wei.huang.vr@riken.jp}
\icmlcorrespondingauthor{Hau-San Wong}{ cshswong@cityu.edu.hk}

\icmlkeywords{Machine Learning, ICML}

\vskip 0.3in
]


\begingroup
\makeatletter\def\Hy@Warning#1{}\makeatother\printAffiliationsAndNotice{}  
\endgroup
\begin{abstract}
Neural Network-based active learning (NAL) is a cost-effective data selection technique that utilizes neural networks to select and train on a small subset of samples. While existing work successfully develops various effective or theory-justified NAL algorithms, the understanding of the two commonly used query criteria of NAL: uncertainty-based and diversity-based, remains in its infancy.  In this work, we try to move one step forward by offering a unified explanation for the success of both query criteria-based NAL from a feature learning view. Specifically, we consider a feature-noise data model comprising easy-to-learn or hard-to-learn features disrupted by noise, and conduct analysis over 2-layer NN-based NALs in the pool-based scenario. We provably show that both uncertainty-based and diversity-based NAL are inherently amenable to one and the same principle, i.e., striving to prioritize samples that contain yet-to-be-learned features. We further prove that this shared principle is the key to their success-achieve small test error within a small labeled set. Contrastingly, the strategy-free passive learning exhibits a large test error due to the inadequate learning of yet-to-be-learned features, necessitating resort to a significantly larger label complexity for a sufficient test error reduction. Experimental results validate our findings.

\end{abstract}

\section{Introduction}
In the deep learning era, we witness the power of neural networks in representation learning. It is also well-known that their success relies on a substantial amount of data and extensive labeling efforts. On the other hand, active learning offers various approaches to select a small subset of unlabeled samples from a large pool of data for labeling and training, while achieving comparable generalization performance to learning on the entire dataset \cite{settles2009active, aggarwal2014active}. To enjoy the best of both worlds, people combine neural networks with active learning, giving rise to Neural Network-based Active Learning (NAL) or Deep Active Learning (DAL), such that over-parameterized neural models can work with limited size of labeled data. As summarized in \citet{takezoe2023deep}, NAL/DAL incorporates two primary criteria for querying (selecting) unlabeled samples: uncertainty-based \cite{roth2006margin, joshi2009multi} and diversity-based \cite{sener2017active, gissin2019discriminative}. Also, some studies leverage both criteria  to design NAL algorithms \cite{yin2017deep, shui2020deep}. \par
Notably, while various NAL algorithms, based on two query criteria, have achieved significant empirical success, they often come without  provable performance guarantees. To overcome this limitation, recent studies \cite{gu2014batch, gu2014online, wang2022deep} came up with theory-driven NAL algorithms. These studies reformulate the problem into a subset selection problem or multi-armed bandit problem, and then utilize theoretical analysis techniques to guarantee the test performance. However, the internal mechanism remains not well understood on why the two widely used query criteria in the NAL family work so well, which naturally leads us to the following questions.\par

\begin{Msg}{Essential Questions}
    1. What is the theoretical rationale  behind the success of the two query criteria-based NAL algorithms, namely uncertainty-based and diversity-based? \\
    2. Whether and how do the two query criteria of NAL connect to each other intrinsically? 
\end{Msg}

\subsection{Our Contribution}

To answer  the above questions, in this work, we delve into the \textbf{feature learning dynamics} of NAL algorithms. To start with, we draw inspiration from the data models in \citet{zou2023benefits, allenzhu2023understanding, lu2023benign} that consist of multiple task-related feature patches and noise patches with varying strengths and frequencies, similar to what is observed in real-world imbalanced datasets, and conjecture that successful NAL algorithms are able to ensure adequate learning of all types of task-related features.\par

In this spirit, we adopt a multi-view feature-noise data model that comprises two main components: i) easy-to-learn (i.e., strong $\&$ common) features or hard-to-learn (i.e., weak $\&$ rare) features, and ii) noise. In Figure \ref{fig:lions}, the easy-to-learn features are exemplified by the frontal male lions with brown fur in the first row, given their common and easily identifiable lion traits, while lions in all the other rows can be characterized as the hard-to-learn features since they exhibit distinctive poses, colors, ages, races, fur patterns, and even heterogeneity. Hard-to-learn features are less common in the dataset and correspond to weakly recognizable lion traits, compared to the easy-to-learn features.\par
\begin{figure}[H]
    \centering
    \includegraphics[width=\linewidth]{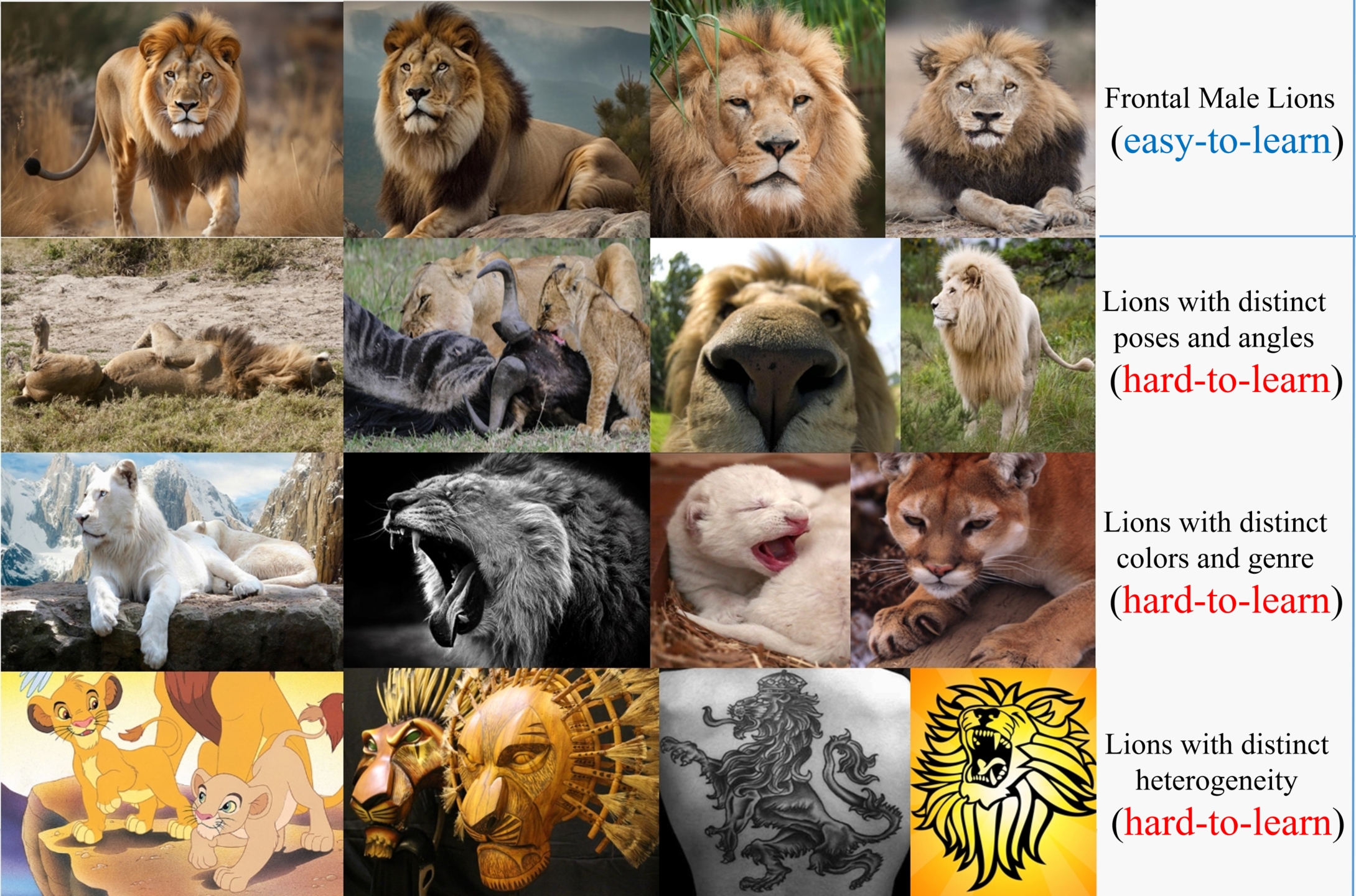}
    \caption{Lions in real-world dataset.}
    \label{fig:lions}
\end{figure}
Under our data model, we reformulate two representative NAL algorithms, i.e., Uncertainty Sampling and Diversity Sampling, in a pool-based setting, corresponding to two query criteria, respectively. Both are built upon a two-layer ReLU convolutional neural network, and trained by gradient descent. In accordance with the principle of each approach family \cite{takezoe2023deep}, the proposed Uncertainty Sampling queries based on the lowest confidence \cite{lewis1994heterogeneous}, and Diversity Sampling queries based on the largest distance between feature representations of unlabeled samples in the pool and those of labeled data \cite{sener2017active}. \par

Over our data and algorithm models, our theory sheds light on the benefits of the two primary query criteria in the NAL family. Surprisingly, our analysis unveils that the success of both criteria-based NAL stems from their inherent shared principle, leading to a unified view. Specifically, we make the following contributions in this work.

\begin{itemize}

    \item We offer valuable insights that from a feature learning view, the two query criteria-based NAL can be \textbf{unified} as one family. We provably show that the two query criteria-based NAL share the same working principle, i.e., prioritizing \hyperlink{msg:perplexing}{\textbf{perplexing samples}}-samples with yet-to-be-learned features. Our analysis reveals that in our scenario, those yet-to-be-learned features are actually those weak $\&$ rare features.

    \item We elucidate a marked disparity in the generalization capabilities between passive learning and NAL algorithms. Our analysis suggest that, both NAL algorithms can learn weak $\&$ rare features adequately via prioritizing \hyperlink{msg:perplexing}{\textbf{perplexing samples}}, and thus achieve a small test error. Contrastingly, the strategy-free passive learning exhibits a large test error. The disparity can be intensified in some out-of-distribution cases. Our experimental study corroborates this finding.
    
    \item We further uncover why and to what extent the two query criteria can alleviate labelling effort. The key lies in NAL's ability to effectively query \hyperlink{msg:perplexing}{\textbf{perplexing samples}} in the training distribution. But in contrast, we find that the strategy-free passive learning requires a significantly larger label complexity to adequately learn all types of features.
     
\begin{Msg}{Perplexing Samples}\hypertarget{msg:perplexing}{}
    Samples in the sampling pool that possess yet-to-be-learned features. We prove that both Uncertainty Sampling and Diversity Sampling inherently strive to query them.
\end{Msg}
   
\end{itemize}

\subsection{Related Work}
\textbf{Neural Active Learning.} Neural Network-based Active Learning (NAL) is one of the core data selection automation techniques in the field of Data-centric approaches for AutoML and Computer Version. As summarized in recent surveys \cite{zhan2021comparative, zhan2022comparative, takezoe2023deep}, there are two main query criteria: uncertainty-based, which chooses samples that the neural models feel most uncertain about \cite{seung1992query, lewis1994heterogeneous, roth2006margin, joshi2009multi, houlsby2011bayesian, cai2013maximizing, yang2016active, kampffmeyer2016semantic,gal2017deep, wang2022boosting, Kye_2023_ICCV, Duan_2024_WACV, cho2024querying} and diversity(representative)-based that selects samples that diverse from labeled set in the feature space \cite{stark2015captcha, du2015exploring, wang2016cost, sener2017active, gissin2019discriminative, sinha2019variational, shui2020deep}. Also, many works combine the two query criteria into the sampling (querying) strategy through weighted-sum optimization \cite{yin2017deep} or two-stage optimization \cite{ash2019deep, zhdanov2019diverse, shui2020deep}. In addition, to develop reliable algorithms, several design methods with theoretical guarantees, including theories such as VC bound \cite{balcan2006agnostic, zhu2022active}, Logistic Bound \cite{gu2014batch}, Rademacher Complexity \cite{gu2014online, shui2020deep}, and Neural Tangent Kernel \cite{wang2021neural,mohamadi2022making,kong2022neural,wang2022deep, wen2023ntkcpl}. However, despite the development of numerous effective and theory-justified algorithms, the existing studies have not yet offered a comprehensive explanation for the underlying mechanisms of the two query criteria widely applied in NAL. Largely different from prior work, our work pioneeringly explore the theoretical aspect of the two criteria, via studying the \textbf{feature learning dynamic} in NAL algorithms.\par

\textbf{Feature Learning in Learning Theory.} Recent years witness an extensive body of research in learning theory on structured data from the perspective of feature learning \cite{li2018learning, karp2021local,allenzhu2023understanding,chen2022understanding,chen23brobust,chen2023understanding,chen2023clip,chen2023why, zou2023understanding,li2023theoretical,kou2023how, kou2023implicit, huang2023graph,huang2023understanding, chidambaram2023provably,deng2023robust}. The essence of this line-of-research is to explicitly study the learning progress of features and memorization degree of noise under different data and algorithm scenarios, which serves as an intermediate proxy to examine the convergence of training and 0-1 loss. Specifically, \citet{cao2022benign} demonstrate the occurrence of \textit{benign overfitting} in Convolutional Neural Network over linearly separable data under distinct conditions. Subsequently, \citet{kou2023benign} conduct similar results with ReLU activation, \citet{meng2023benign} further derive results over XOR data, \citet{zou2023benefits} reveal the benefits of Mixup training over linearly separable data with common and rare features, and \citet{lu2023benign} explore the phenomenon of \textit{benign oscillation} over linearly separable data with common $\&$ weak and rare $\&$ strong features. Our work extends the line of research by investigating the rationale behind the two primary criteria in NAL family, over both linearly and non-linearly separable data scenarios that include common $\&$ strong and rare $\&$ weak features. Our study focuses on characterizing the \textbf{feature learning dynamics} in NAL algorithms and providing a mathematical explanation for the benefits and inner relationship of the two primary query criteria of NAL.\par

\section{Problem Settings}
\textbf{Notations.} For $l_p$ norm we utilize $\| \cdot \|_p $ to denote its computation. Considering two series $a_n$ and $b_n$, we denote $a_n=O\left(b_n\right)$ if there exists positive constant $C>0$ and $N>0$ such that for all $n \geq N$, $\left|a_n\right| \leq C\left|b_n\right|$. Similarly, we denote $a_n =\Omega\left(b_n\right)$ if $b_n=O\left(a_n\right)$ holds, $ a_n=\Theta\left(b_n\right)$ if $a_n=O\left(b_n\right)$ and $a_n=\Omega\left(b_n\right)$ both hold, $c_n=O(a_n, b_n)$ if $c_n =O(\min\{a_n, b_n\})$ holds and  $c_n=\Omega(a_n,b_n)$ if $c_n=\Omega(\max \{a_n, b_n\})$ holds. To omit logarithmic terms, we apply the notations $\widetilde{O}(\cdot), \widetilde{\Omega}(\cdot)$, and $\widetilde{\Theta}(\cdot)$. Our $\mathbb{1}(\cdot)$ is to denote the indicator variable of an event. We say $y=\operatorname{poly}\left(a_1, \ldots, a_k\right)$ if $y=O\left(\max \left\{a_1, \ldots, a_k\right\}^D\right)$ for some $D>0$, and $b=\operatorname{polylog}(a)$ if $b=$ poly $(\log (a))$.\par

\subsection{Data Distribution}
In this study, our focus is on the pool-based selective sampling scenario, where the algorithms initially train the model using an initial labeled set and subsequently query a single batch of unlabeled samples from a large sampling pool. Then the algorithms would retrain the model again with fresh initialization. We denote the size of the initial labeled set as $n_0$, the querying (sampling) size for all querying algorithms as $n^*$ ($n^* = \Omega(n_0)>n_0$), and the size of the labeled set after querying as $n_1 = n_0 + n^*$. We also define $\widetilde{n}$ as the maximum size of the labeled set after querying, such that $n_1 \leq \widetilde{n}$. Moreover, we have the initial labeled set represented as $\mathcal{D}_{n_0} \mathrel{\mathop:}= \{ \mathbf{x}^{(i)} \}_{i=1}^{n_0}$, and the sampling pool denoted as $\mathcal{P}$. Both of them are synthesized from the same data distribution $\mathcal{D}$, which is specified as follows. \par
\begin{definition}
    Let $\boldsymbol{\mu}_1 \perp \boldsymbol{\mu}_2  \in \mathbb{R}^d$ be two fixed feature vectors. Each data point $(\mathbf{x},  y)$, where $\mathbf{x}$ contains two patches as $\mathbf{x}$=$[\mathbf{x}_1^{T},  \mathbf{x}_2^T]^T$ $\in$ $\mathbb{R}^{2d}$ and $y$ $\in \{ -1 , 1 \}$ are generated from the distribution $\mathcal{D}$:
    \begin{itemize}
        \item The ground truth label y is synthesized from a Rademacher distribution.
        \item \textbf{Noise Patch.} One patch of $\mathbf{x}$ is selected as a noise patch $\boldsymbol{\xi}$, synthesized from Gaussian distribution $N(\mathbf{0}, \sigma_p^2 \cdot \mathbf{I})$.
        \item \textbf{Feature Patch.} For a feeble $p$ satisfying $p < 0.5$, the remaining patch of $\mathbf{x}$ is selected as label-related feature patch, and with high probability (1-$p$) the feature patch is a strong feature $y \cdot \boldsymbol{\mu}_1$, while only with probability $p$ the feature patch is a weak feature $y \cdot \boldsymbol{\mu}_2$.
    \end{itemize}
    We assume the following about the feature norms: \footnote{The choices of $\|\boldsymbol{\mu}_l\|$ aim to prevent learning of features completely disrupted by noise, while amplifying the distinguishability of the strong feature patch compared to the weak one. Our theory allows for a broader range of parameter settings (see Appendix \ref{app:Discussion of data distribution under other conditions} for general cases), but for the sake of simplicity in presentation, we here chose a feasible one.}: $ \forall l \in \{1, 2\}, \|\boldsymbol{\mu}_{l} \|_2^2 = \Omega( \sigma_p^2 \log (n_0/\delta), \widetilde{n}^{-1}d \sigma_p^{4})$, $\|\boldsymbol{\mu}_1\|_2^4 =\Omega( \sigma_p^4 d n_{0}^{-1}) $ and  $\|\boldsymbol{\mu}_2 \|_2^4 = O( \sigma_p^4 d n_{0}^{-1})$. 
\label{Def3.1}\end{definition}
This feature-noise data model captures the structure of an image, as depicted in Figure \ref{fig:lions}, by incorporating task-oriented distinctive patterns (features) and background patterns (noise) with different frequencies and strengths. Same as the patches setting in \citet{zou2023benefits, allenzhu2023understanding, lu2023benign}, the weak feature patches are orthogonal to the strong feature patches in our setting, which is reasonable since the rare features appear largely different to the common ones. Worth noting that this type of data setting is common in the widely-recognized feature learning line-of-research \cite{allenzhu2023understanding,cao2022benign,kou2023benign,zou2023benefits,meng2023benign}. \citet{allenzhu2023understanding} justify this type of data settings as plausible theoretical setups by highlighting the common occurrence of multiple one-task-oriented features in the latent space of Resnet, as shown in their Figure 2-4, 9. Furthermore, recent empirical and theoretical studies indicate the orthogonal nature of different features within the latent space of ViT and LLM \cite{yamagiwa2023ICA, jiang2024origins}. To extend our contributions to more practical scenarios, we also conduct rigorous study and draw similar theoretical findings over a non-linearly separable, non-orthogonal data distribution - the XOR data defined in Definition \ref{def_XORdata} - and obtained similar results.

\subsection{Querying Algorithms}

\textbf{Neural Setting.} This work considers a two-layer ReLU CNN adopted in \citet{kou2023benign, meng2023benign, kou2023implicit, chen2023why} as the base neural network for querying algorithms. The CNN function $f(\mathbf{W}, \mathbf{x})$ is defined as $\sum_{j=\pm 1} j \cdot F_j(\mathbf{W}, \mathbf{x})$, with $F_j(\mathbf{W}, \mathbf{x})$ as
\[
 F_j(\mathbf{W}, \mathbf{x}) =\frac{1}{m} \sum_{r=1}^{m} \left[\sigma\left(\left\langle\mathbf{w}_{j, r}, y \cdot \boldsymbol{\mu}\right\rangle\right) 
+\sigma\left(\left\langle\mathbf{w}_{j, r}, \boldsymbol{\xi}\right\rangle\right)\right].
\]
where the second layer is fixed as $\pm 1/m$, $m$ is the number of neurons, $\sigma(z)=\max\{z, 0\}$ is ReLU function, $\mathbf{w}_{j, r} \in \mathbb{R}^d$ denotes the weights of the $r$-th neuron of $F_j$, $\mathbf{W}_j \in \mathbb{R}^{m\times d}$ collects the weights in $F_j$ and $\mathbf{W}$ collects all weights.\par

\textbf{Training Setting.} We utilize gradient descent to train the neural model. Denote $n$ as the size of current labeled training set, denoted as $\mathcal{D}=\left\{\left(\mathbf{x}^{(i)}, y_i\right)\right\}_{i=1}^n$ generated from $\mathcal{D}$ over $\mathbf{x} \times y$. We apply the empirical logist loss:
\begin{equation}
L_S(\mathbf{W})=\frac{1}{n} \sum_{i=1}^n \ell\left[y_i \cdot f\left(\mathbf{W}, \mathbf{x}^{(i)}\right)\right],
\label{Eq:training loss}\end{equation}
where $\ell(z)=\log (1 + \exp (-z))$. The gradient update of the filters in the first layer can be written as follows:
\begin{equation}
\begin{aligned}
\mathbf{w}_{j, r}^{(t+1)} & =\mathbf{w}_{j, r}^{(t)}-\eta \cdot \nabla_{\mathbf{w}_{j, r}} L_S\left(\mathbf{W}^{(t)}\right) \\
& =\mathbf{w}_{j, r}^{(t)}-\frac{\eta}{n m} \sum_{i=1}^n {\ell_i^{\prime}}^{(t)} \cdot \sigma^{\prime}\left(\left\langle\mathbf{w}_{j, r}^{(t)}, \boldsymbol{\xi}_i\right\rangle\right) \cdot j y_i \boldsymbol{\xi}_i\\
&-\frac{\eta}{n m} \sum_{l=1}^2 \sum_{i \in U^l} \ell_i^{(t)} \cdot \sigma^{\prime(t)}\left(\left\langle\mathbf{w}_{j, r}^{(t)}, y_i \boldsymbol{\mu}_l\right\rangle\right) \cdot  j \boldsymbol{\mu}_l,
\end{aligned}
\label{Eq:w update}\end{equation}
where $U^l= \{ \mathbf{x} \in \mathcal{D} \mid \mathbf{x}_{\text {signal part }}=\boldsymbol{\mu}_l\}$ denote as the set of indices of $\mathcal{D}$ where the data's feature patch is $\boldsymbol{\mu}_l$, ${\ell_i^{\prime}}^{(t)}$ denotes $\ell^{\prime}\left[ y_i \cdot f(\mathbf{W}^{(t)}, \mathbf{x}^{(i)}) \right]$. The initial values of all elements in $\mathbf{W}^{(0)}$ are generated from independent and identically distributed (i.i.d.) Gaussian distributions with mean 0 and variance $\sigma_0^2$. The querying algorithms would have the neural models retrained after a single querying with the same model initialization.\par
\textbf{Querying Setting.} During the querying stage, all the querying algorithms select $n^*$ new unlabeled samples from $\mathcal{P}$, where the pool size $\lvert \mathcal{P} \rvert$ satisfies $\lvert \mathcal{P} \rvert =\Omega(  p^{-1} \sigma_p^4 d  \|\boldsymbol{\mu}_2 \|_2^{-4}, p^{-1} \log (1/\delta))$\footnote{The choice on $\lvert \mathcal{P} \rvert$ is to ensure the sufficient presence of weak features in $\mathcal{P}$.}. The three querying algorithms differentiate from each other by their own sampling rules as below:
\begin{itemize}
\item \textbf{Random Sampling} (strategy-free passive learning) randomly selects $n^*$ new samples from $\mathcal{P}$.
\item \textbf{Uncertainty Sampling} (uncertainty-based NAL) selects top $n^*$ new samples from $\mathcal{P}$ based on the lowest Confidence Score at time step $t$. The Confidence Score $ C \left(\mathbf{W}, \mathbf{x}\right)$ measures the model's confidence in predicting the label of sample $\mathbf{x}$, defined as below:
\[
\begin{split}
    C\left(\mathbf{W}, \mathbf{x}\right)
    &=\max \Big\{\frac{1}{1+\exp (-y \cdot f(\mathbf{W}, \mathbf{x}))},\\ 
    &\phantom{=}1-\frac{1}{1+\exp (-y \cdot f(\mathbf{W}, \mathbf{x}))}\Big\},
\end{split}
\label{query:confidence_score}\]
which represents the probability of the predicted label $y$ of logistic loss. In our scenario, the proposed Uncertainty Sampling is actually equivalent to many well-known uncertainty-based approaches such as Least Confidence \cite{lewis1994heterogeneous}, Margin \citet{roth2006margin}, and Entropy methods \cite{joshi2009multi}, as discussed in Lemma \ref{lemma:Equivalance of score function} in Appendix \ref{App:Order def}.
\item \textbf{Diversity Sampling} (diversity-based NAL) selects the top $n^*$ new samples from $\mathcal{P}$ based on the highest Feature Distance at time step $t$. The Feature Distance $D\left(\mathbf{W}, \mathbf{x} \ \mid \mathcal{D}_{n_0}\right)$ measures the $l_p$ distance between sample $\mathbf{x}$ and $\mathcal{D}_{n_0}$ in feature space, specified as:
\[
D (\mathbf{W}, \mathbf{x} \ | \mathcal{D}_{n_0} ) =\| \mathbf{Z}(\mathbf{x}, t) - \displaystyle \underset{\mathbf{x}^{(i)} \in \mathcal{D}_{n_0}}{\mathbb{E}} \mathbf{Z}(  \mathbf{x}^{(i)}, t) \|_p,
\]
where the $\mathbf{Z}(\mathbf{x}, t)$ is defined as the sum of feature maps in the feature space of CNN:
\[
\mathbf{Z}(\mathbf{x}, t) = \sum_{j} ( \sigma(\langle \mathbf{W}_j^{(t)}, \mathbf{x}_1 \rangle)) + \sigma(\langle \mathbf{W}_j^{(t)}, \mathbf{x}_2 \rangle)).
\]
Specifically, Lemma \ref{lemma of comparison} reveals that in our scenario, the proposed Diversity Sampling is equivalent for all values of $p$ within the range of $[1, \infty)$. This implies that our metric can be various distance measure, including Euclidean, Manhattan, or Minkowski distance.
\end{itemize}
The newly acquired samples are provided to an oracle to obtain their ground truth labels, which are then added to the training set. The whole procedure of the three querying algorithms are shown in Algorithm \ref{Alg:NAL}.

\textbf{Testing Setting.} The model performances at initial stage (before querying) and stage after querying are all measured by test error on a test distribution $\mathcal{D}^*$:\par
\begin{equation}
L_{\mathcal{D}^*}^{0-1}(\mathbf{W})\mathrel{\mathop:}=\mathbb{P}_{(\mathbf{x}, y) \sim \mathcal{D}^*}[y \cdot f(\mathbf{W}, \mathbf{x}) < 0].
\label{Eq:test error}\end{equation}
It is important to note that $\mathcal{D}^*$ shares the same definition as stated in Definition \ref{Def3.1}. However, it can have any occurrence probability of the weak feature, denoted as $p^*$, without the limitation of $p^* < 0.5$ compared to the training distribution. Also, the test loss is defined as :
\[L_{\mathcal{D}^*}(\mathbf{W})\mathrel{\mathop:}=\underset{{(\mathbf{x}, y) \sim \mathcal{D}^*}}{\mathbb{E}} \ell[y \cdot f(\mathbf{W}, \mathbf{x})].\] 
\par 
\begin{algorithm}[tb]
\caption{Querying Algorithms}
\label{Alg:NAL}
\begin{algorithmic}[1]
\REQUIRE Neural Network $f(\cdot; \cdot)$, initial labeled set $ \mathcal{D}_{n_0} \mathrel{\mathop:}= \{ \mathbf{x}^{(i)} \}_{i=1}^{n_0} \subseteq \mathcal{D} $, sampling pool $\mathcal{P} \subseteq \mathcal{D}$, test distribution $\mathcal{D}^*$, sample size $n^*=\widetilde{n}-n_0$, $ \sigma_0$, $T$
\STATE Initialize Neural Network $f(\mathbf{W}^{(0)}; \cdot)$
\FOR{$t \leftarrow 1$ to $T$}
\STATE Train Neural Network over $\mathcal{D}_{n_0}$ by $L_S(\mathbf{W})$
\ENDFOR
\STATE Querying: Sample $n^*$ new samples from $\mathcal{P}$ based on particular rules. New samples $ \mathcal{D}_{n^*}$ are labeled by oracle and included to the new labeled set $ \mathcal{D}_{n_1} \mathrel{\mathop:}= \mathcal{D}_{n_0} \cup \mathcal{D}_{n^*}$
\STATE Initialize Neural Network $f(\mathbf{W}^{(0)}; \cdot)$
\FOR{$t \leftarrow 1$ to $T$}
\STATE Train Neural Network over $\mathcal{D}_{n_1}$ by $L_S(\mathbf{W})$
\ENDFOR
\STATE Test performance of Neural Network $f(\mathbf{W}^{(T)}; \cdot)$  over $\mathcal{D}^{*}$ and obtain $L_{\mathcal{D}^*}^{0-1}\left(\mathbf{W}^{(T)}\right)$
\STATE \textbf{return} $L_{\mathcal{D}^*}^{0-1}\left(\mathbf{W}^{(T)}\right)$
\end{algorithmic}
\end{algorithm}

\section{Theoretical Results}
For both the initialization stage and the second stage, we consider the learning period $0 \leq t \leq T^*$, where $T^*=\eta^{-1}$ poly $\left(\varepsilon^{-1}, d, n_0, m\right) \geq \widetilde{\Omega}\left(\eta^{-1} \varepsilon^{-1} m n_0 d^{-1} \sigma_p^{-2}\right)$ is the maximum admissible iterations for the initial stage. The following provides our main theories over linearly separable data. For non-linear XOR data, please refer to our similar theoretical results in Appendix \ref{Appdx:XOR}.\par
We first adopt \textit{signal-noise decomposition} techniques in \citet{cao2022benign} over $\mathbf{w}_{j, r}^{(t)}$. By the update rule in (\ref{Eq:w update}), we can derive that there exist unique coefficients $\gamma_{j, r, l}^{(t)}$ and $\rho_{j, r, i}^{(t)}$ such that
\begin{equation}
\mathbf{w}_{j, r}^{(t)}=\mathbf{w}_{j, r}^{(0)}+j \cdot  \sum_{l=1}^2 \gamma_{j, r, l}^{(t)} \cdot \dfrac{\boldsymbol{\mu}_l}{\|\boldsymbol{\mu}_l\|_2^{2}} +\sum_{i=1}^{n} \rho_{j, r, i}^{(t)} \cdot \dfrac{\boldsymbol{\xi}_i}{\|\boldsymbol{\xi}_i\|_2^{2}}
\label{eq:firstsnd}\end{equation}
The normalization factors $\|\boldsymbol{\mu}_l\|_2^{-2}$ and $\|\boldsymbol{\xi}_i\|_2^{-2}$ leads to $\langle \mathbf{w}_{j, r}^{(t)}, \boldsymbol{\mu}_l\rangle \approx \gamma_{j,r,l}^{(t)}, \langle \mathbf{w}_{j, r}^{(t)}, \boldsymbol{\xi}_i \rangle \approx \rho_{j, r, i}^{(t)}$. Importantly, $\gamma_{j,r,l}^{(t)}$ characterizes \textbf{the learning progress} of feature $\boldsymbol{\mu}_l$, and $\rho_{j, r, i}^{(t)}$ characterizes the degree of noise memorization. Geometrically, the $\gamma_{j,r,l}$ indicates how well the model filters integrate the low-dimensional patterns of the task-oriented features in its latent projection space, and $\rho_{j,r,i}$ quantifies the extent to which model filters memorize the high-dimensional complex noise. Then, by conducting a scale analysis of the two coefficients, we can then assess the cases where models mainly focus on capturing underlying patterns while avoiding excessive fitting of noise, which we refer to as \textit{benign overfitting}. Additionally, this analysis helps us identify situations of \textit{harmful overfitting}, where the models become overly complex, primarily memorizing noise and leading to poor generalization on new, unseen data.\par

Our findings then reveal that in our case, both the two heuristic NAL methods inherently amenable to query those data with yet-to-be-learned features (i.e., features that model exhibits low $\gamma_{j,r,l}$). Consequently, the NNs are enabled to sufficiently learn all types of features, and then exhibit \textit{benign overfitting} even in the case where the label complexity is quite limited.\par

To present our findings, we make the following assumptions.
\begin{condition} Suppose that:
\begin{enumerate}
    \item The initial training size $n_0$, the maximum admissible size after querying $\widetilde{n}$, and the width of neural network $m$ satisfy $n_0  = \Omega(\log (m/\delta),  p^{-1} \log (1/\delta))$, $ \widetilde{n} =O(p^{-1} \sigma_p^4 d \|\boldsymbol{\mu}_2 \|_2^{-4})$, $m = \Omega( \log (n_0/\delta))$.
    \item Dimension $d$ is sufficiently large:
    $\forall l \in \{ 1, 2 \}$, $d =\Omega(\widetilde{n}\sigma_p^{-2} \|\boldsymbol{\mu}_{l} \|_2^2 \log \left(T^*\right), \widetilde{n}^2 \log (\widetilde{n} m/\delta)(\log(T^*))^2)$.
    \item The standard deviation of Gaussian initialization $\sigma_0$ is appropriately chosen such that $\forall l \in \{ 1, 2 \}$, $\sigma_0 = O( \|\boldsymbol{\mu}_{l} \|_2^{-1} (\log (m/\delta)^{-1/2}), \sigma_p^{-1} d^{-1} \widetilde{n}^{1/2} )$. The learning rate of all algorithms $\eta$ satisfies that $\eta =O( \sigma_p^{-2} d^{-1} \widetilde{n} ,\sigma_p^{-2} d^{-3 / 2} \widetilde{n}^2 m (\log (\widetilde{n} / \delta))^{1/2})$.
\end{enumerate}\label{Con4.1}
\end{condition}

The condition on $n_0$ is to guarantee there exists enough strong features in the initial training set with probability at least $1 - O(e^{-n_0 p})$, while the condition on $\widetilde{n}$ prevents the final training size from being too large, even for passive learning to perform well with considerable chance. The requirement on $d$ ensures the problem is in a sufficiently overparameterized setting, as in prior works \cite{chatterji2021finite, cao2022benign, frei2022benign, kou2023benign, lu2023benign, chidambaram2023provably}. The conditions on $\sigma_0$ and $\eta$ guarantee that gradient descent can effectively minimize the empirical loss. A detailed discussions over parameter settings are provided in Appendix \ref{Appendix: Discussions on the Parameters}.

The following results illustrate the presence of \textit{benign overfitting} (i.e., small training loss and small test error) as well as \textit{harmful overfitting} (i.e., small training loss but large test error) in the three querying algorithms.

\begin{proposition}\textbf{(Before Querying)} At the initial stage before querying, $\forall \varepsilon>0$, under Condition \ref{Con4.1},  with probability at least $1-\delta$, there exists $t=\widetilde{O}\left(\eta^{-1} \varepsilon^{-1} m n_0 d^{-1} \sigma_p^{-2}\right)$, the followings hold for all of the three querying algorithms:
\begin{enumerate}
    \item The training loss converges to $\varepsilon$, i.e., $L_S\left(\mathbf{W}^{(t)}\right) \leq \varepsilon$.

    \item The test error remains at constant level, i.e., $L_{\mathcal{D}^*}^{0-1}\left(\mathbf{W}^{(t)}\right) =\Theta(1) \geq 0.12 \cdot p^*$.
\end{enumerate}
\label{Prop4.3}\end{proposition}

Proposition \ref{Prop4.3} outlines the scenarios of \textit{harmful overfitting} for all algorithms at the initial stage, which is not a surprise since the initial size $n_0$ is limited and always insufficient for adequate learning. Subsequently, the following lemma uncovers a crucial finding regarding the querying stage.

\begin{proposition}
    \textbf{(Querying Stage)} During Querying, under the same conditions as Proposition \ref{Prop4.3}, if\footnote{We can relax the requirement for the discrepancy of feature norms, as discussed in Appendix \ref{app:Discussion of data distribution under other conditions}. The specific choice made in our presentation was for the sake of simplicity and clarity.} $  \|\boldsymbol{\mu}_1\|_2^2 -  \|\boldsymbol{\mu}_2\|_2^2 = \Omega( {\sigma_p}^{2} (d n_0^{-1} \log (m/\delta^{\prime}))^{1/2})$, with probability at least $1-\Theta(\delta+\delta^{\prime})$, both Uncertainty Sampling and Diversity Sampling pick $n^*$ samples that exhibit lowest $\displaystyle \underset{j,r}{\mathbb{E}} \gamma_{j,r,l}^{(t)}$. 
\label{prop:sample stage: learning on sample}\end{proposition}
Proposition \ref{prop:sample stage: learning on sample} provides a unifying insight that both NAL algorithms prioritize \hyperlink{msg:perplexing}{\textbf{perplexing samples}}-samples that exhibit a lack of learning progress (measured by $\displaystyle \underset{j,r}{\mathbb{E}} \gamma_{j,r,l}^{(t)}$). Lemma \ref{lem:final coefficient} indicates that these \hyperlink{msg:perplexing}{\textbf{perplexing samples}} here are essentially samples that contain weak $\&$ rare features. We discuss the nature of these \hyperlink{msg:perplexing}{\textbf{perplexing samples}} in general cases in Appendix \ref{app:Discussion of data distribution under other conditions}. Our inference process for the following theorem reveals that the ability to prioritize these samples is the main contributor to the success of both NAL algorithms.
\begin{theorem} \textbf{(After Querying)}
     If the sampling size $n^*$ of the three querying algorithms satisfies $ C_1 \sigma_p^4 d\|\boldsymbol{\mu}_2 \|_2^{-4} - p n_0 / 2 \leq n^*=\Theta(\widetilde{n} - n_0) \leq \widetilde{n} - n_0$, where $C_1$ is some positive constant. Then for $\forall \varepsilon>0$, under the same conditions as Proposition \ref{prop:sample stage: learning on sample},  with probability more than 1 - $\Theta(\delta + \delta^\prime)$, $\exists t=\widetilde{O}\left(\eta^{-1} \varepsilon^{-1} m (n_0+n^*) d^{-1} \sigma_p^{-2}\right)$ such that:
    \begin{itemize}
    \item For all of the three querying algorithms, the training loss converges to $\varepsilon$, i.e., $L_S\left(\mathbf{W}^{(t)}\right) \leq \varepsilon$.
    \item \textbf{Uncertainty Sampling} and \textbf{Diversity Sampling} algorithms have small test error: $L_{\mathcal{D}^*}^{0-1}\left(\mathbf{W}^{(t)}\right) \leq \exp (\Theta\left( \dfrac{-\widetilde{n}\|\boldsymbol{\mu}_l\|_2^4}{\sigma_p^4 d} \right)), l \in \{ 1,2 \}$.
    \item \textbf{Random Sampling} algorithm would remain constant order test error: $L_{\mathcal{D}^*}^{0-1}\left(\mathbf{W}^{(t)}\right) = \Theta(1) \geq 0.12 \cdot p^*$.
    \end{itemize}
    
\label{theory main}\end{theorem}

Theorem \ref{theory main} implies that NAL algorithms achieve \textit{benign overfitting}, whereas the passive learning remains \textit{harmful overfitting}. It worth noting that as $p^*$ increases, the test error of Random Sampling tends to explode, especially in out-of-distribution scenarios where $p^* > 0.5 > p$. In contrast, Uncertainty Sampling and Diversity Sampling consistently achieve low test errors regardless of the value of $p^*$, which highlights the superiority of Uncertainty Sampling and Diversity Sampling over Random Sampling. \par
Given that strategy-free passive learning can also adequately learn all types of features with ample data, the following corollary aim to show the extent to which NAL algorithms alleviate the burden of labeling.
\begin{corollary}
    \textbf{(Label Complexity)} Under the same conditions as stated in Theorem \ref{theory main}, with a probability of at least $1-\Theta(\delta + \delta^{\prime})$, we observe distinct label complexities for strategy-free passive learning and NAL algorithms in achieving Bayes-optimal generalization ability:
\begin{itemize}
\item For a fully trained neural model, the label complexity $n_{CNN}$ requires $\Omega(p^{-1} \sigma_p^2 d \|\boldsymbol{\mu}_2\|_2^{-4})$.
\item For two NAL algorithms, the maximum label complexity $\widetilde{n}$ only requires $\Omega(\sigma_p^2 d \|\boldsymbol{\mu}_2\|_2^{-4})$.
\end{itemize}
\label{coro:label complexity}\end{corollary}
This corollary suggests that NAL algorithms can significantly reduce labeling effort, approximately on the order of $\Theta(p^{-1})$. This holds true even without the requirement of disparity between feature norms, as demonstrated in Appendix \ref{app:Discussion of data distribution under other conditions}. Hence, we can conclude that NAL algorithms are effective in minimizing labeling effort, particularly in imbalanced data scenarios where the degree of discrimination or rarity varies within the data. In collaboration with Proposition \ref{prop:sample stage: learning on sample} and Theorem \ref{theory main}, the essence lies in NAL's capability to effectively grasp yet-to-be-learned features.
\section{Proof Sketch}
In this section, we provide an overview of the proof outlines for our theory over linearly separable data. Here we denote $n$ as the number of training data in current labeled set, which is $n_0$ at initial stage and $n_1$ after sampling (querying). For $ s \in \{ 1,2 \}, l \in \{ 1,2 \} $, the notations of $n_{s, l}$ represent the number of feature $\boldsymbol{\mu}_l$ at the initial stage $s=1$ and stage after querying $s=2$. And for notation simplicity we denote $\tau_1$ and $\tau_2$ as the proportion of data with strong and weak feature in current dataset.\par
Here are the main challenges we faced and the techniques we used to address them:
\begin{itemize}
\item The synthesis of $\mathcal{D}_{n_0}$, $\mathcal{P}$, and the final labeled set obtained through Random Sampling require sequential martingale-type subset generations from distribution $\mathcal{D}$, which poses a big challenge to our analysis. Our solution was to treat the results as independent binomial random variables, which allow us to conduct a reliable analysis with high-probability results by leveraging the properties of binomial tails.

\item During querying, NAL algorithms need to query the samples with the lowest Confidence Score or the highest Feature Distance from the entire sampling pool $\mathcal{P}$. This involves $\lvert \mathcal{P} \rvert(\lvert \mathcal{P} \rvert-1)/2$ comparison operations. To better scrutinize the sampling dynamics, we defined two full orders and conducted an order-dependent querying analysis to examine the high probability events via combinatorial analysis.

\item Depicting the generalization capability of three different querying algorithms along the whole process was a big challenge. We addressed this by proposing a label complexity-based test error analysis regime, which allowed us to incorporate different scenarios into a single inferential process.

\end{itemize}

\subsection{Feature Learning and Noise Memorization Analysis}

Leverage the inductive techniques adopted in many works \cite{cao2022benign, kou2023benign, meng2023benign, kou2023implicit, chen2023why}, we can in our case study the coefficient scales. 

\begin{lemma} Under Condition \ref{Con4.1}, there exists $T_1=\Theta(\eta^{-1}nm\sigma_p^2d^{-1})$, for $t \in\left[T_1, T^*\right]$ we have the following hold for $\forall j \in\{ \pm 1\}, r \in[m]$ and $l \in \{1, 2\}$:
\begin{itemize}
    \item $\sum_{i=1}^n \rho_{j, r, i}^{(t)} \cdot \mathbb{1}(\rho_{j, r, i}^{(t)}>0)=\Omega(n) $,
    \item $\gamma_{j, r, l}^{(t)}=\Theta\left( \tau_l n \cdot \sigma_p^{-2}d^{-1}\| \boldsymbol{\mu}_l\|_2^2 \right)$.
\end{itemize}
\label{lem:final coefficient}\end{lemma}

It is evident that there is a noticeable disparity in the learning efficiency of features, as $\rho_{j, r, i}^{(t)}$ is directly proportional to both the data proportion $\tau_l$ and the feature norms $\|\boldsymbol{\mu}_l \|_2$. Furthermore, according to Lemma \ref{lem:rho n}, we can model the data synthesis from $\mathcal{D}$ as a binomial variable. This allows effective control over the probability tails, resulting in $\tau_2=\Theta(p)$ and $\tau_1=\Theta(1-p)$. Thus, we can conclude that the \hyperlink{msg:perplexing}{\textbf{perplexing samples}} are actually those $\boldsymbol{\mu}_2$-equipped samples. Subsequently, we can now examine the querying stage closely.

\subsection{Order-dependent Sampling (Querying) Analysis}
To rigorously analyze the statistics of the querying stage, we define two orders, namely Uncertainty Order $\preceq_{C}^{(t)}$ and Diversity Order $\preceq_{D}^{(t)}$. For $\forall \mathbf{x}, \mathbf{x}^{\prime} \in \mathcal{P}$, we have $\mathbf{x}^{\prime} \preceq_{C}^{(t)} \mathbf{x}$ if $ C\left(\mathbf{W}^{(t)}, \mathbf{x}^{\prime}\right) \geq C\left(\mathbf{W}^{(t)}, \mathbf{x}\right)$, and $\mathbf{x}^{\prime} \preceq_{D}^{(t)} \mathbf{x} \text { if \  } D\left(\mathbf{W}^{(t)}, \mathbf{x}^{\prime} \ \mid \mathcal{D}_{n}\right) \leq D\left(\mathbf{W}^{(t)}, \mathbf{x} \ \mid \mathcal{D}_{n}\right), \forall p \in [1, \infty)$. Specifically, if the Confidence Score of all elements in a set $\mathbf{X}$ at time step $t$ are all less than those in the set $\mathbf{X}^\prime$, we utilize the same notation to describe the Uncertainty Order between sets: $ \mathbf{X} \preceq_C^{(t)} \mathbf{X}^{\prime}$. Similarly, we also have set-level notation for $\preceq_{D}^{(t)}$. The detailed definitions are delayed to Appendix \ref{app:details of querying algorithms}.\par

The following lemma presents our important findings when examining the two orders of samples. 
\begin{lemma}
 Under the same conditions in Proposition \ref{prop:sample stage: learning on sample}, for $\mathbf{x}, \mathbf{x}^{\prime} \in \mathcal{P}$, denote $\boldsymbol{\mu}_{l_{\mathbf{x}}}, \boldsymbol{\mu}_{l_{\mathbf{x}^{\prime}}}$ as the feature patches in $\mathbf{x}$ and $\mathbf{x}^{\prime}$ separately, where $l_{\mathbf{x}},l_{\mathbf{x^{\prime}}} \in \{1, 2\}$, it holds that

 \begin{itemize}
          \item $\mathbf{x}^{\prime} \preceq_{C}^{(t)} \mathbf{x}$ has a sufficient event that
    \begin{equation}
    \begin{aligned}
        &\{ \underbrace{\Theta(\underset{r}{\mathbb{E}} (\gamma_{y^{\prime}, r, l_{\mathbf{x}^{\prime}}})) - \Theta(\underset{r}{\mathbb{E}}(\gamma_{y, r, l_{\mathbf{x}}}))}_{\text{Learning Progress Disparity: Feature in $\mathbf{x}$ vs. Feature in $\mathbf{x}^{\prime}$}} \\
        & > \max_{j, r, l}\{ \left| \left\langle\mathbf{w}_{j, r}^{(t)}, \mathbf{z}_l\right\rangle \right|\}\}.
    \end{aligned}
    \label{eq:Comparision_Uncertainty}\end{equation}
        \item $\mathbf{x}^{\prime} \preceq_{D}^{(t)} \mathbf{x}$ has a sufficient event that
     \begin{equation}
     \begin{aligned}
     &\{ \underbrace{\lvert  \Theta(\underset{r}{\mathbb{E}}(\gamma_{y, r, l_{\mathbf{x}}}))- \sum_l \tau_l \cdot \Theta(\underset{i_l \in U_0^l, r}{\mathbb{E}}(\gamma_{y_{i_l}, r, l})) \rvert}_{\text{Learning Progress Disparity: Feature in $\mathbf{x}$ vs. Features in Initial Set}} \\
     &- \lvert  \underbrace{\Theta(\underset{r}{\mathbb{E}} (\gamma_{y^{\prime}, r, l_{\mathbf{x}^{\prime}}}))- \sum_l \tau_l \cdot \Theta(\underset{i_l \in U_0^l, r}{\mathbb{E}}(\gamma_{y_{i_l}, r, l}))}_{\text{Learning Progress Disparity: Feature in $\mathbf{x}^{\prime}$ vs. Features in Initial Set}} \rvert\\
     & > \max_{j, r, l}\{ \left| \left\langle\mathbf{w}_{j, r}^{(t)}, \mathbf{z}_l\right\rangle \right|\} \},
     \end{aligned}
    \label{eq:Comparision_Diversity}\end{equation}
    where $U_0^l= \{ \mathbf{x} \in \mathcal{D}_0 \mid \mathbf{x}_{\text {signal part }}=\boldsymbol{\mu}_l\}$. 

 \end{itemize}
\label{lemma of comparison}\end{lemma}
\begin{remark}
This lemma demonstrate that Uncertainty Sampling holds the comparisons of the model's learning progress of features in $\mathcal{P}$, as shown in (\ref{eq:Comparision_Uncertainty}). On the other hand, Diversity Sampling cares the comparisons of the disparity between model's learning progress of samples and the labeled training set, as shown in (\ref{eq:Comparision_Diversity}).
\end{remark}

\begin{figure*}[t!]
\centering
\subfigure[Full trained model]{\includegraphics[width=0.4975\textwidth]{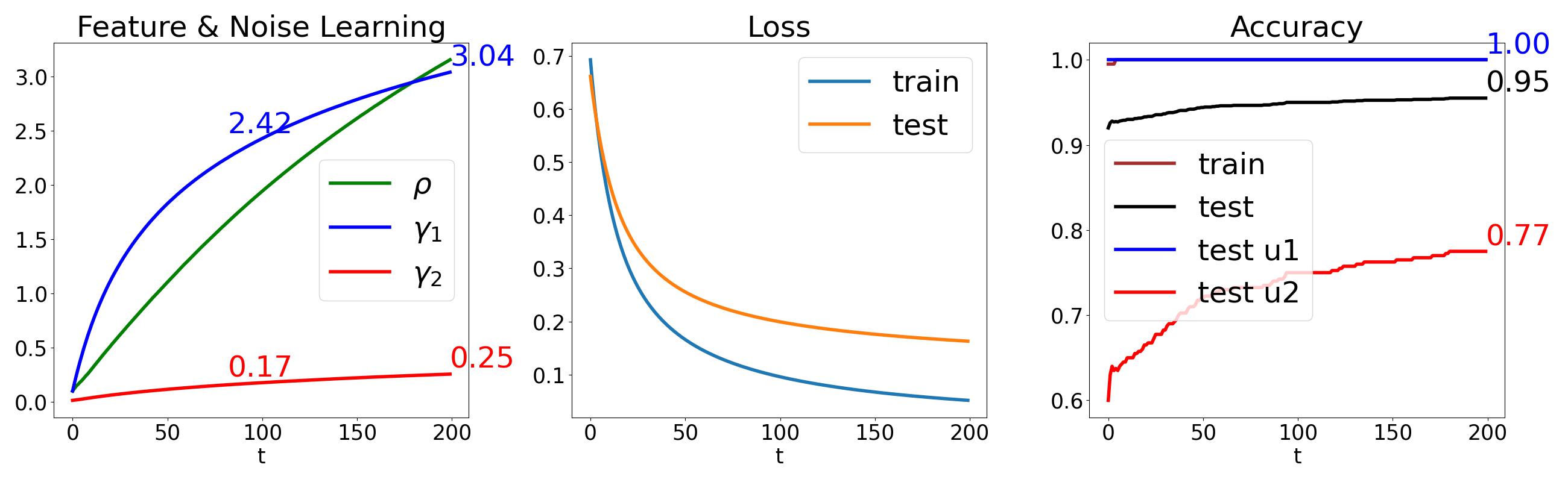}}
\subfigure[Random Sampling]{\includegraphics[width=0.4975\textwidth]{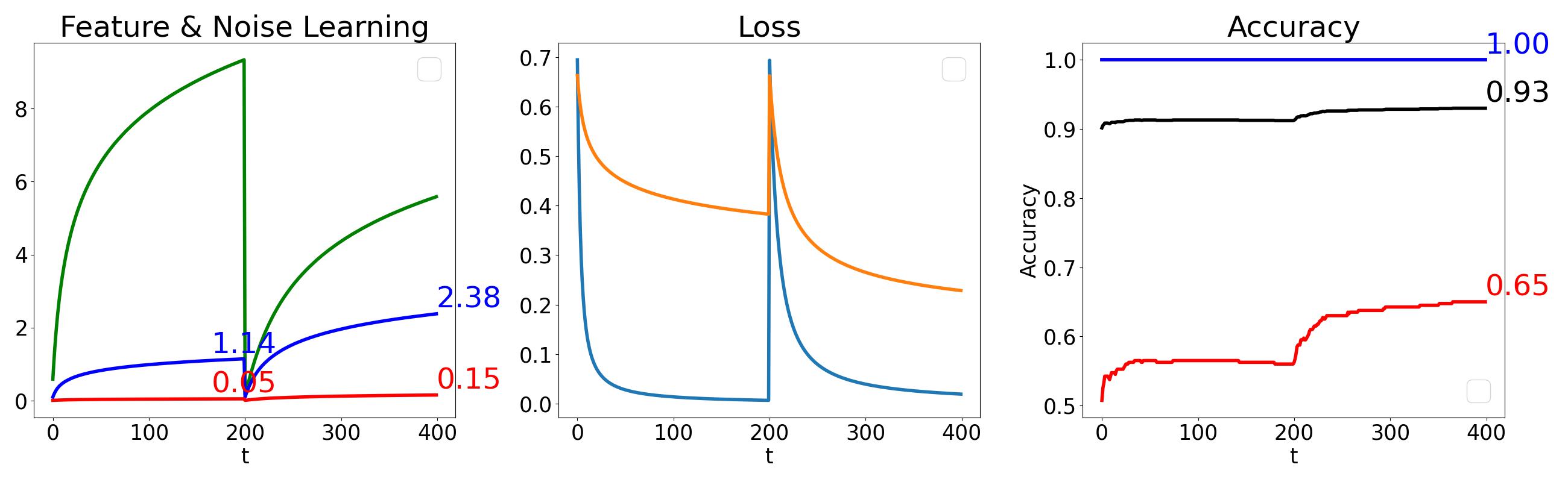}}
\subfigure[Uncertainty Sampling]{\includegraphics[width=0.4975\textwidth]{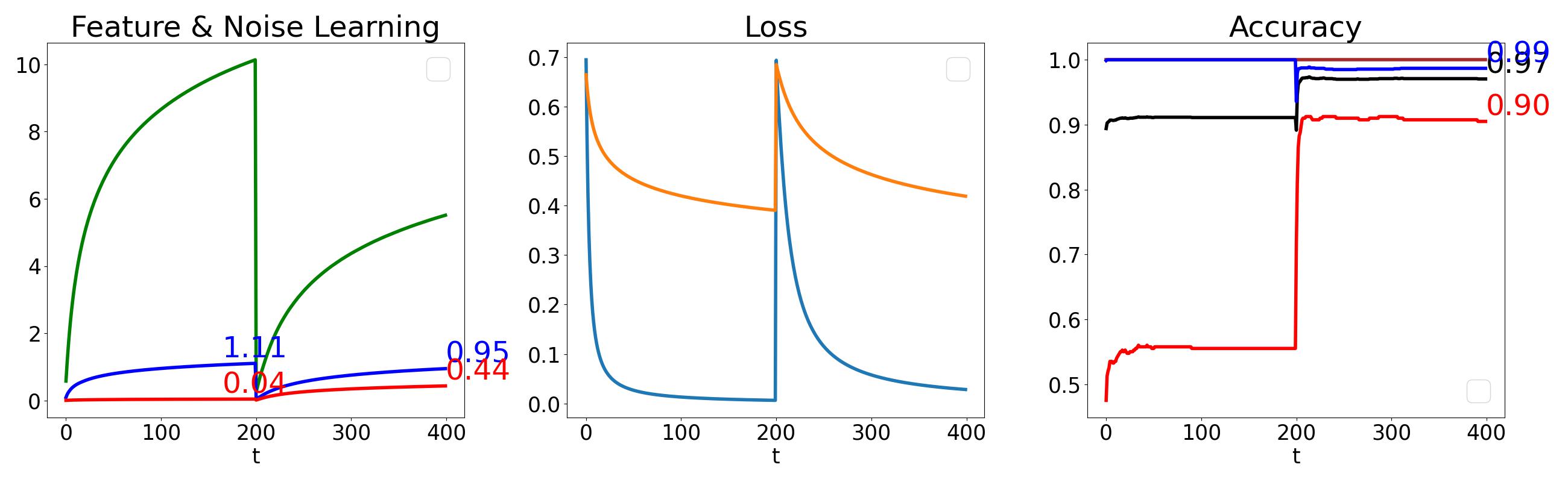}}
\subfigure[Diversity Sampling]{\includegraphics[width=0.4975\textwidth]{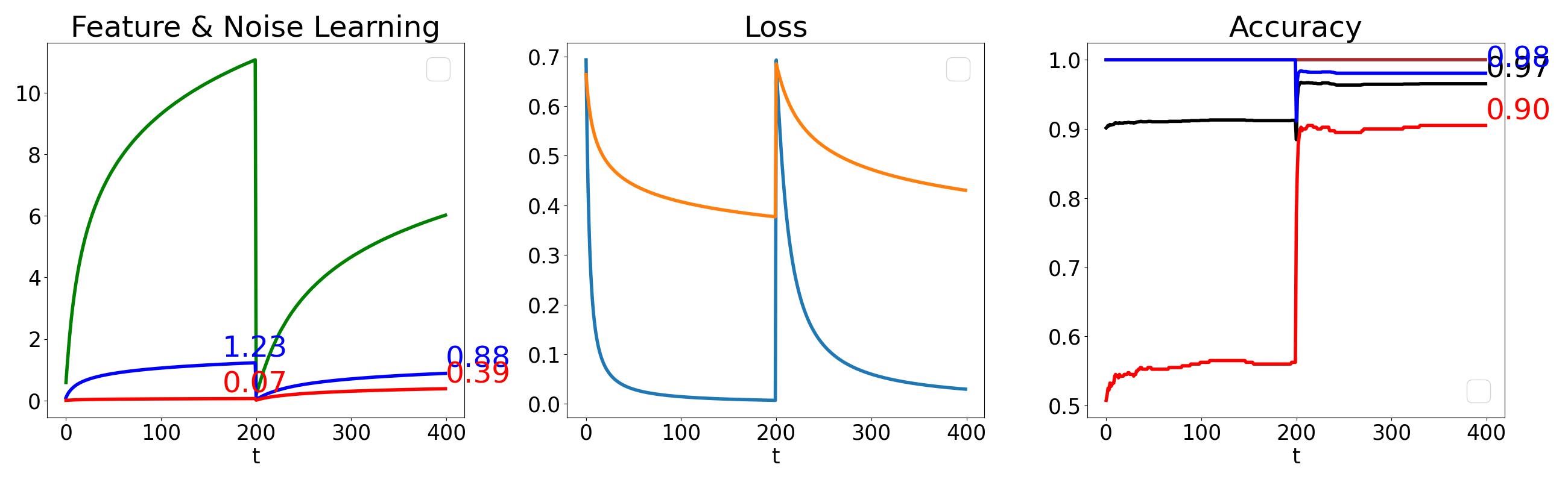}}
\caption{Learning/memorization progress of features and noise ($\gamma_l$ represents $\max_{j,k} \gamma_{j,k,l}^{(t)}$, and $\rho$ represents $\max_{j,k,i} \rho_{j,k,i}^{(t)}$, train/test losses, and test accuracy of the full-trained model and the three querying algorithms, with $T^*=200$,  $d=2000$, $\|\boldsymbol{\mu}_1\|=9$, $p=p^*=0.2$, $\|\boldsymbol{\mu}_2\|=3$, $n_{CNN}=200$, $n_0=10$, $n^{*}=30$ and $\lvert \mathcal{P} \rvert = 190$.}
\label{fig:baseline_lineardata}
\end{figure*}

We note that (\ref{eq:Comparision_Diversity}) is irrelevant to the $l_p$ distance measure metric (i.e., $\forall p \in [1, \infty) $). This is because we can eliminate the scaling term $m^{\frac{1}{p}}$ at two sides of the inequality when examining the probability lower bound (see more details in Appendix \ref{app:Order-dependent Sampling(Querying Analysis)}). Based on Lemma \ref{lem:final coefficient}, the event (\ref{eq:Comparision_Uncertainty}) and event (\ref{eq:Comparision_Diversity}) could be all simplified to the following shared sufficient event
\[
\{\Theta(\underset{j,r}{\mathbb{E}}(\gamma_{j, r, l_{\mathbf{x}^{\prime}}})) - \Theta(\underset{j,r}{\mathbb{E}}(\gamma_{j, r, l_{\mathbf{x}}})) > \max_{j, r, l}\{ \left| \left\langle\mathbf{w}_{j, r}^{(t)}, \mathbf{z}_l\right\rangle \right|\} \}.
\]
This implies that both the event $\{\mathbf{x}^{\prime} \preceq_{C}^{(t)} \mathbf{x}\}$ and the event $\{\mathbf{x}^{\prime} \preceq_{D}^{(t)} \mathbf{x}\}$ have a common occurrence where the model's learning of $\boldsymbol{\mu}_{l_{\mathbf{x}}}$ is considerably worse compared to its learning of $\boldsymbol{\mu}_{l_{\mathbf{x}^{\prime}}}$. Based on this observation and Lemma \ref{lem:final coefficient}, we can deduce the following lemma with some effort.

\begin{lemma}
Under the same conditions as Proposition \ref{prop:sample stage: learning on sample}, denoting $\mathbf{X}_{\mathcal{P}}^1 \subsetneqq \mathcal{P}$ as the collection of all the data points with strong feature $\boldsymbol{\mu}_1$ in $\mathcal{P}$, and  $\mathbf{X}_{\mathcal{P}}^2 \subsetneqq \mathcal{P}$ as the collection of data points with weak feature $\boldsymbol{\mu}_2$, we have the conclusion that with probability more than 1-$\Theta(\delta^\prime)$, $\mathbf{X}_{\mathcal{P}}^1 \preceq_{C}^{(t)} \mathbf{X}_{\mathcal{P}}^2$ and $\mathbf{X}_{\mathcal{P}}^1 \preceq_{D}^{(t)} \mathbf{X}_{\mathcal{P}}^2$ ($\forall p \in \left[1, \infty\right)$) both hold.
\label{Lem:order_pool}\end{lemma}

This lemma directly implies the result in Proposition \ref{prop:sample stage: learning on sample}.

\subsection{Label Complexity-based Test Error Analysis}
To assess the generalization ability of all the three querying algorithms before and after querying, we establish a comprehensive analysis regime that examines the impact of label complexity for each type of feature on the test error, via a single inferential process. Specifically, We introduce the following lemma, employing a standard proving technique utilized in prior research \cite{chatterji2021finite, frei2022benign, kou2023benign, meng2023benign}.

\begin{lemma}
    Under Condition \ref{Con4.1}, $\forall \varepsilon>0$, $\exists\ t=\widetilde{O}\left(\eta^{-1} \varepsilon^{-1} m n_0 d^{-1} \sigma_p^{-2}\right)$, we have the following two situations before and after querying (i.e., $\forall s \in \{0, 1\}$) for three quering algorithms:
\begin{itemize}
    \item The training loss converges to $\varepsilon$, i.e., $L_S\left(\mathbf{W}^{(t)}\right) \leq \varepsilon$.
    \item If $\forall l \in \{1, 2\}, n_{s, l} 
     \geq  C_1 \sigma_p^4 d \|\boldsymbol{\mu}_l\|_2^{-4}$ holds, the test error achieves Bayes-optimal: $L_{\mathcal{D}^*}^{0-1}\left(\mathbf{W}^{(t)}\right) \leq p^*_{1} \cdot \exp \left( \dfrac{-n_{s, 1}\|\boldsymbol{\mu}_1\|_2^4}{C_3 \sigma_p^4 d} \right) + p^*_{2} \cdot \exp \left( \dfrac{-n_{s, 2}\|\boldsymbol{\mu}_2 \|_2^4}{C_4 \sigma_p^4 d}. \right)$
    \item If $\exists l^\prime \in \{1, 2\}, n_{s, {l^\prime}} \leq C_2 \sigma_p^4 d \|\boldsymbol{\mu}_{l^\prime}\|_2^{-4}$ holds, the test error stays constant-level: $L_{\mathcal{D}^*}^{0-1}\left(\mathbf{W}^{(t)}\right) \geq 0.12 \cdot p^*_{l^\prime}.$
\end{itemize}
Here $p^*_l$ denotes the occurrence probability of feature $\boldsymbol{\mu}_{l}$, $C_1$, $C_2$, 
$C_3$ and $C_4$ are some positive constants.
\label{main:lem for thm}\end{lemma}

By Condition \ref{Con4.1}, along with the findings from Lemma \ref{Lem:order_pool} and Lemma \ref{main:lem for thm}, we can deduce that only the two NAL algorithms are able to obtain ample $\boldsymbol{\mu}_2$ for adequate learning after querying, which support the results in Proposition \ref{Prop4.3} and Theorem \ref{theory main}. Also, by Lemma \ref{lem:rho n} and Lemma \ref{main:lem for thm}, Random Sampling necessitates a label complexity that is approximately $\Theta(p^{-1})$ times larger to sufficiently learn $\boldsymbol{\mu}_2$. This finding aligns with the conclusions in Corollary \ref{coro:label complexity}.

\section{Experiments}\label{section:EXP}

In this section, we demonstrate the validity of our theoretical analysis through simulations. The experiments regarding the theories of XOR data as well as other data settings are also conducted, please refer to Appendix \ref{Expe:XORdata} for further details.\par
Here we generate synthetic data exactly following Definition \ref{Def3.1}. Specifically, we let the dimensionality as $d=2000$, and strengths of the strong and weak feature as $\|\boldsymbol{\mu}_1 \|_2 = 9$ and $ \|\boldsymbol{\mu}_2 \|_2 = 3$, respectively. For the occurrence probability, we let $p=p^*=0.2$. For size setting of data, we let the $n_{CNN}$=200, $n_0$=10, $n^*=30$ and $\hat{n}=40$, and set $\lvert \mathcal{P} \rvert = 190$. For model initialization, we let $\sigma_p = 1$ and $\sigma_0 = 0.01$. The parameters are initialized using the default method in PyTorch, and the models are trained using gradient descent with a learning rate of 0.1 for 200 iterations at the initial stage and the stage after sampling. All the data points are generated beforehand and shared by all the algorithms, thus the results are fairly comparable. \par

Figure \ref{fig:baseline_lineardata} illustrates the effectiveness of both Uncertainty Sampling and Diversity Sampling in comparison to Random Sampling and full-trained ReLU CNN model with ample quantity of training samples. It’s evident that the learning of weak $\&$ rare feature (quantified by $\gamma_2$) in hard-to-learn samples are significantly poorer than strong $\&$ common feature (quantified by $\gamma_1$) in easy-to-learn samples at the initial stage. After querying, we see explicitly that both the NAL algorithms learn the weak $\&$ rare feature well and achieve comparable test performance compared to full trained model after querying. In contrast, Random Sampling continues to exhibit limited learning progress of weak features and results in poor test accuracy. The results verify Proposition \ref{Prop4.3} and Theorem \ref{theory main}. Illustrations of the querying stage details are deferred to Appendix \ref{app:Sampling Information}.

\section{Potential Extension and Implication for Practical NALs}\label{Section: Potential Expention and Implication}

In this section, we first explore the potential extensions of our findings to broader theoretical realm, then elaborate on the practical implications derived from our theories.\par

\textbf{Potential Extension to Multi-round NALs.} The intrinsic principle we uncovered underlying both NAL methods is not tied to the single-round setting, and a fine-grained analysis can be conducted on complex iterative processes, as discussed in Appendix \ref{Appendix: Multi-round}.

\textbf{Potential Extension to Broader NALs: BADGE \cite{ash2019deep} as an Exemplar.} The key idea behind BADGE is to prioritize samples exhibiting large and diverse gradients. Our analysis reveals that such samples in our context tend to have smaller-scale latent representations ($\gamma_{j,r,l}$ is smaller) or more diverse gradient directions (many diverging $\gamma_{j,r,l}$) due to the non-increasing nature of the logistic loss function. These characteristics align with the cases described in Lemma \ref{lemma of comparison}, which in our context refers to samples with lower $\gamma_{j,r,l}$ that correspond to yet-to-be-learned features. Therefore, BADGE is well-grounded in the principles uncovered by our theoretical analysis. A more detailed discussion is provided in Appendix \ref{Appendix: BADGE}.\par

\textbf{Potential Extension to Examine Criteria Preference.}
Our results of test error is based on the conditions that there is a clear learning progress disparity between different task-oriented features, under which we see that both NALs inherently favour samples with yet-to-be-learnt features. However, when this disparity does not hold prominently as dicussed in Appendix \ref{app: Case 2: flexible cases}, the behaviors of uncertainty-based and diversity-based sampling may diverge. For example, uncertainty sampling can more precisely prioritize samples with underexplored features when label budgets are not highly constrained. Conversely, diversity sampling may be preferred when label complexity is very limited, as it can enhance the model's ability to capture diverse low-dimensional patterns. This argument is consistent with the claim in recent survey \cite{zhan2021comparative}. Our theory also suggests that when the ``easiness'' of learning various task-oriented features is balanced, uniform random sampling may suffice, without clear advantages for NALs. Additionally, in scenarios of active fine-tuning where the task objective changes, the task-oriented representation could shift, reducing the effectiveness of NAL methods that leverage prior neural representations for sampling. In such cases, random sampling may already be a satisfactory choice. A refined discussion is in Appendix \ref{Appendix: Criteria Preference}.

\textbf{Practical Lessons from Our Theoretical Results.}
Our theoretical analysis yields several important practical insights, as detailed in Appendix \ref{Appendix: Practical Lessons}. First, we find that NALs have the potential to surpass the performance of fully-trained neural networks. As corroborated by the results in \citet{lu2023benign}, NALs can more effectively balance the learning progress across features with different lengths. Additionally, our work suggests that techniques capable of capturing the meaningful orthogonal components of a NN's features or gradients, such as ICA \cite{yamagiwa2023ICA}, could help identify samples underrepresented in NN's latent space. State-of-the-art methods like BADGE \cite{ash2019deep} leverages this idea upon the gradient components.

\section{Conclusion}
In this work, we theoretically demystify and unify the primary query criteria-based NAL methods. We prove they inherently prioritize \hyperlink{msg:perplexing}{\textbf{perplexing samples}} - those with yet-to-be-learned features. This ensures adequate learning of all feature types, underpinning their strong generalization with limited labeled data. Future work can extend our theory to other complex NAL scenarios, such as multi-model committee and stream-based sampling. Additionally, the potential extensions and implications discussed in Section \ref{Section: Potential Expention and Implication} represent valuable directions for further fine-grained exploration.

\section*{Acknowledgements}
We thank the anonymous reviewers for their instrumental comments. DB and HW are supported in part by
the Research Grants Council of the Hong Kong Special Administration Region (Project No. CityU 11206622). WH was partially supported by JSPS KAKENHI (24K20848). TS was partially supported by JSPS KAKENHI (24K02905) and JST CREST (JPMJCR2115, JPMJCR2015).

\section*{Impact Statement}
This paper presents work whose goal is to advance the field of Machine Learning. There are many potential societal consequences of our work, none which we feel must be specifically highlighted here.

\bibliographystyle{plainnat}
\bibliography{main}
\newpage
\onecolumn

\newpage
\onecolumn
\appendix

\section{Additional Related Work: Theory of Feature Learning in Overparameterized Neural Network}

The rapid progress of deep neural networks has prompted growing interest in understanding their underlying theoretical principles, particularly regarding the optimization and generalization properties of overparameterized models. A key development in this area is the study of the Neural Tangent Kernel (NTK) \citep{jacot2020NTK, Chen2020genaralized, cao2018boundsgd, cao2019generalization, cao2020understanding, allenconvergence, chen2021howmuchoverparameterization, Zou2020, huang2020implicit, yilan2021svm, huang2021orthogonal, huang2022GNTK, huang2023analyzing, yanggreg2022featurelearing}. This  has provided powerful insights into the training dynamics of wide neural networks, revealing that their behavior in the $\ell_2$-loss setting closely mirrors the function approximation in reproducing kernel Hilbert spaces (RKHS), where the kernel is associated with the network architecture. However, instead of \textit{feature learning}, this line of research suggest that the parameter update dynamics can be approximated by the first-order Taylor expansion at initialization, where the NN with \textit{wide enough width} can effectively perform linear regression over a prescribed feature map, which cannot characterize the NN's ability to perform \textit{feature learning} \cite{yanggreg2022featurelearing}.

In parallel, an active research direction is the analysis of NN under mean-field regime \cite{mei2018meanfield, mei2019MFdimension}, which allows the network parameters to evolve away from the initialization, thereby enabling \textit{feature learning} for various target functions \cite{Ba2022onstep, Suzuki2023Sparseparties}. Recently, Mean-Field Langevin Dynamics (MFLD) has attracted increased attention, where Gaussian noise is added to the gradient to encourage ``exploration'' \cite{mei2018meanfield, Suzuki2023Convergence}. This framework lifts the learning of finite-width neural networks to an infinite-dimensional optimization problem in the space of probability measures, and by exploiting the convexity of the loss function in this measure space, MFLD can achieve near-optimal global convergence under gradient-based optimization \citep{NitandaSuzuki2017SPGDensembel, mei2018meanfield, Nitanda2021particledualaveraging, Nitanda2022convex, Nitanda2023primaldualanalysis, NitandaFinitesum, nitanda2024improvedcomputationalcomplexity, oko2022particle, Otto2000LSI, Rotskoff2018Interacting, Sirignano2020CentralLimit, suzuki2023uniformintime, Suzuki2023Sparseparties, Suzuki2023Convergence, nitanda2024improvedparticleapproxerror, kim2024symmetric, kim2024transformers}. Despite the remarkable ability of NNs under the MFLD regime to learn complex ``features'', their superior performance still requires a large width at the order of $e^{O(d)}$ \cite{Suzuki2023Sparseparties}. Moreover, the optimization behavior of MFLD differs from the widely-applied SGD-based neural network algorithms, leaving the real-world \textit{feature learning} phenomenon of commonly-utilized deep learning algorithms largely unexplained.\par

To overcome the technical challenges and shed light on the practical \textit{feature learning} observed in GD/SGD-based learning algorithms, the seminal work by \citet{allenzhu2023understanding} takes a step forward. It first attempted to explain the observed success of ensemble methods in deep learning by adopting the NTK framework, but recognized the limitations of this approach. To tackle this challenge and fill the understanding gap, \citet{allenzhu2023understanding} considers a multi-view data model, which is a more complex version of the data model examined in the main body of our work. \citet{allenzhu2023understanding} justify this multi-view data model as plausible theoretical setups by empirically demonstrating the common occurrence of multiple one-task-oriented features in the latent space of ResNet, as shown in their Figures 2-4 and 9. Given the plausibility and suitability of this data setting for theoretical investigations of \textit{feature learning dynamics}, a considerable body of research has delved into examining the capabilities of different learning algorithms under different structured conditions \cite{li2018learning, karp2021local,Gilad2019RandomFeature,cao2022risk, chen2022understanding,chen2023understanding,chen2023clip,chen23brobust,chen2023why, zou2023understanding,li2023theoretical, kou2023benign, kou2023how, kou2023implicit, meng2023benign, huang2023graph,huang2023understanding, chidambaram2023provably,deng2023robust, Frei2023Random, tian2023scan, tian2024joma}. Notably, the width requirement for this line of research is considerably weaker compared to the NTK and MFLD regimes, which allows for a more fine-grained analysis of \textit{feature learning dynamics} based on inner product-based feature direction reconstruction.\par
We believe our work extend this line of research by showing that the two primary criteria-based NALs are inherently prioritizing those underrepresented samples with yet-to-be-learned features. We hope this insight can help the community gain a deeper understanding of the heuristic NAL methods, and develop new principled approaches that can alleviate the data hungriness of deep learning.

\section{Discussions on the Parameter Settings}\label{Appendix: Discussions on the Parameters}
In this section, we motivate the settings of our systems and dicuss the consequences of violating the requirements. 

\subsection{Choice of Systems}

We would like to motivate our choice of systems in detail as below.

\begin{itemize}
    \item \textbf{The system of learning dynamic: $d, n, m, ||\boldsymbol{\mu}||, \sigma_0, \eta$.} The choice of $d, n, m$ aligns with the feature learning line of research \cite{li2018learning, karp2021local, frei2022benign,chen2022understanding,allenzhu2023understanding,chen2023understanding,chen2023clip,chen23brobust,chen2023why, zou2023understanding,li2023theoretical,kou2023how, huang2023graph, kou2023implicit, chidambaram2023provably,deng2023robust,huang2023understanding}, with the aim of ensuring the learning problem is in a small but sufficiently overparameterized regime where the benign overfitting - overparamiterized NN can generalize well when trained to convergence - could occur. This phenomenon is non-trivial against prior belief that overfit is always harmful-the greater the capacity of a model to fit data distribution, the worse the model's test results will be. The system chosen allows for analysis of learning progress of features, as the weak requirement on network width $m$ allows us to conduct a fine-grained analysis based on inner product arguments (i.e., scale analysis of $\gamma, \rho$), which fundamentally differs from the NTK line of research \cite{jacot2020NTK} that requires an infinitely wide network to perform linear regression over a prescribed feature map, rather than learning the features themselves. Moreover, this system ensures a small Signal-to-Noise Ratio (SNR), under which the memorization of noise would become the primary contributor to the volume of the NN's weight matrices, allowing a more balanced and controllable coefficient updates \cite{kou2023benign, meng2023benign}.
    \item \textbf{The system of sampling dynamic: $\widetilde{n},n_0,n^{*},p,\lvert \mathcal{P} \rvert,||\boldsymbol{\mu}_1||, ||\boldsymbol{\mu}_2||,\sigma_p$.} The choice of this system is to (i) avoid the cases where all sampling methods would succeed or fail simultaneously, and (ii) ensure there is a marked learning progress disparity between well-learned and yet-to-be-learned features within the initial stage. The reason to maintain these conditions is to help reveal the underlying rationale behind NAL. It's worth noting that we also provide discussions in Appendix \ref{app:Discussion of data distribution under other conditions} on the general settings beyond the specific system chosen in the main body of the work. In these broader scenarios, there might be various patterns in the learning progress of features.
\end{itemize}
In all, albeit the two systems interact and operate together, they have distinct tasks. The first system is tailored to the non-trivial learning problem at hand. Meanwhile, the choice of the second system aims to help reveal the non-trivial connections between the two NAL methods, by closely tracking the learning progress of task-oriented features after sampling.

\subsection{Consequences of Violating System Requirements}
The following outlines the consequences that may arise where the requirements over the systems are violated:

\begin{enumerate}
    \item The choice of $d$. The large $d$ technically ensures the per-sample loss contributions are in a controllable order during training, preventing any individual's noise from exerting outsized influence on the dynamics. When $d$ decreases with respect to $n, m$, the control of the order over $<\frac{\boldsymbol{\mu}_{l}}{||\boldsymbol{\mu}_{l}||},\frac{\boldsymbol{\xi}_{i}}{|| \boldsymbol{\xi}_{i}||}>,<\frac{\boldsymbol{\xi}_{i}}{|| \boldsymbol{\xi}_{i}||},\frac{\boldsymbol{\xi}_{j}}{|| \boldsymbol{\xi}_{j}||}>,<\mathbf{w}_{j,r}^{(0)},\frac{\boldsymbol{\mu}_{l}}{||\boldsymbol{\mu}_{l}||}>,<\mathbf{w}_{j,r}^{(0)},\frac{\boldsymbol{\xi}_{i}}{|| \boldsymbol{\xi}_{i}||}>,\forall l,i\neq j$ no longer hold with high probability as listed in Appendix \ref{app:preliminary lemmas}, and our technical results on training convergence can not be assured to hold with high chance. Also, a small $d$ leads to a large Signal-to-Noise Ratio (SNR), where the memorization of noise is no longer the dominant factor in the NN's weight matrix volume. This makes the \textit{automatic balance of coefficient updates} techniques in \citet{kou2023benign, meng2023benign} cannot hold, which serves as a convenient lever to observe the bounds on the coefficients and matrix volume update.
    \item The choices of occurrence probability $p$, initial size $n_0$, query size $n^*$, pool size $\lvert \mathcal{P} \rvert$, feature norm $\|\boldsymbol{\mu}_{l}\|$ jointly determine the sampling results.
    \begin{itemize}
        \item Combinations of $p, \|\boldsymbol{\mu}_{l}\|$ reflect the diverse ``easiness'' to learn particular features, leading to varied sampling scenarios as discussed in Appendix \ref{app: Case 2: flexible cases}.
        \item As $p, n_0$ and $n^*$ increase, the chance of getting all features well-learned goes up, reducing NAL's advantage over random sampling as discussed in Appendix \ref{app: Case 2: flexible cases}.
        \item Lower $p$ values (e.g. $p<0.5$) allow NAL to better alleviate labeling efforts by prioritizing the samples with yet-to-be-learned features, but if $p \rightarrow 0$ or $\lvert \mathcal{P} \rvert$ decreases, there might be few yet-to-be-learned features in the pool, limiting NAL's ability to select enough of them to ensure sufficient learning, as discussed in Appendix \ref{app: Case 2: flexible cases}.
        \item Smaller $n_0$ may limit the learning of all features at initial stage, and all sampling methods might behave similarly since all types of features require further learning as discussed in Appendix \ref{app: Case 2: flexible cases}. Decreases in $n_0$, $n^*$, and $\lvert \mathcal{P} \rvert$ would make it challenging to reliably control the proportions of samples as in Lemma \ref{lem:rho n}.
    \end{itemize}

    \item The choices of $\sigma_0$ and $\eta$ aim to ensure effective optimization via GD. As $\sigma_0$ grows, the model has a stronger ``belief'' that is harder to change. While analysis under larger $\eta$ is also doable \cite{lu2023benign}, a small $\eta$ is preferred to better present our main findings.
\end{enumerate}
Amidst parameter variations, we believe our findings are non-trivial. 

\section{Theoretical Results: XOR data version}\label{Appdx:XOR}
In a similar vein to the theoretical results on linearly separable data, we now present a theory specifically tailored for XOR data. The purpose or effect of each result is similar to those obtained for linearly separable data, so we will omit the detailed description of each result. The experiments and proofs can be found in Appendix \ref{app:expXOR} and Appendix \ref{App:proofs of XOR}.
\begin{definition} \cite{meng2023benign}
Let $\mathbf{a}, \mathbf{b} \in \mathbb{R}^d \backslash\{\mathbf{0}\}$ with $\mathbf{a} \perp \mathbf{b}$ be two fixed vectors. For $\boldsymbol{\mu} \in \mathbb{R}^d$ and $\bar{y} \in\{ \pm 1\}$, we say that $\boldsymbol{\mu}$ and $\bar{y}$ are jointly generated from distribution $\mathcal{D}_{\mathrm{XOR}}(\mathbf{a}, \mathbf{b})$ if the pair $(\boldsymbol{\mu}, \bar{y})$ is randomly and uniformly drawn from the set $\{(\mathbf{a}+\mathbf{b},+1),(-\mathbf{a}-\mathbf{b},+1),(\mathbf{a}-\mathbf{b},-1),(-\mathbf{a}+\mathbf{b},-1)\}$.
\label{def_XOR}\end{definition}

\begin{definition} For $l \in \{1, 2\}$, let $\{ \mathbf{a}_1, \mathbf{b}_1 \} \perp \{ \mathbf{a}_{2}, \mathbf{b}_{2}\}  \in \mathbb{R}^d \backslash\{\mathbf{0}\}$, with $ \mathbf{a}_l \perp \mathbf{b}_l$ be two pair of fixed vectors satisfying $\| \mathbf{a}_l \|^2 + \| \mathbf{b}_l \|^2 = \| \boldsymbol{\mu}_l \|_2^2$, where $\| \boldsymbol{\mu}_l \|_2$ represents feature strength. Then each data point $(\mathbf{x}, y)$ with $\mathbf{x}=\left[\mathbf{x}^{(1) \top}, \mathbf{x}^{(2) \top}\right]^{\top} \in \mathbb{R}^{2 d}$ and $y \in\{ \pm 1\}$ is generated from $\mathcal{D}$ as follows:
\begin{itemize}
    \item \textbf{Feature Patch.} For a feeble $p$ satisfying $p < 0.5$, one patch of $\mathbf{x}$ is randomly selected as feature patch, and with high probability (1-$p$) the feature patch $\mathbf{x}_1$ is easy-to-learn feature $ \boldsymbol{\mu}_1$, while only with probability $p$ the feature patch is hard-to-learn feature $\boldsymbol{\mu}_2$. $\boldsymbol{\mu}_l \in \mathbb{R}^d$ and $\bar{y} \in\{ \pm 1\}$ are jointly generated from $\mathcal{D}_{\mathrm{XOR}}(\mathbf{a}_l, \mathbf{b}_l)$.
    \item \textbf{Noise Patch.} The other patch of $\mathbf{x}$ is assigned as a randomly generated Gaussian vector $\boldsymbol{\xi} \sim N\left(\mathbf{0}, \sigma_p^2 \cdot\left(\mathbf{I}-\sum_{l}{(\mathbf{a}_l \mathbf{a}_l^{\top} /\|\mathbf{a}_l\|^2-\mathbf{b}_l \mathbf{b}_l^{\top} /\|\mathbf{b}_l\|^2)}\right)\right)$.
    \item The ground truth label y is synthesized from a Rademacher distribution.
\end{itemize}
    Here we assume the two types of feature differ: $(1-p)\|\boldsymbol{\mu}_1\|_2^4 =\Omega( \sigma_p^4 d n_{0}^{-1}) $ and  $p \|\boldsymbol{\mu}_2 \|_2^4 = O(\sigma_p^4 d n_{0}^{-1})$. Also, we assume the noise cannot completely disturb the learning of features: $ \widetilde{n} \|\boldsymbol{\mu}_l \|_2^4 = \Omega( \sigma_p^4d), l \in \{ 1, 2 \}$.
\label{def_XORdata}\end{definition}

For $(\mathbf{x}, y) \sim \mathcal{D}$ in Definition \ref{def_XORdata}, it's safe to say that:
\[
(\mathbf{x}, y) \stackrel{d}{=}(-\mathbf{x}, y) \text {, and therefore } \mathbb{P}_{(\mathbf{x}, y) \sim \mathcal{D}}(y \cdot\langle\boldsymbol{\theta}, \mathbf{x}\rangle>0)=1 / 2 \text { for any } \boldsymbol{\theta} \in \mathbb{R}^{2 d} .
\]
In other words, all linear predictors will provably fail to learn the XOR-type data $\mathcal{D}$.
\begin{condition}
For certain $\varepsilon, \delta>0$, suppose that
\begin{enumerate}
    \item The initial training size $n_0$, the maximum admissible size after querying $\widetilde{n}$, and the width of neural network $m$ satisfy $n_0  = \Omega(\log (m/\delta),  p^{-1} \log (1/\delta))$, $ \widetilde{n} =O(p^{-1} \sigma_p^4 d \|\boldsymbol{\mu}_2 \|_2^{-4})$, $m = \Omega( \log (\widetilde{n}/\delta))$.
    \item The dimension $d$ satisfies: $d=\widetilde{\Omega}\left(\widetilde{n}^2, \widetilde{n}\|\boldsymbol{\mu}_l\|_2^2 \sigma_p^{-2}\right) \cdot \operatorname{polylog}(1 / \varepsilon) \cdot \operatorname{polylog}(1 / \delta)$, for $l \in \{1, 2\}$.

    \item Random initialization scale $\sigma_0$ satisfies: $\sigma_0 \leq \widetilde{O}\left(\min \left\{\sqrt{n_0} /\left(\sigma_p d\right), n_0\|\boldsymbol{\mu}_l\|_2 /\left(\sigma_p^2 d\right)\right\}\right)$, for $l \in \{1, 2\}$, the learning rate $\eta$ satisfies: $\eta=\widetilde{O}\left(\left[\max \left\{\sigma_p^2 d^{3 / 2} /\left(n_0^2 \sqrt{m}\right), \sigma_p^2 d /(n_0 m)\right]^{-1}\right)\right.$.
    \item The angle $\theta$ between $\mathbf{a}_l + \mathbf{b}_l$ and $\mathbf{a}_l - \mathbf{b}_l$ satisfies $\cos \theta < 1 / 2$, for $\forall l \in \{1, 2\}$.
\end{enumerate}
\label{condition_XOR}\end{condition}
\begin{proposition} \textbf{(Before Querying)}
For any $\varepsilon, \delta>0$, if Condition \ref{condition_XOR} holds, when the probability of the appearance of weak feature in each data point generated from the testing distribution $\mathcal{D}^*$ is $p^*$, then with probability at least $1-2 \delta$, the following results hold at a certain $t=\Omega\left(n_0 m /\left(\eta \varepsilon \sigma_p^2 d\right)\right)$ :
\begin{itemize}
    \item The training loss converges below $\varepsilon$, i.e., $L_S\left(\mathbf{W}^{(t)}\right) \leq \varepsilon$.
    \item The test error achieve sub-optimal constant-level $L_{\mathcal{D}^*}^{0-1}\left(\mathbf{W}^{(t)}\right) \geq p^* \cdot 0.12$.
\end{itemize}
\label{prop_XOR_ini}\end{proposition}
\begin{proposition}
    \textbf{(Querying Stage)} During Querying, under the same conditions as Proposition \ref{prop_XOR_ini}, if $(1-p) \| \boldsymbol{\mu}_{1} \|_2^2-p \| \boldsymbol{\mu}_{2} \|_2^2= \Omega( {\sigma_p}^{2} d^{1/2} n_0^{-1/2} (\log (m/\delta^{\prime}))^{1/2})$ and the size of the sampling pool $\lvert \mathcal{P} \rvert$ is adequately substantial, satisfying: $\lvert \mathcal{P} \rvert =\Omega(  p^{-1} \sigma_p^4 d  \|\boldsymbol{\mu}_2 \|_2^{-4}, p^{-1} \log (1/\delta))$, then with probability at least $1-\Theta(\delta+\delta^{\prime})$, both Uncertainty Sampling and Diversity Sampling pick samples with hard-to-learn features $\boldsymbol{\mu}_2$ in $\mathcal{P}$.
\label{app:propXOR:sample stage:learning on sample}\end{proposition}
\begin{theorem} \textbf{(After Querying)}
     If the sampling size $n^*$ of the two types of Sampling algorithm satisfies $ \dfrac{\hat{C}_1 \sigma_p^4 d}{\|\boldsymbol{\mu}_2 \|_2^4} - \dfrac{p n_0}{2} \leq n^*=\Theta(\widetilde{n} - n_0) \leq \widetilde{n} - n_0$, where $\hat{C}_1$ is some positive constant, under the same conditions as Proposition \ref{app:propXOR:sample stage:learning on sample}, the $\mathcal{D}^*$ and $p^*$ follows the same definitions in Proposition \ref{prop_XOR_ini}, then with probability at least 1 - $\Theta(\delta + \delta^\prime)$, we have the following results hold at a certain $t=\Omega\left((n_0+n^*) m /\left(\eta \varepsilon \sigma_p^2 d\right)\right)$:
    \begin{itemize}
    \item For both the Random Sampling method and Uncertainty Sampling method, the training loss converges to $\varepsilon$, i.e., $L_S\left(\mathbf{W}^{(t)}\right) \leq \varepsilon$.
    \item \textbf{Uncertainty Sampling} and \textbf{Diversity Sampling} algorithms both have negligible test error: $L_{\mathcal{D}^*}^{0-1}\left(\mathbf{W}^{(t)}\right) \leq \exp (\Theta\left( \dfrac{-\widetilde{n}\|\boldsymbol{\mu}_l\|_2^4}{\sigma_p^4 d} \right)), l \in \{ 1,2 \}$.
    \item \textbf{Random Sampling} algorithm would remain the sub-optimal constant-level test error: $L_{\mathcal{D}^*}^{0-1}\left(\mathbf{W}^{(t)}\right) \geq p^* \cdot 0.12$.
    \end{itemize}
\label{theory XOR}\end{theorem}

\section{Discussions over General Scenarios}
Our findings align with the concept of ``Active Learning," where models resemble students (models) actively selecting valuable practice questions (samples) to prepare for exams (tasks). Students prioritize perplexing questions based on high uncertainty of their answers, or rare knowledge points (features), in order to enhance their understanding of yet-to-be-mastered (lack of learning progress) knowledge points (features) in test questions. Similar to students, for most black-box deep neural models, the ``learning progress" of particular ``feature" is not readily available for algorithm developer due to their inherent opacity. From a feature learning view, that's why NAL algorithms need to indirectly prioritize those yet-to-be-learned features, since this is the key for their good generalization ability and achieve \textit{benign overfitting}. Our study shows that uncertainty-based and diversity-based NAL inherently strive to prioritize yet-to-be-learned feature-assisted samples (i.e., \hyperlink{msg:perplexing}{\textbf{perplexing samples}}) via different comparisons in a heuristic manner. We believe future work can figure out if developed interpretable models \cite{yu2023white} reduced labelling efforts by prioritizing \hyperlink{msg:perplexing}{\textbf{perplexing samples}}.\par

Below, we present several discussions regarding general scenarios and the potential wider applicability of our theorems, beyond the specific conditions considered in the main body of our work. It is important to note that our point-mass querying approach and one-round querying settings were adopted to better unveil the inherent principle of query criteria-based NAL algorithms in a rigorous manner, albeit other complex NAL algorithms may be better suited for real-world complex data distribution and corresponding tasks. Note that our multiple task-oriented feature-noise data modellings follow the modellings in \citet{allenzhu2023understanding,chen2022understanding,chen2023understanding,chen2023clip,chen23brobust,chen2023why, zou2023understanding,li2023theoretical,kou2023how, kou2023implicit}, which empirically mirror the latent representation of models like Resnet \cite{allenzhu2023understanding} or transformer \cite{yamagiwa2023ICA, jiang2024origins}.

\subsection{Discussion of the Role of Benign Oscillation}\label{Appendix: Discussion of Oscillation}

In the work by \citet{lu2023benign}, they analyze the role of a large learning rate in the context of feature learning. Their data modeling includes weak features present in each data point, strong features present in a small fraction of data points, and noise. Although our work differs in terms of the data modeling and analysis framework, we might also observe the impact of a large learning rate. In Figures \ref{fig:baseline_lineardata}, \ref{fig:baseline_lineardata_other data}, and \ref{fig:baseline_XOR}, we can see that Uncertainty Sampling and Diversity Sampling algorithms empirically outperform the fully-trained model. Drawing insights from the results in \citet{lu2023benign}, we attribute this phenomenon to the large learning rate, which drives the model to be trained to focus more on weak and rare features. It is worth noting that although our training loss does not exhibit the \textit{benign oscillation} phenomenon mentioned in \citet{lu2023benign}, this probably could be due to the difference in optimization algorithms (GD with logistic loss in our work versus SGD with square loss in \citet{lu2023benign}).

\subsection{Potential Extension over State-of-arts and Criteria-combined NALs: BADGE as an Exampler}\label{Appendix: BADGE}

We believe our analysis can indeed be extended to reveal the success of methods like BADGE \cite{ash2019deep} that combine uncertainty and diversity criteria. We show they share a common principle of prioritizing samples with yet-to-be-learned features. Like the inner product arguments in prior theoretical results \cite{li2018learning, karp2021local,allenzhu2023understanding,chen2022understanding,chen2023understanding,chen2023clip,chen23brobust,chen2023why, zou2023understanding,li2023theoretical,kou2023how, huang2023graph, kou2023implicit, chidambaram2023provably,deng2023robust,huang2023understanding}, our theory characterizes learning progress via the coefficients $\gamma_{j,r,l}$, which high-levelly represent how well the NN has integrated low-dimensional task-oriented patterns into its latent space. We believe the underlying principle of BADGE \cite{ash2019deep} aligns well with this view:
\begin{itemize}
    \item \textbf{Core idea of BADGE.} The key idea behind BADGE is to query samples that exhibit large and diverse gradients within a single batch, achieved through $k$-MEANS ++ or $k$-DPP in the pseudo gradient space.
    \item \textbf{Connection between gradient and latent space of NN.} Since our analysis utilizes the well-applied non-increasing logistic loss, the smaller the magnitude of the latent representation, the larger the magnitude of the gradient embedding will be. Additionally, the diversity of the latent vectors' directions will be preserved in the gradient space. Based on Lemma \ref{app:lemma_wt}, we see that the rows of the latent representations are roughly of the order as $\gamma_{j,r,l}^{(t)}$.
    \item \textbf{BADGE also prioritizes samples with yet-to-be-learned feature.} We now know the BADGE tends to prioritize samples with smaller scale latent representations (smaller $\gamma_{j,r,l}$) or more diverse directions (many diverging $\gamma_{j,r,l}$). These samples correspond to the cases described in Lemma \ref{lemma of comparison}, which in our context refers to samples with lower $\gamma_{j,r,l}$ that have yet-to-be-learned features.
\end{itemize}
Therefore, we claim that BADGE, in the context of our analysis regime, can be explained as a well-motivated NAL method. The key reason is that the two core ideas of BADGE align with the shared underlying rationale of NAL that we has uncovered. One of our future work would serve to give a fine-grained analysis of the success factors behind BADGE, and we also believe our theoretical framework has the potential to extend to the understanding of some other state-of-the-art methods.

\subsection{Extension over Data Distribution under Other Conditions}\label{app:Discussion of data distribution under other conditions}
The theory presented in our main study focuses on a data model that includes weak and rare features, strong and common features, and noise. This setting is motivated by real-world imbalanced datasets, as illustrated in Figure \ref{fig:lions}. However, \textbf{thanks to our general analysis framework}, we can also discuss more general scenarios with broader conditions. In the following sections, we first discuss a theory version that \textbf{relaxes the conditions on feature norms}. This case suggests that rare features may also possess sufficiently discriminative label-related features, such as Simba in the last row of Figure \ref{fig:lions}, even though they are rare occurrences in the overall data distribution. Secondly, we introduce \textbf{a more general theoretical results}. While our discussions below focused on results for linearly separable data, we assert that the same results hold for non-linearly separable XOR data, as the requirements for the parameters are indeed similar. The proofs of all results in this section can be readily obtained based on our results in Appendix \ref{app:Order-dependent Sampling(Querying Analysis)},  \ref{appXOR:Order-dependent Sampling(Querying Analysis)} , \ref{App:Label Complexity-based Test Error Analysis} or \ref{AppXOR:Label Complexity-based Test Error Analysis}. \par
To start, we present the condition-relaxed versions of Proposition \ref{prop_App:order_newdata}, which describe the order situation of samples in $\mathcal{P}$ under relaxed conditions. Here we denote $\tau_l$ as the proportion of $\boldsymbol{\mu}_l$-equipped data in $\mathcal{D}_{n_0}$.
\begin{proposition} (\textbf{Proposition \ref{prop_App:order_newdata} with relaxed conditions on feature norms})
     Under Condition \ref{Con4.1}, there exist $t=\widetilde{O}\left(\eta^{-1} \varepsilon^{-1} m n d^{-1} \sigma_p^{-2}\right)$ that for $\forall \mathbf{x}, \mathbf{x}^\prime \in \mathcal{P} \subsetneq \mathcal{D}$ where $\mathbf{x}$ contains feature patch $y \cdot \boldsymbol{\mu}_2$ while $\mathbf{x}^\prime$ contains feature patch $y^\prime \cdot \boldsymbol{\mu}_1$, with probability at least $1-8 m \exp \left\{-\Theta\left({\left[ \tau_1 \left\|\boldsymbol{\mu}_1\right\|_2^2- \tau_2 \left\|\boldsymbol{\mu}_2\right\|_2^2\right]^2}/{(\sigma_p^{4} d / n_0 )}\right)\right\}$, we have $\mathbf{x}^\prime \preceq^{(t)} \mathbf{x}$. 
\label{prop_App_relaxednorm_con:order_newdata}
\end{proposition}
\textit{Proof of Proposition \ref{prop_App_relaxednorm_con:order_newdata}.} See the proving process of Proposition \ref{prop_App:order_newdata}.\par
This theorem serve as the key to analysis of the querying statistics, as samples with the lower $\underset{j, r}{\mathbb{E}}(\gamma_{j, r, l})$ are \hyperlink{msg:perplexing}{\textbf{perplexing samples}}. Based on the coefficient scale presented in Lemma \ref{app:lem:final coefficient}, we can obtain the probability lower bound for $\mathbf{x}^\prime \preceq^{(t)} \mathbf{x}$, which is 
\begin{equation}
    P(\mathbf{x}^\prime \preceq^{(t)} \mathbf{x}) \geq 1-8 m \exp \left\{-\Theta\left(\frac{\left[ \tau_1 \left\|\boldsymbol{\mu}_1\right\|_2^2- \tau_2 \left\|\boldsymbol{\mu}_2\right\|_2^2\right]^2}{\sigma_p^{4} d / n_0 }\right)\right\}.
\label{eq:flexible}\end{equation}
Thus we can conclude that \hyperlink{msg:perplexing}{\textbf{perplexing samples}} are samples with lower $\tau_l \left\|\boldsymbol{\mu}_l\right\|_2^2$. We then can relax the conditions on feature norms by imposing specific conditions on $p$. Additionally, we can relax both conditions on feature norms and conditions on $p$ to consider a more general case. The upcoming sections will discuss these scenarios in detail.
\subsubsection{Case 1: Relaxed Conditions on Feature Norms}\label{app: Case 1: relaxed conditions on feature norms}

In the main body of our work, we have the conditions on feature norms: $\|\boldsymbol{\mu}_1\|_2^4 =\Omega( \sigma_p^4 d n_{0}^{-1})$, $ \|\boldsymbol{\mu}_2 \|_2^4 = O(\sigma_p^4 d n_{0}^{-1})$ and $\|\boldsymbol{\mu}_1\|_2^2 -  \|\boldsymbol{\mu}_2\|_2^2 = \Omega( {\sigma_p}^{2} d^{1/2} n_0^{-1/2} (\log (m/\delta^{\prime}))^{1/2})$ for the ease of presentations. In this section we provide a theory version that relaxes these requirements (i.e., no discrepancy in terms of feature norms). The essence is that we can impose stricter assumptions on $p$ to ensure there exists a learning progress disparity between the two features. Despite this, the inherent principle of the two-criteria-based NAL approach would still drive the algorithms to preferentially query the samples containing the yet-to-be-learned features. The rigorous rationale behind these will be thoroughly explored in Appendix \ref{app:feature learning and Noise Memorization Analysis} and Appendix \ref{App:Label Complexity-based Test Error Analysis}. Here, we can leverage the deduction results in Appendix \ref{app:feature learning and Noise Memorization Analysis}, Appendix \ref{app:Order-dependent Sampling(Querying Analysis)} and Appendix \ref{App:Label Complexity-based Test Error Analysis} to readily form the following results.\par

\begin{definition} (\textbf{Definition with relaxed conditions on feature norms})
    Let $\boldsymbol{\mu}_1 \perp \boldsymbol{\mu}_2  \in \mathbb{R}^d$ be two fixed feature vectors. Each data point $(\mathbf{x},  y)$, where $\mathbf{x}$ contains two patches as $\mathbf{x}$=$[\mathbf{x}_1^{T},  \mathbf{x}_2^T]^T$ $\in$ $\mathbb{R}^{2d}$ and $y$ $\in \{ -1 , 1 \}$ are generated from the distribution $\mathcal{D}$:
    \begin{itemize}
        \item The ground truth label y is synthesized from a Rademacher distribution.
        \item \textbf{Noise Patch.} One patch of $\mathbf{x}$ is selected as a noise patch $\boldsymbol{\xi}$, synthesized from Gaussian distribution $N(\mathbf{0}, \sigma_p^2 \cdot \mathbf{I})$.
        \item \textbf{Feature Patch.} For a feeble $p$ satisfying $p < O( n_0  \sigma_p^4 d \|\boldsymbol{\mu}_{2}\|_2^{-4} ), (\|\boldsymbol{\mu}_{1}\|_2^2+\|\boldsymbol{\mu}_{2}\|_2^2 )^{-1} ( \|\boldsymbol{\mu}_{1}\|_2^2 + {\sigma_p}^{2} d^{1/2} n_0^{-1/2} (\log (8 m/\delta^{\prime}))^{1/2})$, the remaining patch of $\mathbf{x}$ is selected as label-related feature patch, and with high probability (1-$p$) the feature patch is a common feature $y \cdot \boldsymbol{\mu}_1$, while only with probability $p$ the feature patch is a rare feature $y \cdot \boldsymbol{\mu}_2$.
    \end{itemize}
    Here we only require that the learning of features would not completely disturbed by noise: $\forall l \in \{1, 2\}, \|\boldsymbol{\mu}_{l} \|_2^2 = \Omega( \sigma_p^2 \log (n_0/\delta), n_0^{-1}d \sigma_p^{4})$.
\label{app:relaxfeaturenorm_Def3.1}\end{definition}
The specific condition on the occurrence probability $p$ serves two purposes. Firstly, it ensures that strategy-free passive learning cannot sample enough rare data to adequately learn the rare label-related feature $\boldsymbol{\mu}_{2}$, as observed in the real-world scenario depicted in Figure \ref{fig:lions}. Secondly, it helps distinguish the learning progress between $\boldsymbol{\mu}_{1}$ and $\boldsymbol{\mu}_{2}$.\par
We can prove that three querying algorithms still exhibit \textit{harmful overfitting} at the initial stage.
\begin{proposition} \textbf{(Before Querying)} At the initial stage before querying, $\forall \varepsilon>0$, under Condition \ref{Con4.1}, with probability at least $1-\delta$, there exists $t=\widetilde{O}\left(\eta^{-1} \varepsilon^{-1} m n_0 d^{-1} \sigma_p^{-2}\right)$, the followings hold for all of the three querying algorithms:
\begin{enumerate}
    \item The training loss converges to $\varepsilon$, i.e., $L_S\left(\mathbf{W}^{(t)}\right) \leq \varepsilon$.
    \item The test error remains at constant level, i.e., $L_{\mathcal{D}^*}^{0-1}\left(\mathbf{W}^{(t)}\right) =\Theta(1) \geq 0.12 \cdot p^*$.
\end{enumerate}
\label{app:relax_munorm_prop4.3}\end{proposition}
Then, we can still have a look on the querying stage based on the techniques in Appendix \ref{app:Order-dependent Sampling(Querying Analysis)}. 
\begin{proposition}
    \textbf{(Querying Stage)} During Querying, under the same conditions as Proposition \ref{app:relax_munorm_prop4.3}, then with probability at least $1-\Theta(\delta+\delta^{\prime})$, Uncertainty Sampling and Diversity Sampling would all pick $n^*$ samples that models exhibit lowest $\displaystyle \underset{j,r}{\mathbb{E}} \gamma_{j,r,l}^{(t)}$ (i.e., \hyperlink{msg:perplexing}{\textbf{perplexing samples}}). Moreover, those \hyperlink{msg:perplexing}{\textbf{perplexing samples}} are samples with rare feature $\boldsymbol{\mu}_2$.
\label{app:relax_munorm_prop:sample stage: learning on sample}\end{proposition}
Similar to the theories presented in the main body of our study, we can establish the following theorem.
\begin{theorem} \textbf{(After Querying)}
     If the sampling size $n^*$ of the three querying algorithms satisfies $ C_1 \sigma_p^4 d\|\boldsymbol{\mu}_2 \|_2^{-4} - p n_0 / 2 \leq n^*=\Theta(\widetilde{n} - n_0) \leq \widetilde{n} - n_0$, where $C_1$ is some positive constant. Then for $\forall \varepsilon>0$, under the same conditions as Proposition \ref{prop:sample stage: learning on sample},  with probability more than 1 - $\Theta(\delta + \delta^\prime)$, there exists $t=\widetilde{O}\left(\eta^{-1} \varepsilon^{-1} m (n_0+n^*) d^{-1} \sigma_p^{-2}\right)$ such that:
    \begin{itemize}
    \item For all of the three querying algorithms, the training loss converges to $\varepsilon$, i.e., $L_S\left(\mathbf{W}^{(t)}\right) \leq \varepsilon$.
    \item \textbf{Uncertainty Sampling} and \textbf{Diversity Sampling} algorithms have negligible near Bayes-optimal test error: $L_{\mathcal{D}^*}^{0-1}\left(\mathbf{W}^{(t)}\right) \leq \exp (\Theta\left( \dfrac{-\widetilde{n}\|\boldsymbol{\mu}_l\|_2^4}{\sigma_p^4 d} \right)), l \in \{ 1,2 \}$.
    \item \textbf{Random Sampling} algorithm would remain constant order test error: $L_{\mathcal{D}^*}^{0-1}\left(\mathbf{W}^{(t)}\right) = \Theta(1) \geq 0.12 \cdot p^*$.
    \end{itemize}
    
\label{app:relax_munorm_theory main}\end{theorem}

\subsubsection{Case 2: Flexible Cases}\label{app: Case 2: flexible cases}

Indeed, we can relax both the conditions on feature norms and the conditions on $p$ to explore more general cases. By (\ref{eq:flexible}), if $\tau_1 \left\|\boldsymbol{\mu}_1\right\|_2^2 \approx \tau_2 \left\|\boldsymbol{\mu}_2\right\|_2^2$, the learning progress of the two types of features would be alike (i.e., $\underset{j, r}{\mathbb{E}}(\gamma_{j, r, 1}) \approx \underset{j, r}{\mathbb{E}}(\gamma_{j, r, 2})$), and we cannot clearly observe which type of feature-equipped samples are likely to be queried. Thanks to our sample-complexity analysis regimes in Appendix \ref{App:Label Complexity-based Test Error Analysis}, we can clearly examine two scenarios at the initial stage based on (\ref{lem:rho n}) and Lemma \ref{app:lemma for thm}: 
\begin{itemize}
    \item \textit{Benign Overfitting}: if $\tau_l \|\boldsymbol{\mu}_l\|_2^{4} \geq 2C_1 \sigma_p^4 d n_0^{-1}$, the learning of $\boldsymbol{\mu}_l$-equipped data would be adequate, and the test error of algorithms achieve Bayes-optimal.
    \item \textit{Harmful Overfitting}: if $\tau_l \|\boldsymbol{\mu}_l\|_2^{4} \leq 2C_2/3 \sigma_p^4 d n_0^{-1}$, the learning of $\boldsymbol{\mu}_l$-equipped data would be inadequate, and the test error of algorithms remains constant level.
\end{itemize}
Then, we can list some cases with certain $p$ ($\tau_2=\Theta(p)$ by Lemma \ref{lem:rho n})  $, \|\boldsymbol{\mu}_l\|_2, l \in \{1, 2\}$ in our analysis regime:
\begin{enumerate}
    \item When the learning of $\boldsymbol{\mu}_1$ and $\boldsymbol{\mu}_2$ are all adequate, we can conclude that $n_0$ is already sufficient for training in this case.
    \item When the learning of $\boldsymbol{\mu}_1$ and $\boldsymbol{\mu}_2$ are all inadequate at the initial stage, all querying algorithms (i.e., Random Sampling, Uncertainty Sampling and Diversity Sampling) can help leverage learning of features. While our theory indicates that the two NAL algorithms would tend to prioritize samples with comparatively poorer learned feature (i.e., $\{ \boldsymbol{\mu}_l \mid \tau_l \|\boldsymbol{\mu}_l\|_2^{4} = \min (\tau_1 \|\boldsymbol{\mu}_1\|_2^{4}, \tau_2 \|\boldsymbol{\mu}_2\|_2^{4}) \}$), the difference in generalization ability between Random Sampling and the two NAL algorithms would depend on certain parameters (i.e., $p, n^*, \lvert \mathcal{P} \rvert, \|\boldsymbol{\mu}_1\|_2, \|\boldsymbol{\mu}_2\|_2$).
    \item When the learning of $\boldsymbol{\mu}_{l_1}$ is adequate while the learning of $\boldsymbol{\mu}_{l_2}$ is inadequate ($l_1 \ne l_2 \in \{1, 2\}$), we have the following cases based on our theory:
    \begin{itemize}
        \item If $\tau_{l_1} \left\|\boldsymbol{\mu}_{l_1}\right\|_2^2 \approx \tau_{l_2} \left\|\boldsymbol{\mu}_{l_2}\right\|_2^2$, the prioritization by two NAL algorithms is not obvious, and they would perform similarly to Random Sampling.
        \item If $\tau_{l_1} \left\|\boldsymbol{\mu}_{l_1}\right\|_2^2 > \tau_{l_2} \left\|\boldsymbol{\mu}_{l_2}\right\|_2^2 $, two NAL algorithms would tend to prioritize \hyperlink{msg:perplexing}{\textbf{perplexing samples}} (i.e., samples with $\boldsymbol{\mu}_{l_2}$), and their prioritization has lower probability bound in (\ref{eq:flexible}). Meanwhile, the difference in generalization ability between Random Sampling and the two NAL algorithms would depend on certain parameters (i.e., $p, n^*, \lvert \mathcal{P} \rvert, \|\boldsymbol{\mu}_1\|_2, \|\boldsymbol{\mu}_2\|_2$). Specifically, under Condition \ref{Con4.1}, Definition \ref{Def3.1} and Definition \ref{app:relaxfeaturenorm_Def3.1} provide two parameter settings satisfying $\tau_{l_1} \| \boldsymbol{\mu}_{l_1} \|_2^2-\tau_{l_2} \| \boldsymbol{\mu}_{l_2} \|_2^2=\Omega( {\sigma_p}^{2} d^{1/2} n_0^{-1/2} (\log (m/\delta^{\prime}))^{1/2})$, where the two NAL algorithms succeed while Random Sampling fails (i.e., Theorem \ref{theory main} and Theorem \ref{app:relax_munorm_theory main}). Other general scenarios can also be rigorously analyzed with the prioritization probability lower bound in (\ref{eq:flexible}) and permutation probability.
    \end{itemize}
    \item Other cases would be similar to the second or third case (i.e., where $\exists l \in \{1, 2\}, 2C_2/3 \sigma_p^4 d n_0^{-1} \leq \tau_l \|\boldsymbol{\mu}_l\|_2^{4} \leq  2C_1 \sigma_p^4 d n_0^{-1}$).
\end{enumerate}
In real-world scenarios, the pool-based setting often resembles a wide range of flexible cases. From the perspective of feature learning, our theoretical observations suggest that the occurrence probability and strength of different task-specific features can profoundly impact the efficiency of NAL algorithms.

\subsection{Cases of Criteria Preference}\label{Appendix: Criteria Preference}
Our work has uncovered a non-trivial connection between the two query criteria-based NAL methods. Specifically, they share a sufficient condition - which we also called it as the shared principle - that is vital to the success of NAL methods, which holds when the learning progress of the well-learned features greatly surpasses the learning of the yet-to-be-learned features to a certain degree

$${\underbrace{\Theta(\underset{j,r}{\mathbb{E}}(\gamma_{j,r,1}))-\Theta(\underset{j,r}{\mathbb{E}}(\gamma_{j,r,2}))}_{\text{Learning Progress Disparity: }\text{well-learned Feature } vs. \text{yet-to-be-learned Feature}}}>\max_{j,r,l}|<\mathbf{w}_{j,r}^{(t)},\mathbf{z}_l>|.$$

However, as discussed in Appendix \ref{app: Case 2: flexible cases} above, when this shared sufficient condition (or principle) does not hold, the behaviors of the two heuristic criteria-based sampling methods may differ.\par

\textbf{Cases favoring uncertainty-based Sampling.} Specifically, when the label budget is not highly limited and there is sufficient opportunity to capture all feature types, uncertainty-based sampling may be preferred. Our analysis shows that compared to uncertainty sampling, diversity sampling has a stricter requirement, with a less than 1 scalar $(\tau_1 - \tau_2)$ in the left side of inequalities (\ref{Eq:OmegaD}) and (\ref{EqXOR:OmegaD}), versus (\ref{Eq:OmegaC}) and (\ref{EqXOR:OmegaC}) for uncertainty sampling. This allows uncertainty sampling to more precisely prioritize samples with yet-to-be-learned features, more easily ensuring adequate learning across all feature types.\par

\textbf{Cases favoring diversity-based Sampling.} However, when label complexity is quite limited (as per Appendix \ref{app: Case 2: flexible cases}) where all task-oriented features require further labelling budget, we may favor diversity-based sampling. Despite all sampling algorithms increasing test accuracy by addressing insufficient learning of certain features, diversity sampling's efficiency in obtaining diverse features could enhance the model's ability to grasp diverse low-dimensional patterns. This in turn could enrich generalization, even when the test distribution differs from training.\par

Our statements here align with discussions in the recent survey \cite{zhan2021comparative}. We believe this nuanced perspective deserves further exploration. \par

\textbf{Cases favoring Strategy-free Random Sampling.} As discussed in Appendix \ref{app: Case 2: flexible cases}, our theory suggests that when $\tau_{1}\|\boldsymbol{\mu}_{1}\|^2\approx\tau_{2}\|\boldsymbol{\mu}_{2}\|^2$ where $\tau_{l}$ denotes the proportion of $\boldsymbol{\mu}_{l}$ in training set, it indicates a balanced ``easiness'' to learn multiple task-oriented features. In such cases, the learning progress of these features tends to be similar, and the prioritization by NAL methods may not be clearly evident. In other words, if there is no distinct gap between well-learned and yet-to-be-learned features, uniform sampling might be sufficient, and the advantage of NAL methods only emerges when there is a clear distinction of ``learning easiness'' among various task-oriented feature categories.\par

Additionally, when it comes to the scenarios of active fine-tuning, where the task objective is heavily or slightly changing. In such situations, the task-oriented low-dimensional patterns may shift, and the model's optimal representation could differ from before. As a result, NAL methods that leverage prior neural representations for sampling may not be as effective, and uniform sampling could be a satisfactory choice.

\subsection{Discussions of Multi-round NALs}\label{Appendix: Multi-round}
Our theory suggests that the core principle underlying both NAL methods is their tendency to prioritize the selection of samples containing yet-to-be-learned features. This fundamental characteristic is not inherently tied to the single-round setting, but rather reflects an intrinsic property of the two primary criteria-based NAL family.

In the multi-round iterative process, the learning progress of different features may diverge across rounds and potentially align with the various cases discussed in Appendix \ref{app: Case 2: flexible cases}. However, we expect the NAL methods to continue performing well due to their innate focus on prioritizing the selection of samples containing yet-to-be-learned features.

\subsection{Discussions of Practical Lessons of our Results}\label{Appendix: Practical Lessons}

Here are some key takeaways of our theory:
\begin{itemize}
    \item \textbf{Potential of NAL to surpass fully-trained NN.} As discussed in Appendix \ref{Appendix: Discussion of Oscillation}, and corroborated by the results in \citet{lu2023benign}, fully-trained neural networks tend to learn hard-to-learn features in an inefficient manner, as they place disproportionate emphasis on the easy-to-learn ones. In contrast, our analysis suggests that the NAL approach prioritizes samples with low $\gamma_{j,r,l}$, making it more likely to achieve a balanced rise in $\gamma_{j, r, 1}$ and $\gamma_{j,r,2}$ during the new round of training. This implies that NAL has a better chance of ensuring sufficient learning of all features within a certain number of iterations, compared to fully-trained neural networks. This conclusion is partially validated by the empirical results presented in our Figures \ref{fig:baseline_lineardata}, \ref{fig:baseline_lineardata_other data}, and \ref{fig:baseline_XOR}, where the NALs outperform the neural networks. In real-world settings, we conjecture that NAL might have this potential when the neural network is sufficiently overparameterized and has the capacity to capture all relevant patterns of the problem instances within limited iterations.
    \item \textbf{Care orthogonal components of features or gradients.} Our theory suggests that if techniques can be adopted to capture the meaningful orthogonal components of a neural network's features or gradients (e.g., using ICA \cite{yamagiwa2023ICA}), then the samples with low-magnitude latent feature components or high-magnitude gradient components might align with the perplexing samples in our work. This is because our theory indicates that yet-to-be-learned features are often underrepresented in the neural network's latent space, and if the loss is non-increasing, the length in the latent space might be inversely proportional to the length in the corresponding gradient space. Notably, existing state-of-the-art methods such as BADGE \cite{ash2019deep} also leverage a similar idea with respect to the gradient component of the last layer.
    \item \textbf{Incorporate Signal-to-Noise Ratio (SNR) Measurement.} Our discussions in Appendix \ref{app:Discussion of data distribution under other conditions} denote that the perplexing samples are often characterized by their rarity and low SNR (the scale ratio between feature and noise). Techniques, whether learnable or unlearnable, that can accurately or approximately measure the SNR of multiple task-oriented features in a NN's latent space may help develop a principled NAL approach, and for specific tasks and datasets, it may be feasible to develop such task-oriented SNR measurement methods.
\end{itemize}

\section{Additional Experiments}\label{Expe:XORdata}

\subsection{Sampling Information of Main Results}\label{app:Sampling Information}

\begin{figure}[H]
  \centering
\includegraphics[width=0.73\linewidth]{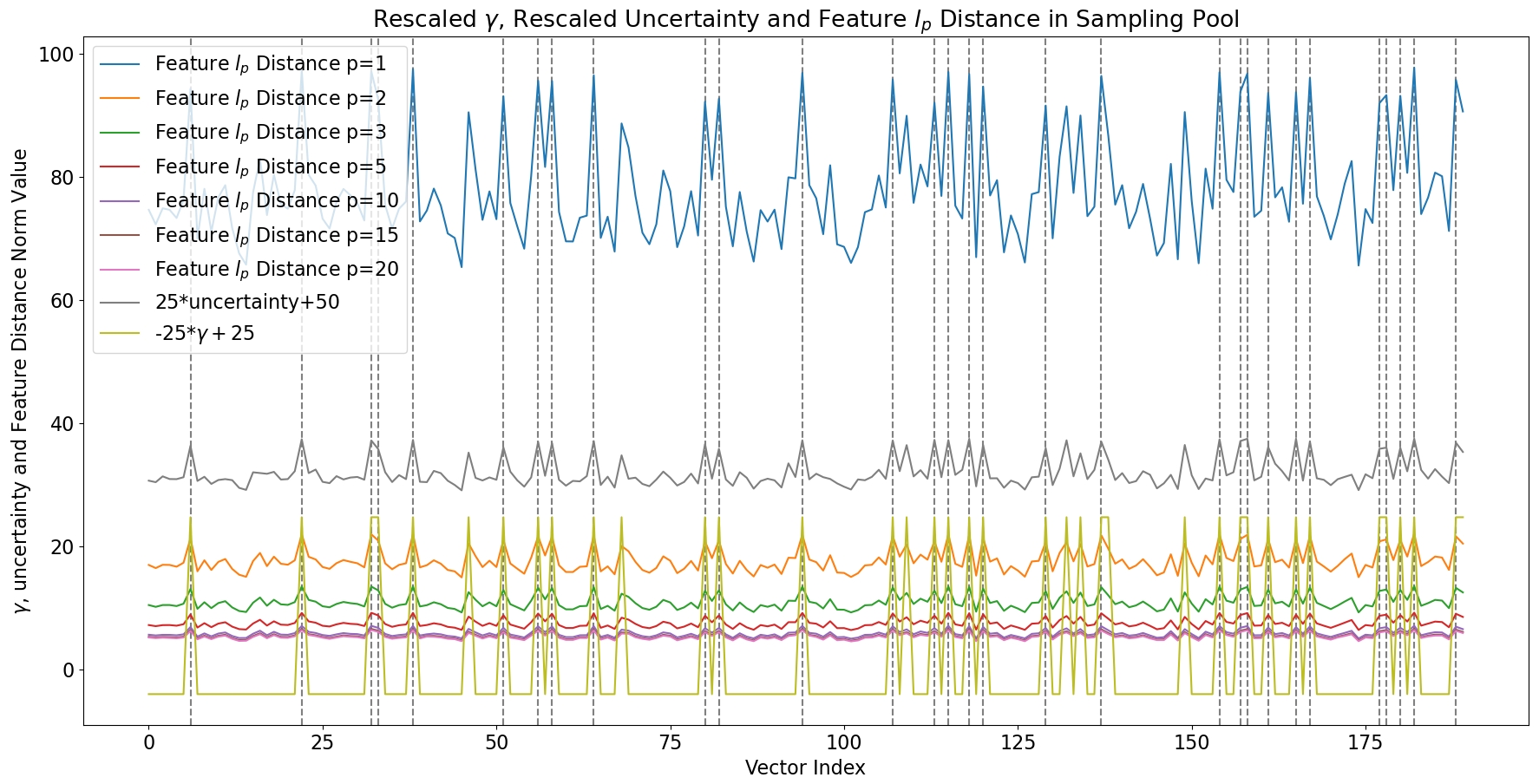}
     \caption{Rescaled $\gamma$ ($\gamma={\mathbb{E}} \gamma_{j,k,l}^{(t)}$), Uncertainty (i.e., $-$Confidence Score) and Feature Distance (with various $p$ of $l_p$ norm) of the samples in sampling pool $\mathcal{P}$, where $\gamma$ represents the learning progress of feature in particular sample. The dashed line in the graph represents the top 30 samples with the highest Feature Distance.}
    \label{fig:Lp_norms}
\end{figure}
\begin{figure}[H]
\centering
\subfigure[Random Sampling]{\includegraphics[width=0.49\textwidth]{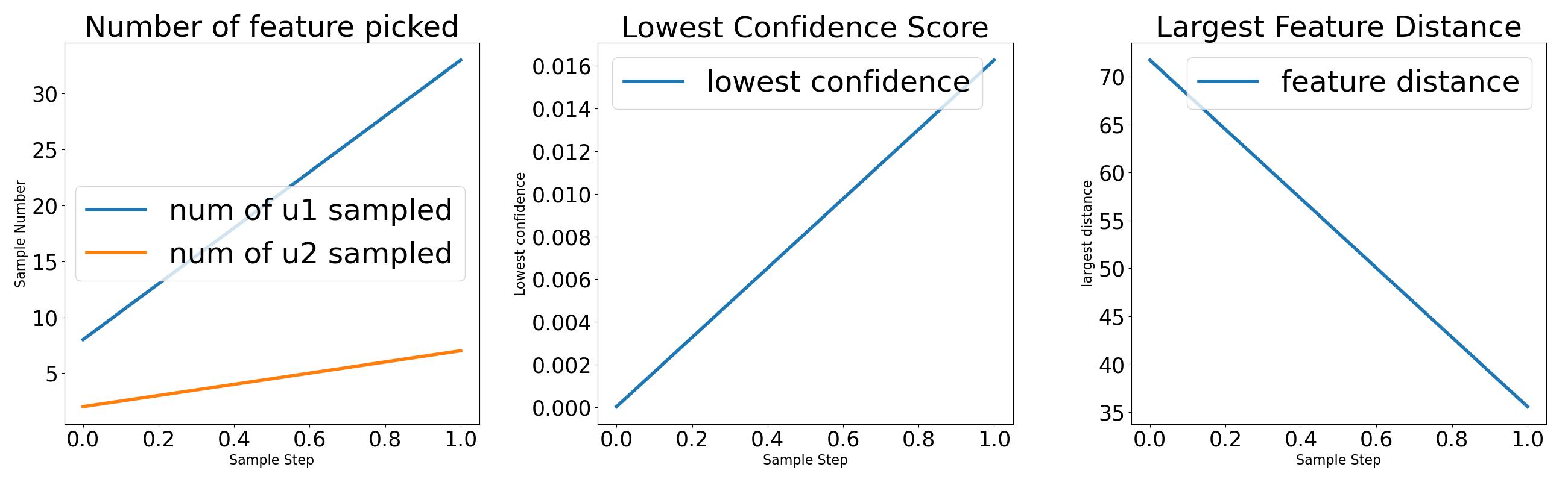}}
\subfigure[Uncertainty Sampling]{\includegraphics[width=0.49\textwidth]{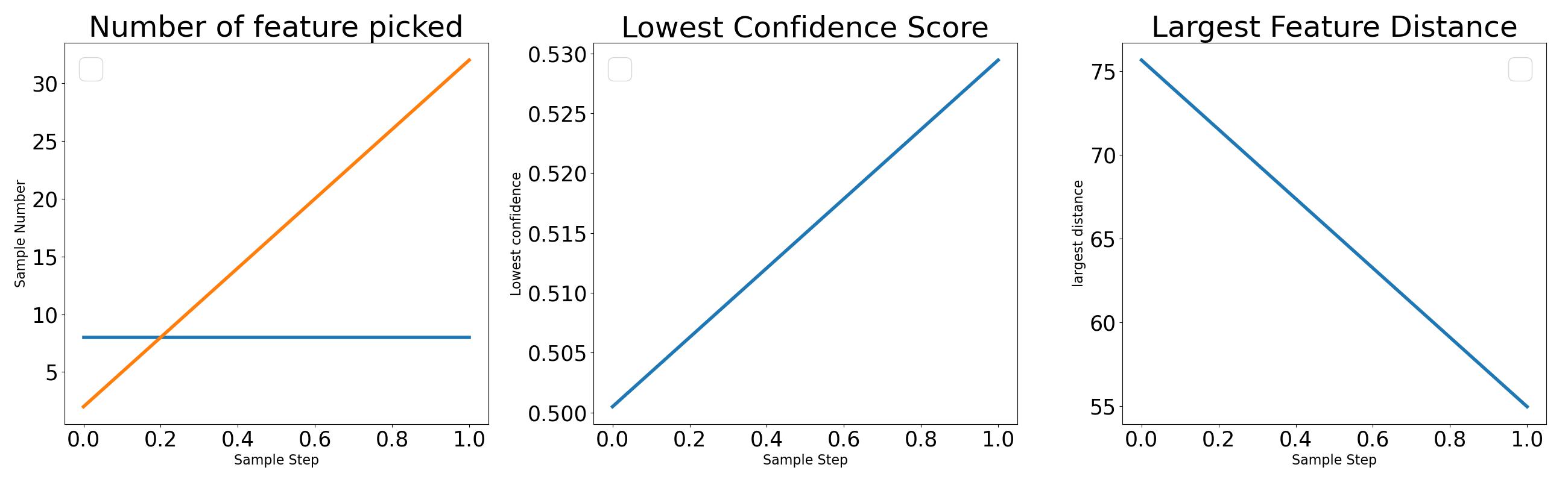}}
\subfigure[Diversity Sampling]{\includegraphics[width=0.5\textwidth]{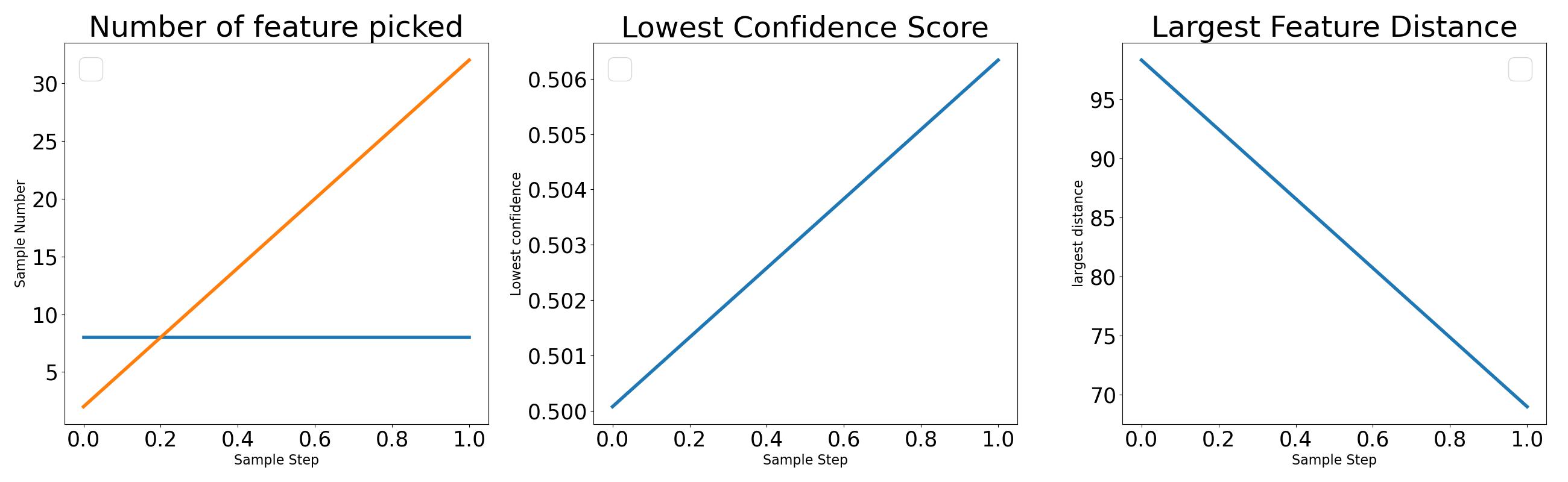}}
\caption{Comparison of querying information between two NAL algorithms, illustrating training size changes in labeled data sets, Confidence Score, and Feature Distance before and after querying.}
\label{fig:sample infor}\end{figure}
Here we give more visualized details of the querying stage. The parameter settings are the same in Section \ref{section:EXP}. Figure \ref{fig:Lp_norms} visualized the rescaled $\underset{j,k,l}{\mathbb{E}} \gamma_{j,k,l}$, uncertainty(-Confidence Score) and Feature Distance of each samples in the unlabeled sampling pool $\mathcal{P}$, where the dash line corresponds to the top $n^*$ samples based on Diversity Order. It's obvious that regardless of the value $p$, the Uncertainty Order and Diversity Order of samples remain the same, and corresponds to the order of $\underset{j,k,l}{\mathbb{E}} \gamma_{j,k,l}$. This validates our unification claims in Proposition \ref{prop:sample stage: learning on sample}, and Lemma \ref{Lem:order_pool}. Figure \ref{fig:sample infor} makes it clear that the two NAL algorithms successfully obtain those hard-to-learn samples, while Random Sampling hardly obtain hard-to-learn samples as it selects samples in a random manner.

\subsection{Experiments: Data Model under Other Conditions}\label{Experiments: data model under other conditions}

\begin{figure}[H]
\centering
    \subfigure[Full trained model]{\includegraphics[width=0.4975\textwidth]{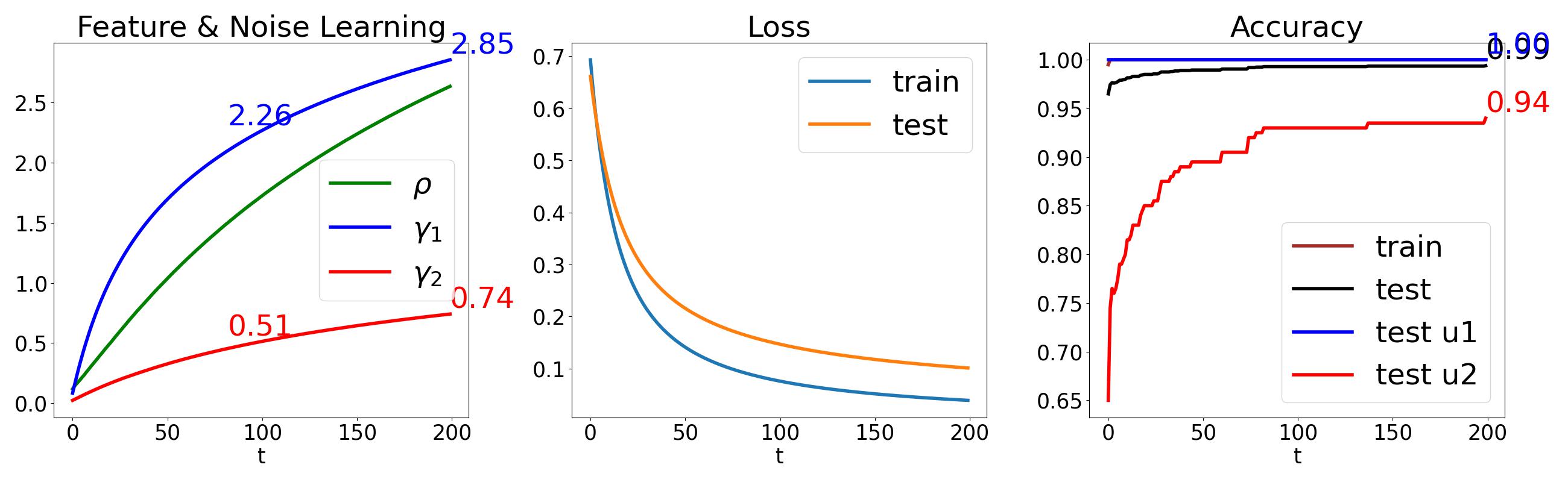}}
\subfigure[Random Sampling]{\includegraphics[width=0.4975\textwidth]{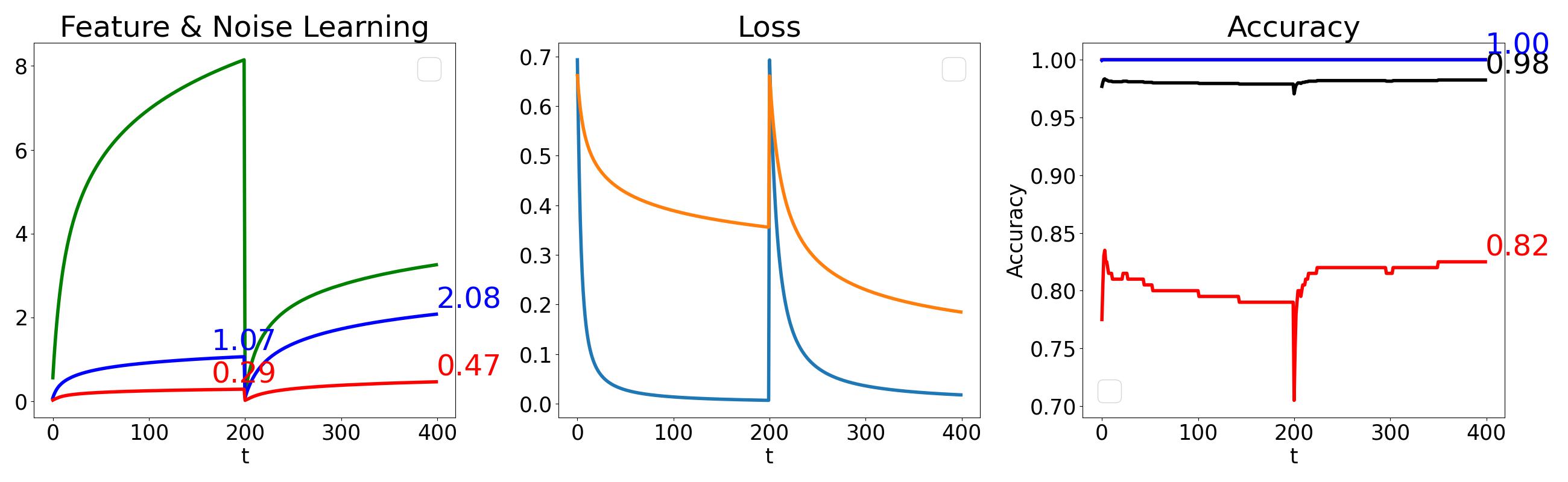}}
\subfigure[Uncertainty Sampling]{\includegraphics[width=0.4975\textwidth]{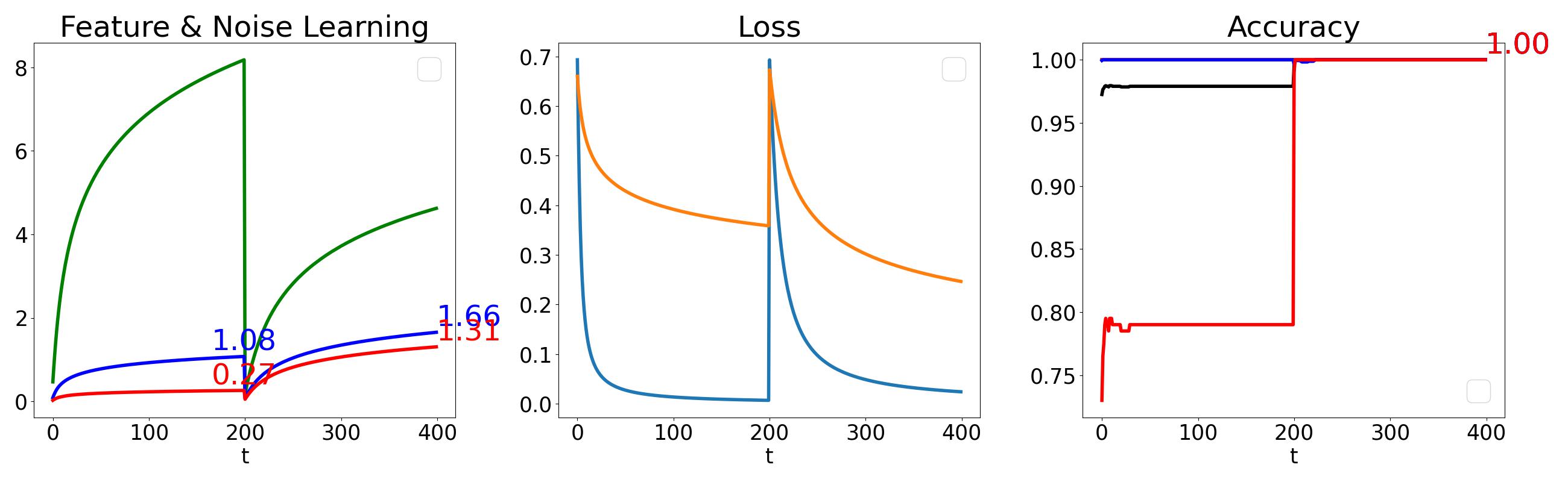}}
\subfigure[Diversity Sampling]{\includegraphics[width=0.4975\textwidth]{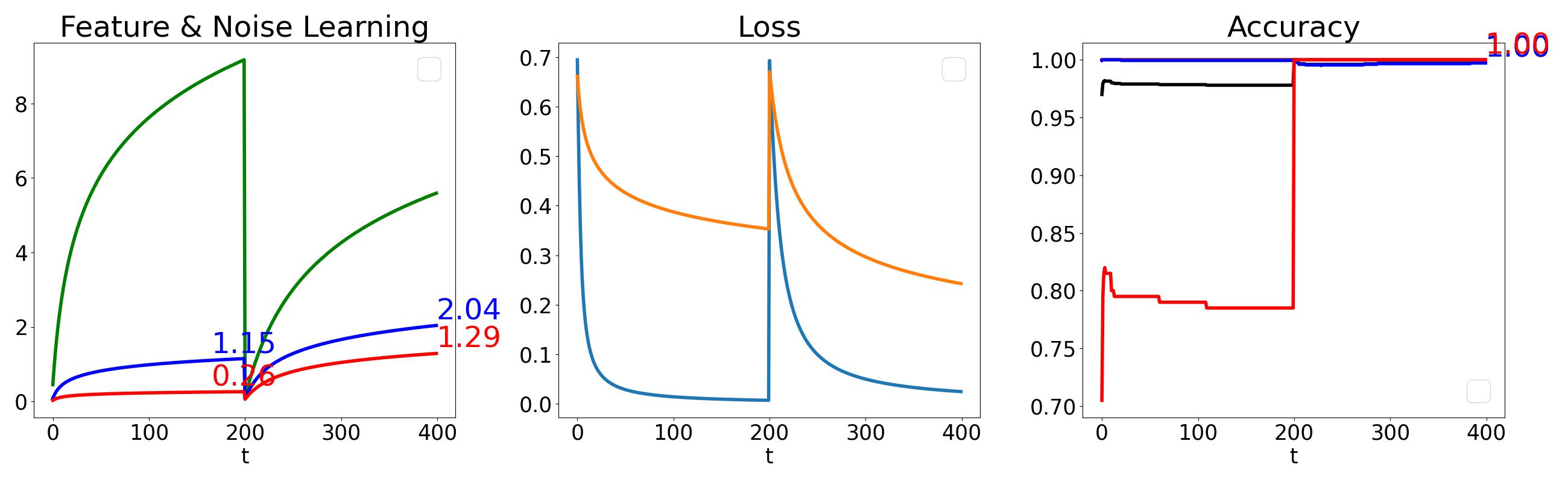}}
\caption{Learning/memorization progress of features and noise ($\gamma_l$ represents $\max_{j,k} \gamma_{j,k,l}^{(t)}$, and $\rho$ represents $\max_{j,k,i} {j,k,i}^{(t)}$), train/test losses, and test accuracy of the full-trained model and the three querying algorithms, with $T^*=200$,  $d=2000$, $\|\boldsymbol{\mu}_1\|=8$, $p=p^*=0.1$, $\|\boldsymbol{\mu}_2\|=8$, $n_{CNN}=200$, $n_0=10$, $n^{*}=30$ and $\lvert \mathcal{P} \rvert = 190$.}
\label{fig:baseline_lineardata_other data}
\end{figure}

\begin{figure}[H]
\centering
\subfigure[Random Sampling]{\includegraphics[width=0.49\textwidth]{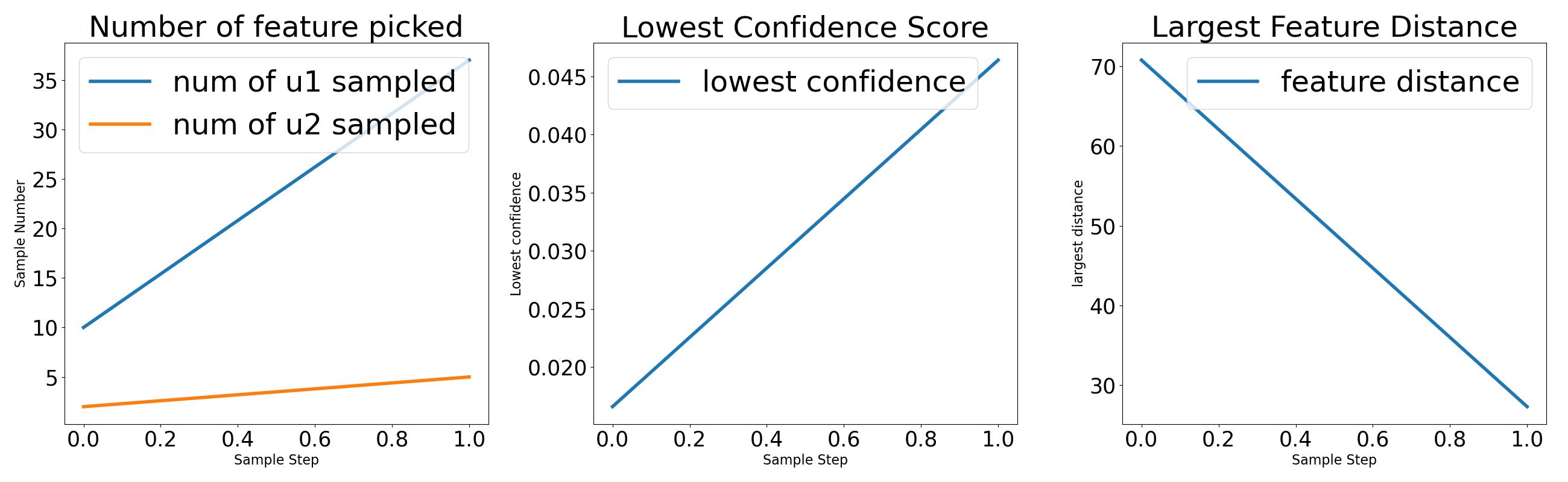}}
\subfigure[Uncertainty Sampling]{\includegraphics[width=0.49\textwidth]{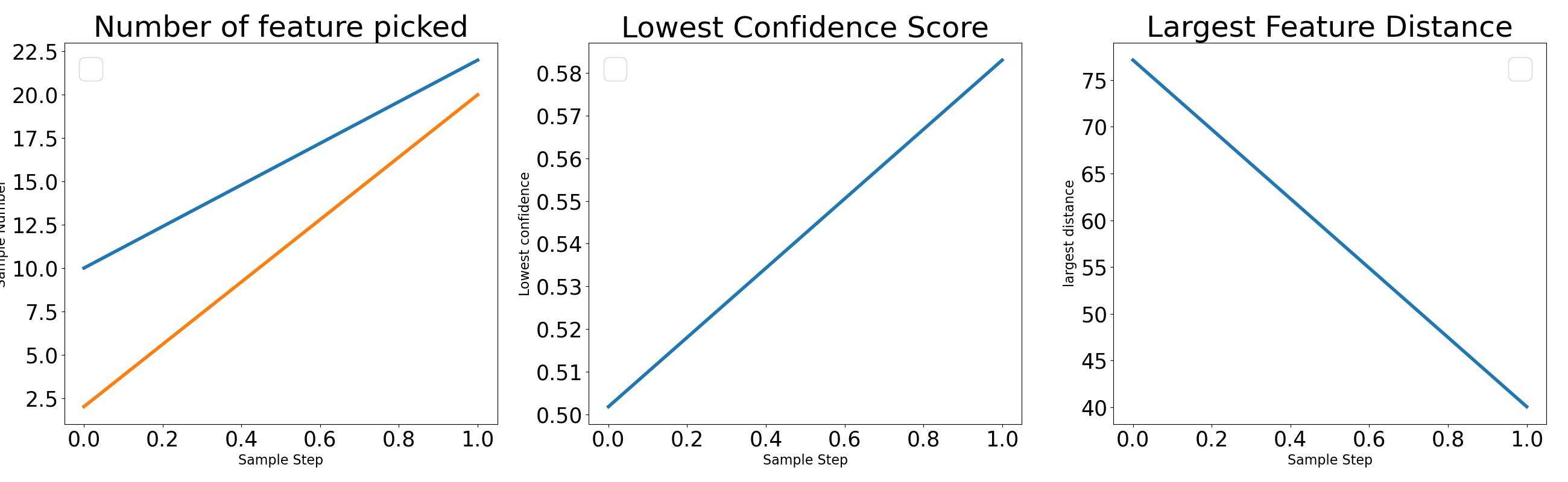}}
\subfigure[Diversity Sampling]{\includegraphics[width=0.50\textwidth]{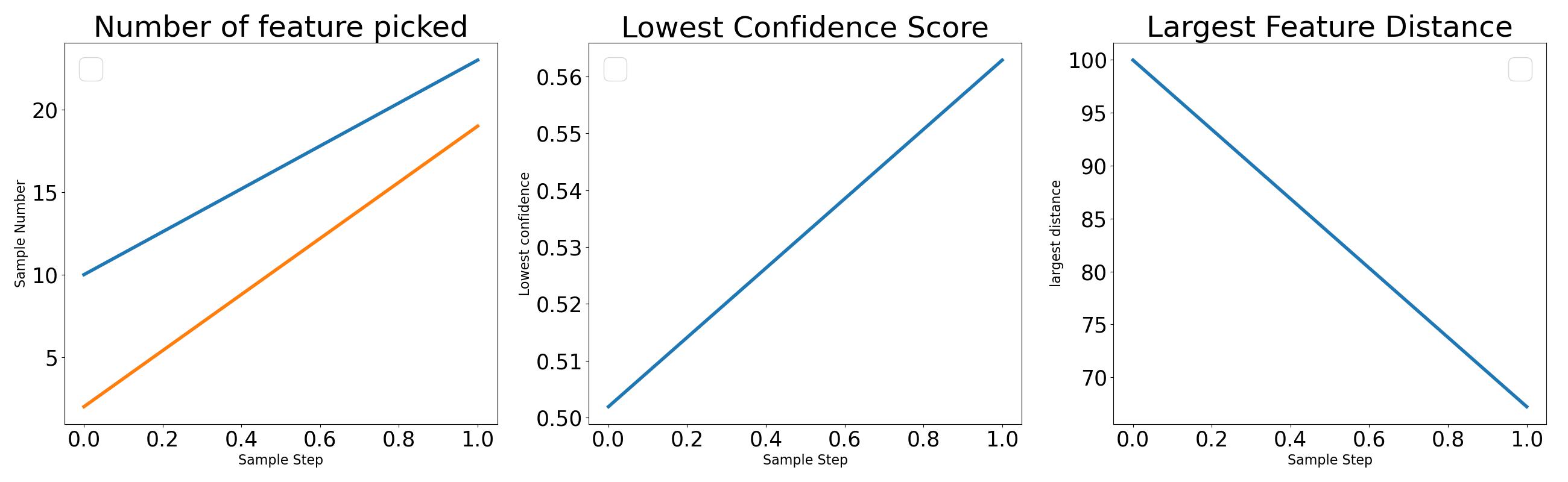}}
\caption{Comparison of querying information between two NAL algorithms, illustrating training size changes in labeled data sets, Confidence Score, and Feature Distance before and after querying. ($T^*=200$,  $d=2000$, $\|\boldsymbol{\mu}_1\|=9$, $p=p^*=0.2$, $\|\boldsymbol{\mu}_2\|=3$, $n_{CNN}=200$, $n_0=10$, $n^{*}=30$ and $\lvert \mathcal{P} \rvert = 190$)}
\label{fig:sample infor_other data}\end{figure}
We investigate the scenario where the strengths (i.e., feature norms) of different features do not vary significantly, as discussed in the main body of our work. Specifically, we set them as the same: $\|\boldsymbol{\mu}_1\|_2=\|\boldsymbol{\mu}_2\|_2=8$. Other parameters are listed as the following: $T^*=200$,  $p=p^*=0.1$, $d=2000$, $n_{CNN}=200$, $n_0=10$, $n^{*}=30$, $\lvert \mathcal{P} \rvert = 190$, $\sigma_p = 1$ and $\sigma_0 = 0.01$. In this case, where $\tau_1\|\boldsymbol{\mu}_1\| < \tau_2\|\boldsymbol{\mu}_2\|$, the \hyperlink{msg:perplexing}{\textbf{perplexing samples}} are those samples equipped with $\boldsymbol{\mu}_2$. It is worth noting that our chosen value of $p=0.1$ is not small enough to satisfy the condition in Definition \ref{app:relaxfeaturenorm_Def3.1}. Instead, our parameter setting falls under the second bullet point of the third case discussed in Appendix \ref{app: Case 2: flexible cases}. Figure \ref{fig:baseline_lineardata_other data} demonstrates the success of both NAL algorithms, while Figure \ref{fig:sample infor_other data} illustrates the sample information. It is clear that both NAL algorithms prioritize the \hyperlink{msg:perplexing}{\textbf{perplexing samples}} more effectively than Random Sampling, resulting in a lower test error rate.

\subsection{Experiments: XOR Data Versions}\label{app:expXOR}

\begin{figure}[H]
\centering
\subfigure[Full-trained 2 layer CNN]{\includegraphics[width=0.49\textwidth]{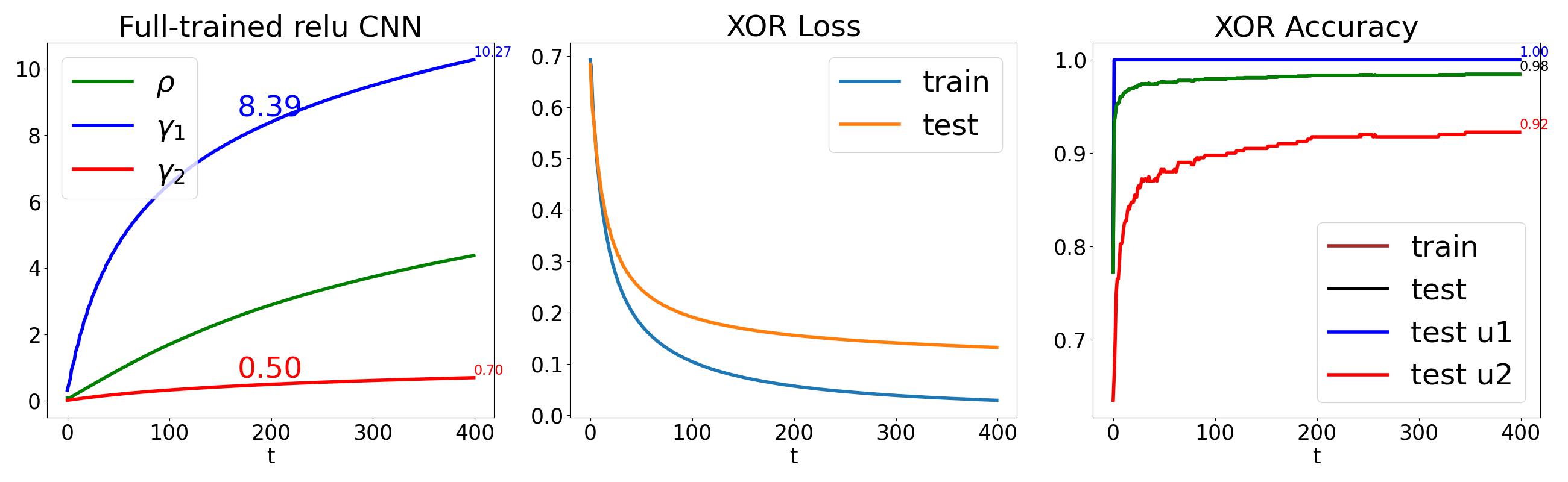}}
\subfigure[Random Sampling]{\includegraphics[width=0.49\textwidth]{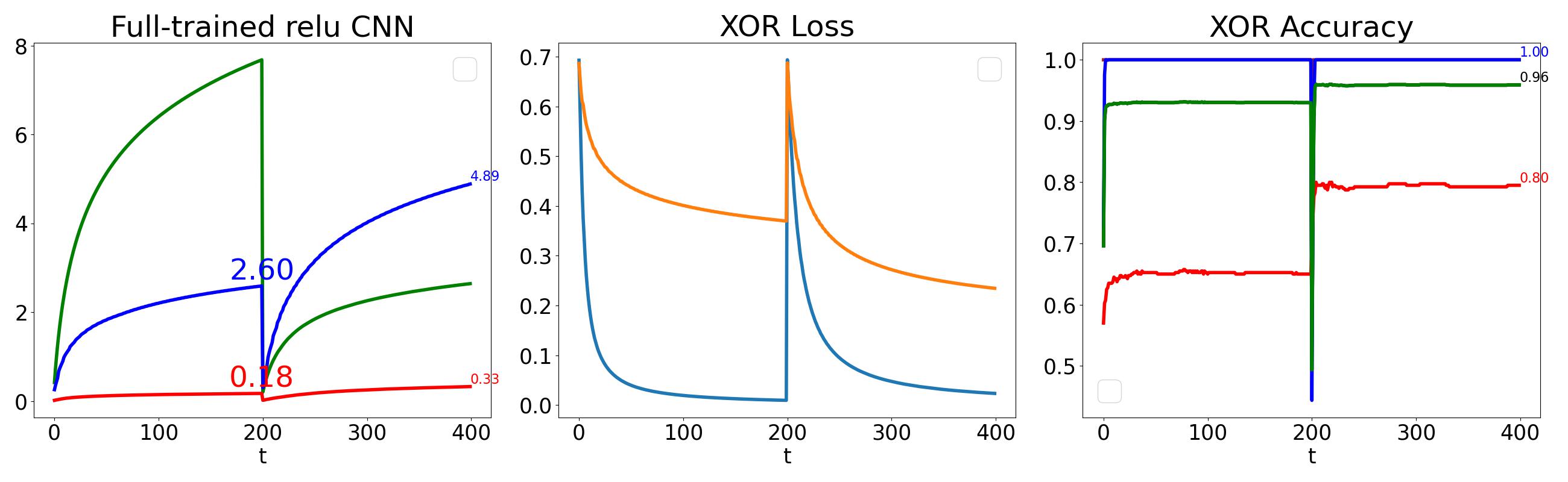}}
\subfigure[Uncertainty Sampling]{\includegraphics[width=0.49\textwidth]{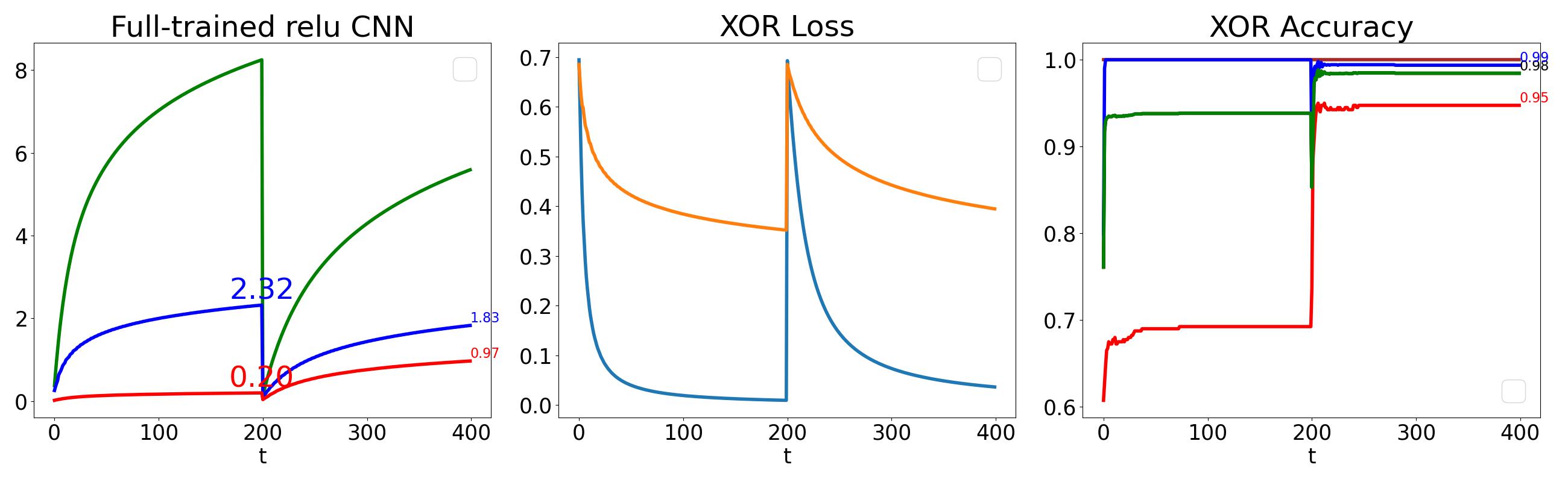}}
\subfigure[Diversity Sampling]{\includegraphics[width=0.49\textwidth]{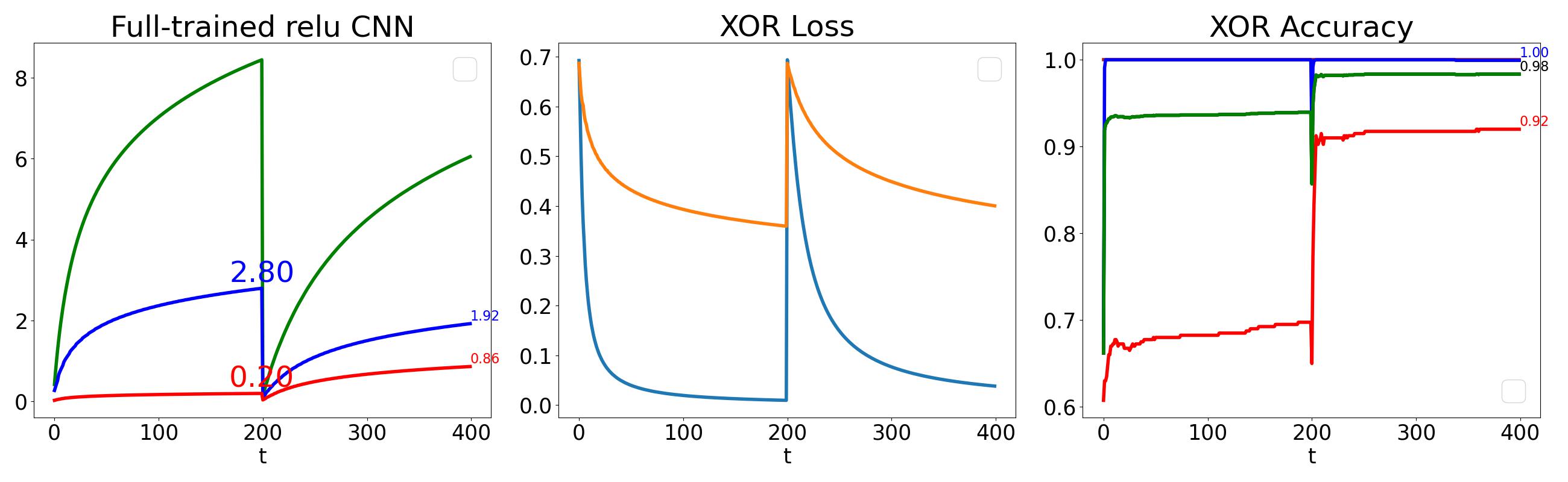}}
\caption{Learning/memorization progress of features and noise ($\gamma_l$ represents $\max_{j,k} \{\gamma_{j,k,\mathbf{u}_l}^{(t)}, \gamma_{j,k,\mathbf{v}_l}^{(t)}\}$, and $\rho$ represents $\max_{j,k,i} \rho_{j,k,i}^{(t)}$), train/test losses, and test accuracy of the full-trained model and the three querying algorithms, with $ \cos{\theta}=0.4, T^*=200$, $d=2000$, $\|\boldsymbol{\mu}_1\|=20$, $p=p^*=0.2$, $\|\boldsymbol{\mu}_2\|=6$, $n_{CNN}=200$, $n_0=10$, $n^{*}=30$ and $\lvert \mathcal{P} \rvert = 190$.}
\label{fig:baseline_XOR}
\end{figure}

\begin{figure}[H]
\centering
\subfigure[Random Sampling]{\includegraphics[width=0.49\textwidth]{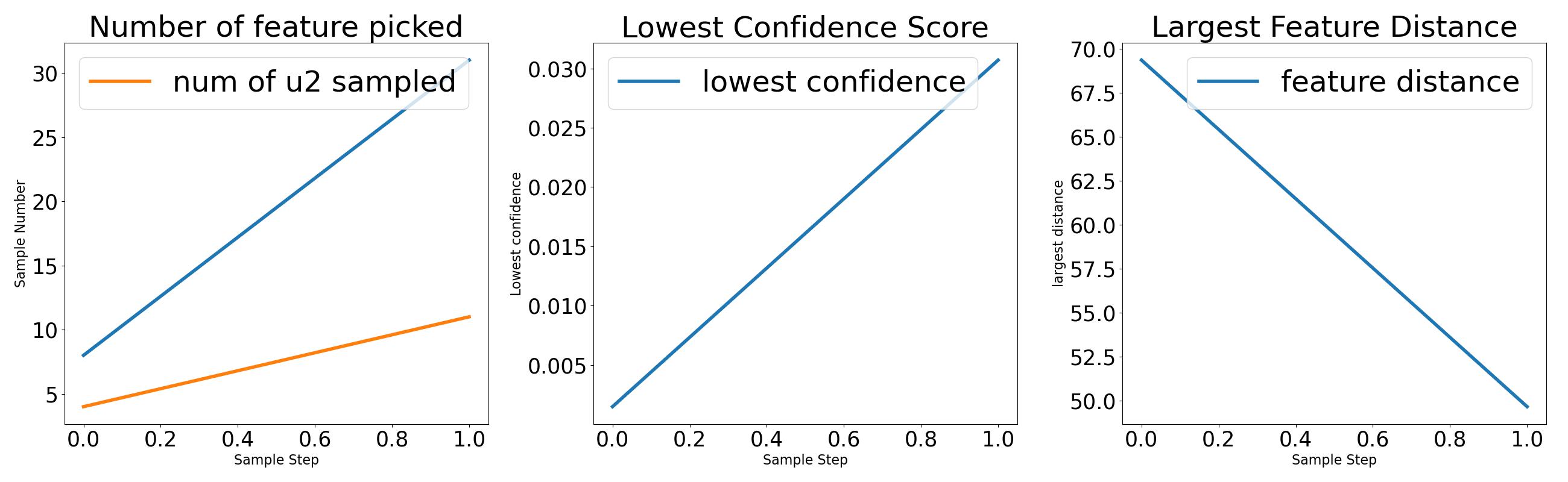}}
\subfigure[Uncertainty Sampling]{\includegraphics[width=0.49\textwidth]{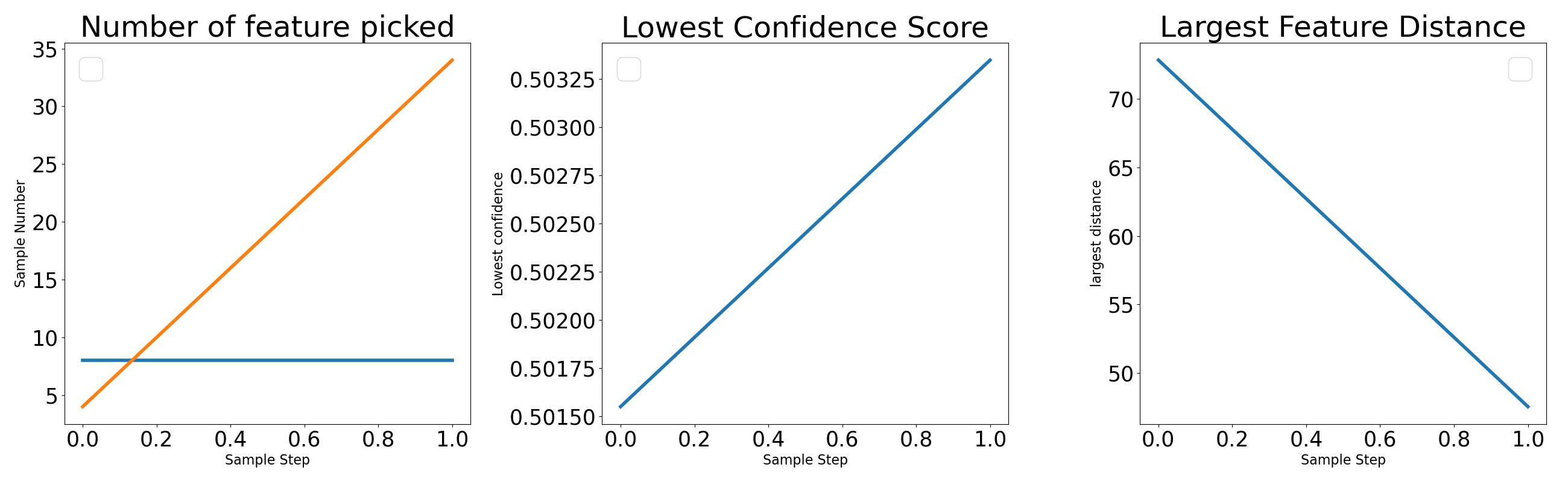}}
\subfigure[Diversity Sampling]{\includegraphics[width=0.50\textwidth]{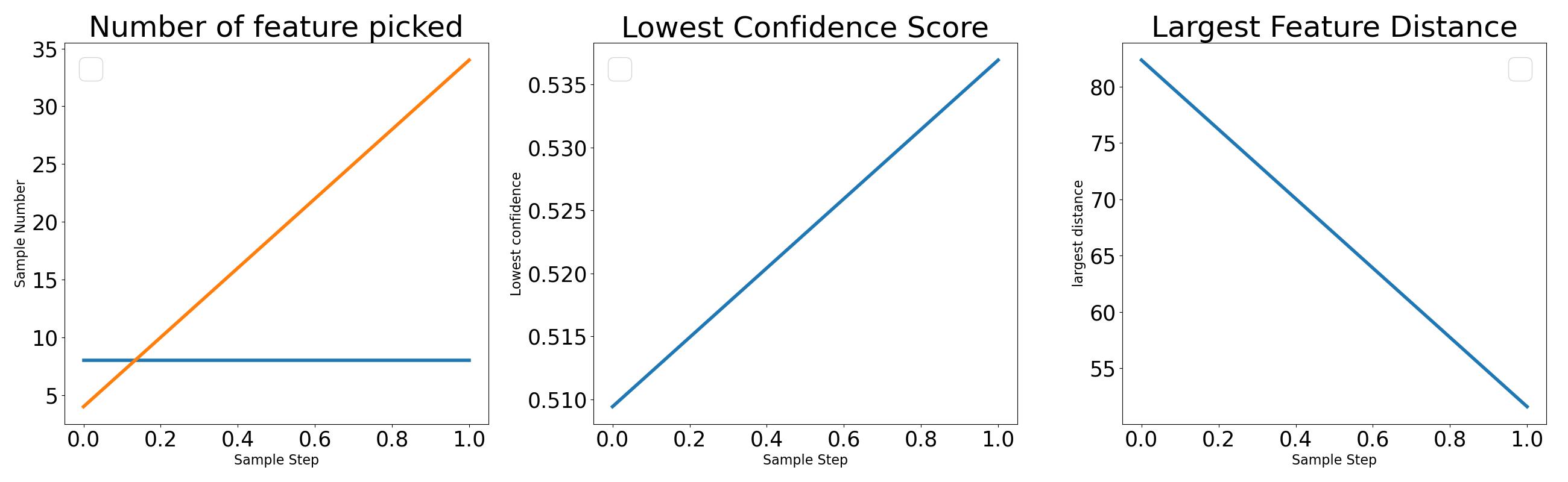}}
\caption{Comparison of querying information between two NAL algorithms over XOR data, illustrating training size changes in labeled data sets, Confidence Score, and Feature Distance before and after querying. ($ \cos{\theta}=0.4, T^*=200$, $d=2000$, $\|\boldsymbol{\mu}_1\|=20$, $p=p^*=0.2$, $\|\boldsymbol{\mu}_2\|=6$, $n_{CNN}=200$, $n_0=10$, $n^{*}=30$ and $\lvert \mathcal{P} \rvert = 190$)}
\label{fig:baseline confidence_XOR}
\end{figure}

We also conduct experiments on XOR data. We set the parameters as: $ \cos{\theta}=0.4, T^*=200$, $d=2000$, $\|\boldsymbol{\mu}_1\|=20$, $p=p^*=0.2$, $\|\boldsymbol{\mu}_2\|=6$, $n_{CNN}=200$, $n_0=10$, $n^{*}=30$ and $\lvert \mathcal{P} \rvert = 190$. Figure \ref{fig:baseline_XOR} and Figure \ref{fig:baseline confidence_XOR} clearly demonstrate that the two NAL algorithms succeed via prioritizing \hyperlink{msg:perplexing}{\textbf{perplexing samples}}-samples with $\boldsymbol{\mu}_2$ features.

\section{Details of Querying Algorithms}\label{app:details of querying algorithms}
\subsection{2-layer ReLU CNN}\label{App:RELUCNN}
We adopted the 2-layer ReLU CNN, which is representative for non-linear neural models. Also, this neural setting makes both the model's uncertainty towards samples and the latent feature representation available, paving the way to design NAL algorithms based on this neural settings. The first layer of the model is composed of $2m$ neurons/filters, with $m$ positive and $m$ negative, each of which is applied separately to the two patches $\mathbf{x}_{1}$ and $\mathbf{x}_{2}$, with a ReLU function $\sigma(z)=\max \{0, z\}$. Specifically, the parameters of the second pooling layer are set to $+\frac{1}{m}$ and $-\frac{1}{m}$ respectively. The network can thus be expressed as $f(\mathbf{W}, \mathbf{x})=F_{+1}\left(\mathbf{W}_{+1}, \mathbf{x}\right)-F_{-1}\left(\mathbf{W}_{-1}, \mathbf{x}\right)$, where the partial network functions for positive and negative neurons/filters. For $j \in\{ +1, -1\}$, $F_{j}\left(\mathbf{W}_{j}, \mathbf{x}\right)$ is defined as follows:
\begin{equation}
\begin{aligned}
F_j\left(\mathbf{W}_j, \mathbf{x}\right)&=\frac{1}{m} \sum_{r=1}^m[\sigma\left(\left\langle\mathbf{w}_{j, r}, \mathbf{x}_1\right\rangle\right) 
+\sigma\left(\left\langle\mathbf{w}_{j, r}, \mathbf{x}_2\right\rangle\right)]\\
&=\frac{1}{m} \sum_{r=1}^m\left[\sigma\left(\left\langle\mathbf{w}_{j, r}, y \cdot \boldsymbol{\mu}\right\rangle\right) 
+\sigma\left(\left\langle\mathbf{w}_{j, r}, \boldsymbol{\xi}\right\rangle\right)\right].
\end{aligned}
\label{equ:cnndefinition}
\end{equation}
We denotes $\mathbf{w}_{j, r} \in \mathbb{R}^d$ as the weight vector for the $r$-th neuron/filter in $\mathbf{W}_j$, where $\mathbf{W}_j$ is the aggregate of model weights associated with $F_j$ filters. We use $\mathbf{W}$ to denote the aggregate of all model weights. Without loss of generality, we let the derivative of the ReLU function at 0 is equal to 1, denoted as $\sigma^\prime(0)=1$.
\subsection{Score and Order of Samples}\label{App:Order def}
We claim that the following definitions and lemmas hold for both linearly s
\begin{definition}\textbf{(Confidence Score)}
    The Confidence Score $ C \left(\mathbf{W}^{(t)}, \mathbf{x}\right)$ is defined as below:
\begin{equation}
\begin{split}
    C\left(\mathbf{W}^{(t)}, \mathbf{x}\right)
    &=\max \Big\{\frac{1}{1+\exp \big\{-y \cdot f\left(\mathbf{W}^{(t)}, \mathbf{x}\right)\big\}},\\ 
    &\phantom{=\max\Big\{}1-\frac{1}{1+\exp \big\{-y \cdot f\left(\mathbf{W}^{(t)}, \mathbf{x}\right)\big\}}\Big\}
\end{split}
\label{Eq:confidence_score}\end{equation}
The Confidence Score $ C \left(\mathbf{W}^{(t)}, \mathbf{x}\right)$ represents the probability of the predicted label $y$ of logistic loss. 
\label{Def:confidence_score}\end{definition}
\begin{definition}\textbf{(Uncertainty Order)}
    We denote the sampling pool as $\mathcal{P}$ that $\mathcal{P} \subsetneq \mathcal{D}$. For $t>0$, $\forall \mathbf{x}$ and $\mathbf{x}^{\prime} \in \mathcal{P}$, we define the Uncertainty Order $\prec_C^{(t)}$ and $\preceq_C^{(t)}$, which denote the order of the model's uncertainty upon its prediction upon $ \mathbf{x}$ and $ \mathbf{x}^{\prime}$ at the time step $t$:
    \begin{equation}
    \begin{aligned}
    & \mathbf{x} \prec_C^{(t)} \mathbf{x}^{\prime} \text { if \  } C\left(\mathbf{W}^{(t)}, \mathbf{x}\right) > C\left(\mathbf{W}^{(t)}, \mathbf{x}^{\prime}\right), \\
    & \mathbf{x} \preceq_C^{(t)} \mathbf{x}^{\prime} \text { if \  } C\left(\mathbf{W}^{(t)}, \mathbf{x}\right) \geq C\left(\mathbf{W}^{(t)}, \mathbf{x}^{\prime}\right). \\
    &
    \end{aligned}
    \end{equation}
    We say the model uncertainty at time step $t$ upon $\mathbf{x}$ is less than $\mathbf{x}^{\prime}$ if $\mathbf{x} \prec_C^{(t)} \mathbf{x}^{\prime}$. Specifically, if the model's uncertainty towards its predictions upon all elements in a set $\mathbf{X}$ at time step $t$ are all less than those in the set $\mathbf{X}^\prime$, we utilize the same notation to describe the Uncertainty Order at time step $t$ between sets: $ \mathbf{X} \prec_C^{(t)} \mathbf{X}^{\prime}$.
\label{Def:confidence_order}\end{definition}

\begin{lemma}
The Uncertainty Order is a full order. In addition, for $\forall \mathbf{x}$ and $\mathbf{x}^{\prime} \in \mathcal{P}$, at $t>0$ we have:
\begin{equation}
\mathbf{x} \preceq_C^{(t)} \mathbf{x}^{\prime} \Leftrightarrow\left|f\left(\mathbf{W}^{(t)}, \mathbf{x}\right)\right| \geq \left|f\left(\mathbf{W}^{(t)}, \mathbf{x}^{\prime}\right)\right|
\label{Eq:order_absolute}\end{equation}
\label{Lemma:order_absolute}\end{lemma}
\begin{proof}
\begin{align*}
& \mathbf{x} \preceq_C^{(t)} \mathbf{x}^{\prime} \Leftrightarrow  C\left(\mathbf{W}^{(t)}, \mathbf{x}\right) \geq C\left(\mathbf{W}^{(t)}, \mathbf{x}^{\prime}\right)  \\
& \Leftrightarrow  \dfrac{1}{1+\exp \{|f(\mathbf{W}, \mathbf{x})|\}} \geq \dfrac{1}{1+\exp \left\{-\left|f\left(\mathbf{W}^{(t)}, \mathbf{x}^{\prime}\right)\right|\right\}} \\
& \Leftrightarrow  \left|f\left(\mathbf{W}^{(t)}, \mathbf{x}\right)\right| \geq \left|f\left(\mathbf{W}^{(t)}, \mathbf{x}^{\prime}\right)\right| \qedhere
\end{align*}
As one can always get $f\left(\mathbf{W}^{(t)}, \mathbf{x}\right) \in \mathbb{R}$ by a given $\mathbf{x}$ at time step $t$, the Uncertainty Order is a full order. 
\end{proof}
In Lemma \ref{lemma:Equivalance of score function}, we will show that sampling based on the Uncertainty Order is equivalent to various typical sampling methods based on the score functions defined in many typical Model Uncertainty-based Approaches, such as Least Confidence \cite{lewis1994heterogeneous}, Margin \citet{roth2006margin} and Entropy \cite{joshi2009multi} methods under our data model scenario, thus it's representative to the main idea of the approaches family while elegant.
\begin{definition}
    The following are the definitions of the score functions of LeastConf \cite{lewis1994heterogeneous}, Margin \citet{roth2006margin} and Entropy \cite{joshi2009multi}.

\begin{itemize}
    \item Least Confidence selects data points whose predicted label $y$ have the lowest posterior probability, so the score function of LeastConf is:
    \begin{equation}
        Score(\mathbf{W}^{(t)}, \mathbf{x}) = - P(y|\mathbf{x}, \mathbf{W}^{(t)}), 
    \label{Eq:leastConf}\end{equation}
    \item The score function of Margin is:
    \begin{equation}
        Score(\mathbf{W}^{(t)}, \mathbf{x}) = - [p(y|\mathbf{x}, \mathbf{W}^{(t)}) - P(-y|\mathbf{x}, \mathbf{W}^{(t)})],
    \label{Eq:Margin}\end{equation}
    \item The score function of Entropy is:
    \begin{equation}
        Score(\mathbf{W}^{(t)}, \mathbf{x}) = -  [P(y|\mathbf{x}, \mathbf{W}^{(t)}) \log P(y|\mathbf{x}, \mathbf{W}^{(t)}) + P(-y|\mathbf{x}, \mathbf{W}^{(t)}) \log P(-y|\mathbf{x}, \mathbf{W}^{(t)})],
    \label{Eq:Entropy}\end{equation}
\end{itemize}
\end{definition}
\begin{lemma}
    Sampling based on the score functions defined in  (\ref{Eq:leastConf}), (\ref{Eq:Margin}) and (\ref{Eq:Entropy}) are equivalent to sampling based on the Confidence Order in Definition \ref{Def:confidence_order}.
\label{lemma:Equivalance of score function}
\end{lemma}
\begin{proof}
    By definitions, $C \left(\mathbf{W}^{(t)}, \mathbf{x}\right)= P(y|\mathbf{x}, \mathbf{W}^{(t)})=-Score(\mathbf{W}^{(t)}, \mathbf{x})$, showing the equivalence of LeastConf methods and ours. Then by Lemma \ref{Lemma:order_absolute} and the property: $P(-y|\mathbf{x}, \mathbf{W}^{(t)}) = 1 -C \left(\mathbf{W}^{(t)}, \mathbf{x}\right)$, it's easy to verify that $\left|f\left(\mathbf{W}^{(t)}, \mathbf{x}\right)\right| \propto C \left(\mathbf{W}^{(t)}, \mathbf{x}\right) \propto [C \left(\mathbf{W}^{(t)}, \mathbf{x}\right)-(1-C \left(\mathbf{W}^{(t)}, \mathbf{x}\right))]$, and $ \left|f\left(\mathbf{W}^{(t)}, \mathbf{x}\right)\right| \propto  C \left(\mathbf{W}^{(t)}, \mathbf{x}\right) \propto [C \left(\mathbf{W}^{(t)}, \mathbf{x}\right) \log C \left(\mathbf{W}^{(t)}, \mathbf{x}\right) + (1-C \left(\mathbf{W}^{(t)}, \mathbf{x}\right))\log(1-C \left(\mathbf{W}^{(t)}, \mathbf{x}\right))]$. Therefore, the priority order of the samples based on those score functions are the same as the Uncertainty Order, thus the proof is completed.
\end{proof}
\begin{definition}
    \textbf{(Feature Distance)} The latent feature representation of a sample $\mathbf{x}$=$[\mathbf{x}_1^{T}, \mathbf{x}_2^T]^T$ in the latent feature space $\mathcal{Z} \subseteq \mathbb{R}^{m}$ of our ReLU CNN at timestep $t$ is:
    \[
     \mathbf{Z}(\mathbf{x}, t) = \sum_{j} ( \sigma(\langle \mathbf{W}_j^{(t)}, \mathbf{x}_1 \rangle)) + \sigma(\langle \mathbf{W}_j^{(t)}, \mathbf{x}_2 \rangle))
    \]
    Apparently $\mathbf{Z}(\mathbf{x}, t) \in \mathbb{R}^{m}$. The Feature Distance is measured by the $l_p$ ($p \in \left[1, \infty\right)$) distance between sample's feature representation and the average feature representation of the current labeled set $ \mathcal{D}_{n} \mathrel{\mathop:}= \{ \mathbf{x}^{(i)} \}_{i=1}^{n}$:
    \begin{equation}
    D\left(\mathbf{W}^{(t)}, \mathbf{x} \ \mid \mathcal{D}_{n}\right) =\| \mathbf{Z}(\mathbf{x}, t) - \displaystyle \underset{\mathbf{x}^{(i)} \in \mathcal{D}_{n}}{\mathbb{E}} \mathbf{Z}(  \mathbf{x}^{(i)}, t) \|_p
    \label{Eq:Feature_Distance}\end{equation}
\label{Def:Feature_Distance}\end{definition}

\begin{definition}
    \textbf{(Diversity Order)}
    Similar to Definition \ref{Def:confidence_order}, we defined Diversity Order $\prec_D^{(t)}$, $\preceq_D^{(t)}$ based on Feature Distance $D\left(\mathbf{W}^{(t)}, \mathbf{x} \ \mid \mathcal{D}_{n}\right)$. Borrowing the same notations in Definition \ref{Def:confidence_order}, we have:
    \begin{equation}
    \begin{aligned}
    & \mathbf{x} \prec_{D}^{(t)} \mathbf{x}^{\prime} \text { if \  } D\left(\mathbf{W}^{(t)}, \mathbf{x} \ \mid \mathcal{D}_{n}\right) < D\left(\mathbf{W}^{(t)}, \mathbf{x}^{\prime} \ \mid \mathcal{D}_{n}\right), \\
    & \mathbf{x} \preceq_{D}^{(t)} \mathbf{x}^{\prime} \text { if \  } D\left(\mathbf{W}^{(t)}, \mathbf{x} \ \mid \mathcal{D}_{n}\right) \leq D\left(\mathbf{W}^{(t)}, \mathbf{x}^{\prime} \ \mid \mathcal{D}_{n}\right). \\
    &
    \end{aligned}
    \label{Eq:Diversity_Order}\end{equation}
    Along with Definition \ref{Def:confidence_order}, we also have set-level notations such that $ \mathbf{X} \prec_{D}^{(t)} (\preceq_{D}^{(t)}) \mathbf{X}^{\prime}$. Based on the triangle inequality for the $l_p$ norm and  (\ref{Eq:Feature_Distance}), we can easily draw the conclusion that the Diversity Order is also a full order. Furthermore, in the case that both $\mathbf{x} \prec_C^{(t)} (\preceq_C^{(t)}) \mathbf{x}^{\prime}$ and $\mathbf{x} \prec_{D}^{(t)} (\preceq_{D}^{(t)}) \mathbf{x}^{\prime}, \forall p \in \left[1, \infty\right)$ hold, we denote the order relationship using $\prec^{(t)} (\preceq^{(t)})$, such that $\mathbf{x} \prec^{(t)} (\preceq^{(t)}) \ \mathbf{x}^{\prime}$.
\label{Diversity_Order}\end{definition}

\section{Proofs of Main Results}\label{app:proofs of main results}
In this section, we denote $n$ as the number of training data in current labeled training set, which is $n_0$ at initial stage and $n_1$ after sampling (querying). Besides, we denote the proportion of easy-to-learn data in current labeled set as $\tau_1$, and utilize $\tau_2$ to represent the proportion of hard-to-learn data in current labeled set for notation simplicity. Notably, we can use the same techniques in \citet{cao2022benign, kou2023benign, meng2023benign, lu2023benign} to achieve some statistical outcomes that are not directly related to our main contribution, we exclude the proof details for those outcomes. Instead, our focus is on providing comprehensive proofs of our primary contribution.

\subsection{Preliminary Lemmas}\label{app:preliminary lemmas}
The following lemmas give finite-sample concentration results to characterize the statistical properties of the random elements involved in our problem, and hold both under the linearly separable data and XOR data (i.e., $\boldsymbol{\mu}_l \in \{ \boldsymbol{\mu}_l,  \mathbf{u}_l,  \mathbf{v}_l \}, \forall l \{1, 2\}$).
\begin{lemma}
Suppose that $\delta>0$ and $d=\Omega(\log (\dfrac{6n}{\delta}))$. Then with probability at least $1-\delta$,
\[
\begin{aligned}
& \dfrac{\sigma_p^2 d}{2} \leq\left\|\boldsymbol{\xi}_i\right\|_2^2 \leq 3 \dfrac{\sigma_p^2 d}{2}, \\
& \left|\left\langle\boldsymbol{\xi}_i, \boldsymbol{\xi}_{i^{\prime}}\right\rangle\right| \leq 2 \sigma_p^2 \cdot \sqrt{d \log \left(\dfrac{6n^2}{\delta}\right)}, \\
&
\left|\left\langle\boldsymbol{\xi}_i, \boldsymbol{\mu}_l\right\rangle\right| \leq\|\boldsymbol{\mu}_l\|_2 \sigma_p \cdot \sqrt{2 \log (\dfrac{12n}{\delta})}
\end{aligned}
\]
for all $i, i^{\prime} \in[n], l \in \{1, 2\}$.\label{Lem:E.1}
\end{lemma}
\textit{Proof of Lemma \ref{Lem:E.1}.} The proof can be found in Lemma B.2 in \citet{cao2022benign}, Lemma B.4 in \citet{kou2023benign}, Lemma B.3 in \citet{meng2023benign} or Lemma A.3 in \citet{lu2023benign}.\par

\begin{lemma}
Suppose that $\delta>0, d=\Omega(\log (\dfrac{mn}{\delta})),$ and $m=\Omega(\log (\dfrac{1}{\delta}))$. Then with probability at least $1-\delta$,
$$
\begin{aligned}
& \dfrac{\sigma_0^2 d}{2} \leq \|\mathbf{w}_{j, r}^{(0)} \|_2^2 \leq 3 \dfrac{\sigma_0^2 d}{2}, \\
& \left|\left\langle\mathbf{w}_{j, r}^{(0)}, \boldsymbol{\mu}_l \right\rangle\right| \leq \sqrt{2 \log (\dfrac{16m}{\delta}}) \cdot \sigma_0\|\boldsymbol{\mu}_l\|_2, \\
& \left|\left\langle\mathbf{w}_{j, r}^{(0)}, \boldsymbol{\xi}_i\right\rangle\right| \leq 2 \sqrt{\log \dfrac{16mn}{\delta}} \cdot \sigma_0 \sigma_p \sqrt{d}
\end{aligned}
$$
for all $r \in[m], j \in\{ \pm 1\}, l \in \{1, 2\}$ and $i \in[n]$. Moreover,
$$
\begin{aligned}
& \dfrac{\sigma_0\|\boldsymbol{\mu}_l\|_2}{2} \leq \max _{r \in[m]} j \cdot\left\langle\mathbf{w}_{j, r}^{(0)}, \boldsymbol{\mu}_l\right\rangle \leq \sqrt{2 \log (\dfrac{16m}{\delta}}) \cdot \sigma_0\|\boldsymbol{\mu}_l\|_2, \\
& \dfrac{\sigma_0 \sigma_p \sqrt{d}}{4} \leq \max _{r \in[m]} j \cdot\left\langle\mathbf{w}_{j, r}^{(0)}, \boldsymbol{\xi}_i\right\rangle \leq 2 \sqrt{\log \dfrac{16mn}{\delta}} \cdot \sigma_0 \sigma_p \sqrt{d}
\end{aligned}
$$
for all $j \in\{ \pm 1\}, l \in \{1, 2\}$ and $i \in[n]$.
\label{Lem:E.2}
\end{lemma}
\textit{Proof of Lemma \ref{Lem:E.2}.} The proof can be found in Lemma B.3 in \citet{cao2022benign}, Lemma B.5 in \citet{kou2023benign}, Lemma B.4 in \citet{meng2023benign} or Lemma A.4 and Lemma C.1 in \citet{lu2023benign}. \par

Next, we utilize the property of binomial tails to examine the proportion of hard-to-learn data within the subsets generated from the data distribution $\mathcal{D}$ (i.e., the initial labeled set $\mathcal{D}_{n_0} \mathrel{\mathop:}= \{ \mathbf{x}^{(i)} \}_{i=1}^{n_0} \subseteq \mathcal{D}$, the sampling pool $\mathcal{P} \subseteq \mathcal{D}$, and the final labeled set $\mathcal{D}_{n_1}^{(random)} \mathrel{\mathop:}= \{ {\mathbf{x}^{(random)}}^{(i)} \}_{i=1}^{n_1}\subseteq \mathcal{D}$ obtained through Random Sampling).
\begin{lemma}
Suppose that $\delta>0$, $n_0, \tilde{n},|P|=\Omega\left(\dfrac{1-p}{p} \log \left(\dfrac{1}{\delta}\right)\right)$, then for $n \in\left\{n_0,|P|, n_1\right\}$. Denote $n_{p} \leq n$ as the number of hard-to-learn data among $n$, then with probability at least $1 - \delta$. We have
\begin{equation}
\frac{1}{2} p \cdot n \leqslant n_p \leqslant \frac{3}{2} p \cdot n
\end{equation}
\label{lem:rho n}\end{lemma}
\textit{proof of Lemma \ref{lem:rho n}.} We can see $n_{p}$ as a binomial random variable with probability $p$ and number of experiments $n$. By Exercise 2.9.(a) in \citet{wainwright2019high}, we have
$$
P\left(\dfrac{p n}{2}  \leq n_{p} \leq \dfrac{3p n}{2} \right) \geqslant 1-2 e^{-n D\left(\frac{p}{2} \| p\right)}
$$
where the quantity $D(\delta \| \alpha)$ for $\forall \delta, \alpha \in\left(0, \frac{1}{2}\right]$ is defined as
$$
D(\delta \| \alpha):=\delta \log \left(\frac{\delta}{\alpha}\right)+(1-\delta) \log \left(\frac{1-\delta}{1-\alpha}\right) .
$$

Since $\dfrac{p}{2}<p$. By Exercise $2.9.(b)$ in \citet{wainwright2019high}, we can obtain $P\left(\dfrac{p n}{2}  \leq n_{p} \leq \dfrac{3p n}{2}\right) \geq 1-\delta$ directly by Hoeffding Inequality. 
\begin{remark}
It is important to note that the generation of $\mathcal{D}_{n_0}$ and $\mathcal{P}$ through sampling from $\mathcal{D}$ is independent. However, the generation of $\mathcal{D}_{n_1}^{(random)}$ is based on $\mathcal{D}_{n_0}$ and $\mathcal{P}$. In our analysis, instead of considering martingale with the perspective of conditional probability, we consider the overall process of the labeled set obtained by Random Sampling, where $\mathcal{D}_{n_1}^{(random)}$ is directly sampled from $\mathcal{D}$.
\end{remark}

\subsection{Coefficient Ratio and Scale Analysis}\label{app:Coefficient Ratio and Scale analysis}
In this section, we provide lemmas that characterize the behavior of coefficients under gradient descent. Subsequently, we establish the scale of the coefficients in the training dynamics. It's worth noting that in this section we assume the results in Appendix \ref{app:preliminary lemmas} all hold with high probability. \par
\begin{definition} 
(Equivalent techniques to Definition 4.1 in \citet{cao2022benign}, Definition 5.1 in \citet{kou2023benign}) Denote $\mathbf{w}_{j, r}^{(t)}$ for $j \in\{ \pm 1\}, r \in[m]$ as the convolution neurons/filters at the $t^{th}$ timestep of gradient descent, then there exist unique coefficients $\gamma_{j, r, l}^{(t)}$ and $\rho_{j, r, i}^{(t)}$ such that
\[
\mathbf{w}_{j, r}^{(t)}=\mathbf{w}_{j, r}^{(0)}+j \cdot  \sum_{l=1}^2 \gamma_{j, r, l}^{(t)} \cdot \dfrac{\boldsymbol{\mu}_l}{\|\boldsymbol{\mu}_l\|_2^{2}} +\sum_{i=1}^{n} \rho_{j, r, i}^{(t)} \cdot \dfrac{\boldsymbol{\xi}_i}{\|\boldsymbol{\xi}_i\|_2^{2}}
\]
Further denote $\bar{\rho}_{j, r, i}^{(t)}$ as $\rho_{j, r, i}^{(t)} \mathbb{1}\left(\rho_{j, r, i}^{(t)} \geq 0\right)$, $ \underline{\rho}_{j, r, i}^{(t)}$ as $\rho_{j, r, i}^{(t)} \mathbb{1}\left(\rho_{j, r, i}^{(t)} \leq 0\right)$. Then:
\begin{equation}
\mathbf{w}_{j, r}^{(t)}=\mathbf{w}_{j, r}^{(0)}+j \sum_{l=1}^2 \cdot \gamma_{j, r, l}^{(t)} \cdot \dfrac{\boldsymbol{\mu}_l}{\|\boldsymbol{\mu}_l\|_2^{2}} 
+\sum_{i=1}^{n} \bar{\rho}_{j, r, i}^{(t)} \cdot \dfrac{\boldsymbol{\xi}_i}{\|\boldsymbol{\xi}_i\|_2^{2}} +\sum_{i=1}^{n} \underline{\rho}_{j, r, i}^{(t)} \cdot \dfrac{\boldsymbol{\xi}_i}{\|\boldsymbol{\xi}_i\|_2^{2}}.
\label{eq:s-n-d}\end{equation}
\label{app:Def:signol-noise-decomposition}\end{definition}
We denote $U_l=\left\{i \in[n] : \mathbf{x}^{(i)} = [y_i \cdot \boldsymbol{\mu}_l, \boldsymbol{\xi}_i] \right\}$, for $l \in \{1, 2\}$. The following lemma presents the update rule of coefficients.
\begin{lemma}
The coefficients $\gamma_{j, r, l}^{(t)}, \bar{\rho}_{j, r, i}^{(t)}, \underline{\rho}_{j, r, i}^{(t)}$ defined in Definition \ref{app:Def:signol-noise-decomposition} satisfy the following iterative equations:
$$
\begin{aligned}
& \gamma_{j, r, l}^{(0)}, \bar{\rho}_{j, r, i}^{(0)}, \underline{\rho}_{j, r, i}^{(0)}=0, \\
& \gamma_{j, r, l}^{(t+1)}=\gamma_{j, r, l}^{(t)}-\frac{\eta}{n m} \cdot\sum_{i \in U_l} {\ell_i^{\prime}}^{(t)} \sigma^{\prime}\left(\left\langle\mathbf{w}_{j, r}^{(t)}, y_i \cdot \boldsymbol{\mu}_l\right\rangle\right) \cdot\|\boldsymbol{\mu}_l\|_2^2, \\
& \bar{\rho}_{j, r, i}^{(t+1)}=\bar{\rho}_{j, r, i}^{(t)}-\frac{\eta}{n m} \cdot {\ell_i^{\prime}}^{(t)} \cdot \sigma^{\prime}\left(\left\langle\mathbf{w}_{j, r}^{(t)}, \boldsymbol{\xi}_i\right\rangle\right) \cdot\left\|\boldsymbol{\xi}_i\right\|_2^2 \cdot \mathbb{1}\left(y_i=j\right), \\
& \underline{\rho}_{j, r, i}^{(t+1)}=\underline{\rho}_{j, r, i}^{(t)}+\frac{\eta}{n m} \cdot {\ell_i^{\prime}}^{(t)} \cdot \sigma^{\prime}\left(\left\langle\mathbf{w}_{j, r}^{(t)}, \boldsymbol{\xi}_i\right\rangle\right) \cdot\left\|\boldsymbol{\xi}_i\right\|_2^2 \cdot \mathbb{1}\left(y_i=-j\right),
\end{aligned}
$$
for all $r \in[m], j \in\{ \pm 1\}, l \in \{1, 2\}$ and $i \in[n]$.
\label{lem:update of coef}\end{lemma}

\begin{remark}
    This lemma serves as a cornerstone in our analysis of dynamics. Originally, the study of neural network dynamics under gradient descent required us to track the variations in weights. However, this Lemma enables us to view these dynamics from a new perspective, focusing on two distinct elements: feature learning (represented by $\gamma_{j, r, l}^{(t+1)}$) and noise memorization (represented by $\rho_{j, r, i}^{(t+1)}$). We can easily observe that the $\gamma_{j, r, l}^{(t)}$ is strictly increasing since ${\ell_i^{\prime}}^{(t)}$ is strictly negative. 
\end{remark}
\textit{Proof of Lemma \ref{lem:update of coef}.} Applying the gradient descent rule in (\ref{Eq:w update}), we get
$$
\begin{aligned}
& \mathbf{w}_{j, r}^{(t+1)}=\mathbf{w}_{j, r}^{(0)}-\frac{\eta}{n m} \sum_{s=0}^t \sum_{i=1}^n {\ell_i^{\prime}}^{(s)} \cdot \sigma^{\prime}\left(\left\langle\mathbf{w}_{j, r}^{(s)}, \boldsymbol{\xi}_i\right\rangle\right) \cdot j y_i \boldsymbol{\xi}_i \\
& -\frac{\eta}{n m} \sum_{s=0}^t \sum_{i=1}^n {\ell_i^{\prime}}^{(s)} \cdot \sigma^{\prime}\left(\left\langle\mathbf{w}_{j, r}^{(s)}, y_i \boldsymbol{\mu}_l\right\rangle\right) \cdot j \boldsymbol{\mu}_l . \\
&
\end{aligned}
$$

Based on the definition of $\gamma_{j, r, l}^{(t)}$ and $\rho_{j, r, i}^{(t)}$, we consider $\gamma_{j, r, l}^{(0)}, \bar{\rho}_{j, r, i}^{(0)}, \underline{\rho}_{j, r, i}^{(0)}=0$ and 
$$
\mathbf{w}_{j, r}^{(t)}=\mathbf{w}_{j, r}^{(0)}+j \cdot \sum_{l=1}^2 \gamma_{j, r, l}^{(t)} \cdot\|\boldsymbol{\mu}_l\|_2^{-2} \cdot \boldsymbol{\mu}_l+ \sum_{i=1}^n \rho_{j, r, i}^{(t)} \cdot\left\|\boldsymbol{\xi}_i\right\|_2^{-2} \cdot \boldsymbol{\xi}_i .
$$

Note that $\boldsymbol{\mu}_1$, $\boldsymbol{\mu}_2$ and $\boldsymbol{\xi}_i$ are linearly independent with probability 1, thus we have the following unique representation
$$
\begin{aligned}
& \gamma_{j, r, l}^{(t)}=-\frac{\eta}{n m} \sum_{s=0}^t \sum_{i \in U_l} {\ell_i^{\prime}}^{(s)} \cdot \sigma^{\prime}\left(\left\langle\mathbf{w}_{j, r}^{(s)}, y_i \boldsymbol{\mu}_l\right\rangle\right) \cdot\|\boldsymbol{\mu}_l\|_2^2, \\
& \rho_{j, r, i}^{(t)}=-\frac{\eta}{n m} \sum_{s=0}^t {\ell_i^{\prime}}^{(s)} \cdot \sigma^{\prime}\left(\left\langle\mathbf{w}_{j, r}^{(s)}, \boldsymbol{\xi}_i\right\rangle\right) \cdot\left\|\boldsymbol{\xi}_i\right\|_2^2 \cdot j y_i .
\end{aligned}
$$

Recall $U_l=\left\{i \in[n] : \mathbf{x}^{(i)} = [y_i \cdot \boldsymbol{\mu}_l, \boldsymbol{\xi}_i] \right\}$, we have
\begin{equation}
\gamma_{j, r, l}^{(t)}=-\frac{\eta}{n m} \sum_{s=0}^t \sum_{i \in U_l} {\ell_i^{\prime}}^{(s)} \cdot \sigma^{\prime}\left(\left\langle\mathbf{w}_{j, r}^{(s)}, y_i \boldsymbol{\mu}_l\right\rangle\right) \cdot\|\boldsymbol{\mu}_l\|_2^2.
\label{eq:update of gamma}\end{equation}

Now with the notation $\bar{\rho}_{j, r, i}^{(t)}\mathrel{\mathop:}=\rho_{j, r, i}^{(t)} \mathbb{1}\left(\rho_{j, r, i}^{(t)} \geq 0\right), \underline{\rho}_{j, r, i}^{(t)}\mathrel{\mathop:}=\rho_{j, r, i}^{(t)} \mathbb{1}\left(\rho_{j, r, i}^{(t)} \leq 0\right)$ and the fact ${\ell_i^{\prime}}^{(s)}<0$, we get
\begin{equation}
\bar{\rho}_{j, r, i}^{(t)}=-\frac{\eta}{n m} \sum_{s=0}^t {\ell_i^{\prime}}^{(s)} \cdot \sigma^{\prime}\left(\left\langle\mathbf{w}_{j, r}^{(s)}, \boldsymbol{\xi}_i\right\rangle\right) \cdot\left\|\boldsymbol{\xi}_i\right\|_2^2 \cdot \mathbb{1}\left(y_i=j\right),
\label{eq:update of rho+}\end{equation}
\begin{equation}
\underline{\rho}_{j, r, i}^{(t)}=\frac{\eta}{n m} \sum_{s=0}^t {\ell_i^{\prime}}^{(s)} \cdot \sigma^{\prime}\left(\left\langle\mathbf{w}_{j, r}^{(s)}, \boldsymbol{\xi}_i\right\rangle\right) \cdot\left\|\boldsymbol{\xi}_i\right\|_2^2 \cdot \mathbb{1}\left(y_i=-j\right).
\label{eq:update of rho-}\end{equation}
The proof is completed.
\begin{remark}
    The proof strategy employed in this study follows the study of feature learning analysis techniques in \citet{cao2022benign, kou2023benign, meng2023benign}. However, our decomposition considers two task-specific features with different proportion. This disparity would finally lead to distinct learning efficiency among samples, as well as different generalization ability.
\end{remark}
Next, we're dedicated to explore range scale evolution of the coefficients in the signal-noise decomposition. Let $T^*=$ $\eta^{-1}$ poly $\left(\varepsilon^{-1}, d, n, m\right)$ be the maximum admissible iteration. Denote
\begin{equation}
\begin{aligned}
& \alpha\mathrel{\mathop:}=4 \log \left(T^*\right), \\
& \beta\mathrel{\mathop:}=2 \max _{l, i, j, r}\left\{\left|\left\langle\mathbf{w}_{j, r}^{(0)}, \boldsymbol{\mu}_l\right\rangle\right|,\left|\left\langle\mathbf{w}_{j, r}^{(0)}, \boldsymbol{\xi}_i\right\rangle\right|\right\}, \\
& \operatorname{SNR}_l\mathrel{\mathop:}=\dfrac{\|\boldsymbol{\mu}_l\|_2}{ \sigma_p \sqrt{d}} .
\end{aligned}
\label{eq:alpha,belta,snr}
\end{equation}

By Lemma \ref{Lem:E.2}, $\beta$ can be bounded by $4 \sigma_0 \cdot \max \left\{\sqrt{\log \dfrac{16mn}{\delta}} \cdot \sigma_p \sqrt{d}, \sqrt{\log (\dfrac{16m}{\delta}}) \cdot\|\boldsymbol{\mu}_l\|_2\right\}$. Under Condition \ref{Con4.1}, it is straightforward to verify the following inequality with a large constant $C$:
\begin{equation}
\max_l
\left\{\beta, \operatorname{SNR}_l \sqrt{\frac{32 \log (\frac{12n}{\delta})}{d}} n \alpha, 5 \sqrt{\frac{\log \left(\frac{6n^2}{\delta}\right)}{d}} n \alpha\right\} \leq \frac{1}{12} .
\label{eq:E.5}
\end{equation}
We then assert the following proposition hold for the entire training period. This proposition serves to show the evolution scale of the coefficients.

\begin{proposition}
Under Condition \ref{Con4.1}, for $0 \leq t \leq T^*$,  there exists a positive constant $C^{\prime}$ such that
\begin{equation}
\begin{aligned}
& 0 \leq \gamma_{j, r, l}^{(t)} \leq C^{\prime} \cdot \tau_l n \cdot \operatorname{SNR}_l^2 \cdot \alpha \\
& 0 \leq \bar{\rho}_{j, r, i}^{(t)} \leq \alpha, \\
& 0 \geq \underline{\rho}_{j, r, i}^{(t)} \geq-\beta-10 \sqrt{\frac{\log \left(\frac{6n^2}{\delta}\right)}{d}} n \alpha \geq-\alpha,
\end{aligned}
\label{eqapp:scale in prop3.8}\end{equation}
for all $ j \in\{ \pm 1\}, r \in[m], l \in \{1, 2\}$ and $i \in[n]$.

\label{prop:E.8}\end{proposition}

\begin{remark}
    Our results resemble those in the study of feature learning of CNN \cite{cao2022benign, kou2023benign, meng2023benign, lu2023benign}. However, the scale of our learning progress coefficient $\gamma_{j, r, l}^{(t)}$ depends on its corresponding feature proportion and strength in the labeled data distribution, which will significantly impact the learning process of specific type of data.
\end{remark}
\textit{Proof of Proposition \ref{prop:E.8}.} See Proposition C.2. and Proposition C.8. in \citet{kou2023benign} or Proposition C.2 and Proposition C.8 in \citet{meng2023benign} for a proof. Regardless of the variations in data settings, obtaining the result through inductive techniques is readily feasible. \par
Based on Proposition \ref{prop:E.8}, we can analyze the convergence of the training dynamics via identifying the degree of feature learning and noise memorization in the following section.
\subsection{Feature Learning and Noise Memorization Analysis}\label{app:feature learning and Noise Memorization Analysis}
In this section, we adopt a two-stage analysis to evaluate the evolution of the coefficients. In the first stage, the loss function's derivative remains nearly constant due to the small weight initialization. However, in the subsequent stage, the derivative of the loss function becomes non-constant, requiring a careful analysis to address this change. We will see that the scale differences in the first stage remain the same. Worth noting that the results in this section are based on the previous results in Appendix \ref{app:Coefficient Ratio and Scale analysis} holding with high probability.
\subsubsection{First Stage: Feature Learning versus Noise Memorization}
\begin{lemma}
There exist
$$
T_1=C_3 \eta^{-1} n m \sigma_p^{-2} d^{-1}, T_2=C_4 \eta^{-1} n m \sigma_p^{-2} d^{-1}
$$
where $C_3=\Theta(1)$ is a large constant and $C_4=\Theta(1)$ is a small constant, such that
\begin{itemize}
    \item $\max _{j, r} \gamma_{j, r, l}^{(t)}=O( \tau_l n \cdot \operatorname{SNR}_l^2)$ , for all $0 \leq t \leq T_1, l\in \{1, 2\}$.
    \item $\min _{j, r} \gamma_{j, r, l}^{(t)}=\Omega( \tau_l n \cdot \operatorname{SNR}_l^2)$, for all $t \geq T_2, l\in \{1, 2\}$.
    \item $\bar{\rho}_{j, r^*, i}^{\left(T_1\right)} \geq 2$, for any $r^* \in S_i^{(0)}=\left\{r \in[m]:\left\langle\mathbf{w}_{y_i, r}^{(0)}, \boldsymbol{\xi}_i\right\rangle>0\right\}, j \in\{ \pm 1\}$ and $i \in[n]$ with $y_i=j$.
    \item $\max _{j, r, i}\left|\underline{\rho}_{j, r, i}^{(t)}\right|=\max \left\{O\left(\sqrt{\log (\dfrac{mn}{\delta})} \cdot \sigma_0 \sigma_p \sqrt{d}\right), O\left(n \sqrt{\log (\dfrac{n}{\delta})} \log \left(T^*\right) / \sqrt{d}\right)\right\}$, for all $0 \leq t \leq T_1$.
    \item $\max _{j, r} \bar{\rho}_{j, r, i}^{\left(T_1\right)}=O(1)$, for all $i \in[n]$.
\end{itemize}
\label{lem:E.10}\end{lemma}
\textit{Proof of Lemma \ref{lem:E.10}.} See Lemma D.1. in \citet{kou2023benign} or Lemma D.1, Proposition D.2-D.4 in \citet{meng2023benign} for a proof.
\subsubsection{Second Stage: Convergence of Training Error}
At the end of the first stage, we have the following feature-to-noise decomposition:
$$
\mathbf{w}_{j, r}^{\left(T_1\right)}=\mathbf{w}_{j, r}^{(0)}+j \cdot \sum_{l=1}^{2} \gamma_{j, r, l}^{\left(T_1\right)} \cdot \frac{\boldsymbol{\mu}_l}{\|\boldsymbol{\mu}_l\|_2^2}+\sum_{i=1}^n \bar{\rho}_{j, r, i}^{\left(T_1\right)} \cdot \frac{\boldsymbol{\xi}_i}{\left\|\boldsymbol{\xi}_i\right\|_2^2}+\sum_{i=1}^n \underline{\rho}_{j, r, i}^{\left(T_1\right)} \cdot \frac{\boldsymbol{\xi}_i}{\left\|\boldsymbol{\xi}_i\right\|_2^2}
$$
for $j \in[ \pm 1]$ and $r \in[m]$. Applying the results we obtain in the first stage, we have the following property holds at the beginning of this stage:
\begin{itemize}
    \item $\gamma_{j, r, l}^{\left(T_1\right)}=\tau_l n \cdot \operatorname{SNR}_l^2)$ for any $j \in\{ \pm 1\}, r \in[m]$.
    \item $\bar{\rho}_{j, r^*, i}^{\left(T_1\right)} \geq 2$ for any $r^* \in S_i^{(0)}=\left\{r \in[m]:\left\langle\mathbf{w}_{y_i, r}^{(0)}, \boldsymbol{\xi}_i\right\rangle>0\right\}, j \in\{ \pm 1\}$ and $i \in[n]$ with $y_i=j$.
    \item $\max _{j, r, i}\left|\underline{\rho}_{j, r, i}^{\left(T_1\right)}\right|=\max \left\{O\left(\sqrt{\log (\dfrac{mn}{\delta})} \cdot \sigma_0 \sigma_p \sqrt{d}\right), O\left(n \sqrt{\log (\dfrac{n}{\delta})} \log \left(T^*\right) / \sqrt{d}\right)\right\}$.
\end{itemize}
Following the technique in \citet{cao2022benign}, now we choose $\mathbf{W}^*$ as follows
$$
\mathbf{w}_{j, r}^*=\mathbf{w}_{j, r}^{(0)}+5 \log (\dfrac{2}{\varepsilon}) \left[\sum_{i=1}^n \mathbb{1}\left(j=y_i\right) \cdot \frac{\boldsymbol{\xi}_i}{\left\|\boldsymbol{\xi}_i\right\|_2^2}\right] .
$$

\begin{lemma} 
Under Condition \ref{Con4.1}, we have 
$$\max _{j, r, i}\left|\underline{\rho}_{j, r, i}^{(t)}\right|=\max \left\{O\left(\sqrt{\log (\dfrac{mn}{\delta}}) \cdot \sigma_0 \sigma_p \sqrt{d}\right), O\left(n \sqrt{\log (\dfrac{n}{\delta})} \log \left(T^*\right) / \sqrt{d}\right)\right\},$$ 
for all $T_1 \leq t \leq T^*$. Besides,
$$
\frac{1}{t-T_1+1} \sum_{s=T_1}^t L_S\left(\mathbf{W}^{(s)}\right) \leq \frac{\left\|\mathbf{W}^{\left(T_1\right)}-\mathbf{W}^*\right\|_F^2}{\eta\left(t-T_1+1\right)}+\varepsilon
$$
for all $T_1 \leq t \leq T^*$. Therefore, we can find an iterate with training loss smaller than $2 \varepsilon$ within $T=T_1+\left|\left\|\mathbf{W}^{\left(T_1\right)}-\mathbf{W}^*\right\|_F^2 /(\eta \varepsilon)\right|=T_1+\widetilde{O}\left(\eta^{-1} \varepsilon^{-1} m n d^{-1} \sigma_p^{-2}\right)$ iterations.
\label{lem:E.15}\end{lemma}
\textit{Proof of Lemma \ref{lem:E.15}.} See Lemma D.5 in \citet{cao2022benign} or Lemma D.6. in \citet{kou2023benign} for a proof.\par
Worth noting that since the $n$ could be $n_0$ or $n_1$ and the $\tau_l$ could be any real number denoting the proportion of specific types of data in the labeled set, we have successfully concluded the proof of training loss convergence for all three querying algorithms. The following lemma characterized the feature-to-noise ratio during the whole duration.
\begin{lemma} Under Condition \ref{Con4.1}, we have
$$
 \sum_{i=1}^n \bar{\rho}_{j, r, i}^{(t)}/ \gamma_{j^{\prime}, r^{\prime}, l}^{(t)} =\Theta\left(\tau_l^{-1}\cdot \operatorname{SNR}_l^{-2} \right)
$$
for all $j, j^{\prime} \in\{ \pm 1\}, r, r^{\prime} \in[m], l \in \{ 1, 2\}$ and $0 \leq t \leq T^*$.
\label{lem:E.16}\end{lemma}
\textit{Proof of Lemma \ref{lem:E.16}.} See Lemma D.7. in \citet{kou2023benign} or Proposition C.8 in \citet{meng2023benign} for a proof.\par
Now we can summarize current results into the following lemma.

\begin{lemma} (Formal restatement of Lemma \ref{lem:final coefficient})
Under Condition \ref{Con4.1}, there exists $T_1=\Theta(\eta^{-1}nm\sigma_p^2d^{-1})$, for $t \in\left[T_1, T^*\right]$ we have the following hold:
\begin{itemize}
    \item $\gamma_{j, r, l}^{(t)}=\Theta\left( \dfrac{\tau_l\|\boldsymbol{\mu}_l \|_2^2}{d \sigma_p^2} \right) \sum_{i=1}^n \bar{\rho}_{j, r, i}^{(t)}$, for all $j \in\{ \pm 1\}, r \in[m]$ and $l \in \{1, 2\}$ (from Lemma \ref{lem:E.16}).
    \item $\sum_{i=1}^n \bar{\rho}_{j, r, i}^{(t)}=\Omega(n) = O (n \log (T^*)) = \widetilde{\Theta}(n)$,  for all $j \in\{ \pm 1\}, r \in[m]$ and $l \in \{1, 2\}$ (from Proposition \ref{prop:E.8} and Lemma \ref{lem:E.10}).
    \item $\max_{j,r,i}\lvert \underline{\rho}_{j, r, i}^{(t)} \rvert = \max \{ O(\sigma_0 \sigma_p \sqrt{d}\cdot\sqrt{ \log (\dfrac{mn}{\delta})}),O(\sqrt{\log(\dfrac{n}{\delta})}\log(T^*) \cdot n / \sqrt{d}) \}$, for all $j \in\{ \pm 1\}, r \in[m]$ and $l \in \{1, 2\}$ (from Lemma \ref{lem:E.15}).
\end{itemize}
\label{app:lem:final coefficient}\end{lemma}

\begin{lemma}
Under Condition \ref{Con4.1}, there exists $t=\widetilde{O}\left(\eta^{-1} \varepsilon^{-1} m n d^{-1} \sigma_p^{-2}\right)$,  we have:
\begin{equation}
\begin{aligned}
& \left\|\mathbf{w}_{j, r}^{(t)}\right\|_2 \leq \Theta\left(\sigma_p^{-1} d^{-\frac{1}{2}} n^{\frac{1}{2}}\right), \\
& \left\langle \mathbf{w}_{y, r}^{(t)}, y \boldsymbol{\mu}_l\right\rangle=\Theta\left( \gamma_{y, r, l}^{(t)} \right), \\
&\left\langle \mathbf{w}_{-y, r}^{(t)}, y \boldsymbol{\mu}_l\right\rangle=-\Theta\left(\gamma_{-y, r, l}^{(t)}\right)<0.
\end{aligned}
\label{Eq_wt}\end{equation}
for all $j \in\{ \pm 1\}, r \in[m]$ and $l \in \{1, 2\}$.
\label{app:lemma_wt}\end{lemma}
\textit{Proof of Lemma \ref{app:lemma_wt}.} Recall the signal-noise decomposition of $\mathbf{w}_{j, r}^{(t)}$:
$$
\mathbf{w}_{j, r}^{\left(t\right)}=\mathbf{w}_{j, r}^{(0)}+j \cdot \sum_{l=1}^{2} \gamma_{j, r, l}^{\left(t\right)} \cdot \frac{\boldsymbol{\mu}_l}{\|\boldsymbol{\mu}_l\|_2^2}+\sum_{i=1}^n \bar{\rho}_{j, r, i}^{\left(t\right)} \cdot \frac{\boldsymbol{\xi}_i}{\left\|\boldsymbol{\xi}_i\right\|_2^2}+\sum_{i=1}^n \underline{\rho}_{j, r, i}^{\left(t\right)} \cdot \frac{\boldsymbol{\xi}_i}{\left\|\boldsymbol{\xi}_i\right\|_2^2}.
$$
For $l \in \{1, 2\}$, we can bound the inner product with $j=y$:
\begin{equation}
\begin{aligned}
\left\langle\mathbf{w}_{y, r}^{(t)}, y \boldsymbol{\mu}_l\right\rangle = & \left\langle\mathbf{w}_{y, r}^{(0)}, y \boldsymbol{\mu}_l\right\rangle+\gamma_{y, r, l}^{(t)}+\sum_{i=1}^n \bar{\rho}_{y, r, i}^{(t)} \cdot\left\|\boldsymbol{\xi}_i\right\|_2^{-2} \cdot\left\langle\boldsymbol{\xi}_i, y \boldsymbol{\mu}_l\right\rangle+\sum_{i=1}^n \underline{\rho}_{y, r, i}^{(t)} \cdot\left\|\boldsymbol{\xi}_i\right\|_2^{-2} \cdot\left\langle\boldsymbol{\xi}_i, y \boldsymbol{\mu}_l\right\rangle \\
\geq & \gamma_{y, r, l}^{(t)}-\sqrt{2 \log (\dfrac{16m}{\delta}}) \cdot \sigma_0\|\boldsymbol{\mu}_l\|_2 -\sqrt{2 \log (\dfrac{12n}{\delta})} \cdot \sigma_p\|\boldsymbol{\mu}_l\|_2 \cdot\left(\dfrac{\sigma_p^2 d}{2}\right)^{-1}\left[\sum_{i=1}^n \bar{\rho}_{y, r, i}^{(t)}+\sum_{i=1}^n \mid \underline{\rho}_{y, r, i}^{(t)}\right] \\
= & \gamma_{y, r, l}^{(t)}-\Theta\left(\sqrt{\log (\dfrac{m}{\delta})} \sigma_0\|\boldsymbol{\mu}_l\|_2\right)-\Theta\left(\sqrt{\log (\dfrac{n}{\delta})} \cdot\left(\sigma_p d\right)^{-1}\|\boldsymbol{\mu}_l\|_2\right) \cdot \Theta\left(\operatorname{SNR}_l^{-2}\right) \cdot \gamma_{y, r, l}^{(t)} \\
= & {\left[1-\Theta\left(\sqrt{\log (\dfrac{n}{\delta})} \cdot \sigma_p /\|\boldsymbol{\mu}_l\|_2\right)\right] \gamma_{y, r, l}^{(t)}-\Theta\left(\sqrt{\log (\dfrac{m}{\delta})}\left(\sigma_p d\right)^{-1} \sqrt{n}\|\boldsymbol{\mu}_l\|_2\right) } \\
= & \Theta\left(\gamma_{y, r, l}^{(t)}\right),
\end{aligned}
\end{equation}
where the inequality is justified by Lemma \ref{Lem:E.1} and Lemma \ref{Lem:E.2}. The second equality is obtained by substituting the coefficient scales in \ref{app:lem:final coefficient}. The third equality follows from the condition $\sigma_0 \leq C^{-1}\left(\sigma_p d\right)^{-1} \sqrt{n}$ in Condition \ref{Con4.1} and the feature-to-noise ratio $\operatorname{SNR}_l = \dfrac{\|\boldsymbol{\mu}_l\|_2}{\sigma_p \sqrt{d}}$. For the fourth equality, it should be noted that $\gamma_{j, r, l}^{(t)}=\Omega(\tau_l n \cdot \operatorname{SNR}_l^2)$, and also $\sqrt{\log (\dfrac{n}{\delta})} \cdot \dfrac{\sigma_p}{\|\boldsymbol{\mu}_l\|_2} \leq 1 / \sqrt{C}$ and $\sqrt{\log (\dfrac{m}{\delta})}\left(\sigma_p d\right)^{-1} \dfrac{\sqrt{n}\|\boldsymbol{\mu}_l\|_2 }{\tau_l n \cdot \operatorname{SNR}_l^2} =$ $\sqrt{\log (\dfrac{m}{\delta})} \dfrac{\sigma_p}{\tau_l \sqrt{n}\|\boldsymbol{\mu}_l\|_2} \leq \sqrt{\log (\dfrac{m}{\delta}) / n} \cdot 1 /(\sqrt{C \log (\dfrac{n}{\delta})}) \leq 1 /(C \sqrt{\log (\dfrac{n}{\delta})})$, which holds due to $\|\boldsymbol{\mu}_l\|_2^2 \geq C \cdot \sigma_p^2 \log (\dfrac{n}{\delta})$ and $n \geq C \log (\dfrac{m}{\delta})$ in Condition \ref{Con4.1}. Therefore, for a sufficiently large constant $C$, the equality holds. Moreover, we can deduce in a similar manner that
\begin{equation}
\begin{aligned}
\left\langle\mathbf{w}_{-y, r}^{(t)}, y \boldsymbol{\mu}_l\right\rangle= & \left\langle\mathbf{w}_{-y, r}^{(0)}, y \boldsymbol{\mu}_l\right\rangle-\gamma_{-y, r, l}^{(t)}+\sum_{i=1}^n \bar{\rho}_{-y, r, i}^{(t)} \cdot\left\|\boldsymbol{\xi}_i\right\|_2^{-2} \cdot\left\langle\boldsymbol{\xi}_i,-y \boldsymbol{\mu}_l\right\rangle+\sum_{i=1}^n \underline{\rho}_{-y, r, i}^{(t)} \cdot\left\|\boldsymbol{\xi}_i\right\|_2^{-2} \cdot\left\langle\boldsymbol{\xi}_i, y \boldsymbol{\mu}_l\right\rangle \\
\leq & -\gamma_{-y, r, l}^{(t)}+\sqrt{2 \log (\dfrac{16m}{\delta}}) \cdot \sigma_0\|\boldsymbol{\mu}_l\|_2 +\sqrt{2 \log (\dfrac{12n}{\delta})} \cdot \sigma_p\|\boldsymbol{\mu}_l\|_2 \cdot(\dfrac{\sigma_p^2 d}{2})^{-1}[\sum_{i=1}^n \bar{\rho}_{-y, r, i}^{(t)}+\sum_{i=1}^n\lvert\underline{\rho}_{-y, r, i}^{(t)}\rvert] \\
= & -\Theta\left(\gamma_{-y, r, l}^{(t)}\right)<0 .
\end{aligned}
\end{equation}
Next, we seek to upper bound $\| \mathbf{w}_{j, r}^{(t)} \|_2$. The techniques are similar to Proposition D.5 in \citet{meng2023benign}. We first tackle the noise term in the decomposition, namely:
\begin{equation}
\begin{aligned}
& \left\|\sum_{i=1}^n \rho_{j, r, i}^{(t)} \cdot \dfrac{\boldsymbol{\xi}_i}{\|\boldsymbol{\xi}_i\|_2^{2}} \right\|_2^2 \\
= & \sum_{i=1}^n \rho_{j, r, i}^{(t)} \cdot \|\boldsymbol{\xi}_i\|_2^{-2}+2 \sum_{1 \leq i_1<i_2 \leq n} \rho_{j, r, i_1}^{(t)} \rho_{j, r, i_2}^{(t)} \cdot \dfrac{\left\langle\boldsymbol{\xi}_{i_1}, \boldsymbol{\xi}_{i_2}\right\rangle}{\| \boldsymbol{\xi}_{i_1}\|_2^{2} \cdot \|\boldsymbol{\xi}_{i_2}\|_2^{2}} \\
\leq & 4 \sigma_p^{-2} d^{-1} \sum_{i=1}^n \rho_{j, r, i}^{(t)}{ }^2+2 \sum_{1 \leq i_1<i_2 \leq n} \rho_{j, r, i_1}^{(t)} \rho_{j, r, i_2}^{(t)} \cdot\left(16 \sigma_p^{-4} d^{-2}\right) \cdot\left(2 \sigma_p^2 \sqrt{d \log \left(\dfrac{6 n^2}{\delta}\right)}\right) \\
= & 4 \sigma_p^{-2} d^{-1} \sum_{i=1}^n \rho_{j, r, i}^{(t)}{ }^2+32 \sigma_p^{-2} d^{-3 / 2} \sqrt{\log \left(\dfrac{6 n^2}{\delta}\right)}\left[\left(\sum_{i=1}^n \rho_{j, r, i}^{(t)}\right)^2-\sum_{i=1}^n \rho_{j, r, i}^{(t)}{ }^2\right] \\
= & \Theta\left(\sigma_p^{-2} d^{-1}\right) \sum_{i=1}^n \rho_{j, r, i}^{(t)}+\widetilde{\Theta}\left(\sigma_p^{-2} d^{-3 / 2}\right)\left(\sum_{i=1}^n \rho_{j, r, i}^{(t)}\right)^2 \\
\leq & {\left[\Theta\left(\sigma_p^{-2} d^{-1} n^{-1}\right)+\widetilde{\Theta}\left(\sigma_p^{-2} d^{-3 / 2}\right)\right]\left(\sum_{i=1}^n \bar{\rho}_{j, r, i}^{(t)}+\sum_{i=1}^n \rho_{j, r, i}^{(t)}\right)^2 } \\
= & \Theta\left(\sigma_p^{-2} d^{-1} n^{-1}\right)\left(\sum_{i=1}^n \bar{\rho}_{j, r, i}^{(t)}\right)^2,
\end{aligned}
\label{eq:11}\end{equation}
where the first inequality is by Lemma \ref{Lem:E.1}; the second inequality is by the Cauchy Schwartz Inequality on $(\sum_{i=1}^n \rho_{j, r, i}^{(t)})^2$. We can then upper bound the $\| \mathbf{w}_{j, r}^{(t)} \|_2$ as:
\begin{equation}
\begin{aligned}
\|\mathbf{w}_{j, r}^{(t)}\|_2 & \leq\left\|\mathbf{w}_{j, r}^{(0)}\right\|_2+\sum_{l=1}^2\dfrac{\gamma_{j, r, l}^{(t)}}{\|\boldsymbol{\mu}_l\|_2}+\left\|\sum_{i=1}^n \rho_{j, r, i}^{(t)} \cdot \dfrac{\boldsymbol{\xi}_i}{\|\boldsymbol{\xi}_i\|_2^{2}} \right\|_2 \\
& \leq\left\|\mathbf{w}_{j, r}^{(0)}\right\|_2+\sum_{l=1}^2\dfrac{\gamma_{j, r, l}^{(t)}}{\|\boldsymbol{\mu}_l\|_2}+\Theta\left(\sigma_p^{-1} d^{-1 / 2} n^{-1 / 2}\right) \cdot \sum_{i=1}^n \bar{\rho}_{j, r, i}^{(t)} \\
& =\Theta\left(\sigma_p^{-1} d^{-1 / 2} n^{-1 / 2}\right) \cdot \sum_{i=1}^n \bar{\rho}_{j, r, i}^{(t)},
\end{aligned}
\label{eq:12}\end{equation}
where the first inequality is due to the triangle inequality, the second inequality is by (\ref{eq:11}), and the third equality is due to the following comparisons:
$$
\frac{\dfrac{\gamma_{j, r, l}^{(t)}}{\|\boldsymbol{\mu}_l\|_2}}{\Theta\left(\sigma_p^{-1} d^{-1 / 2} n^{-1 / 2}\right) \cdot \sum_{i=1}^n \bar{\rho}_{j, r, i}^{(t)}}=\Theta\left(\sigma_p d^{1 / 2} n^{1 / 2}\|\boldsymbol{\mu}_l\|_2^{-1} \operatorname{SNR}_l^2\right)=\Theta\left(\sigma_p^{-1} d^{-1 / 2} n^{1 / 2}\|\boldsymbol{\mu}_l\|_2\right)=O(1),
$$
which is by the coefficient scales in Lemma \ref{app:lem:final coefficient}, the coefficient order $\dfrac{\sum_{i=1}^n \bar{\rho}_{j, r, i}^{(t)}}{ \gamma_{j, r, l}^{(t)}}=\Theta\left(\operatorname{SNR}_l^{-2}\right)$, and the $d$ condition in Condition \ref{Con4.1}; and also we have:
$$
\frac{\left\|\mathbf{w}_{j, r}^{(0)}\right\|_2}{\Theta\left(\sigma_p^{-1} d^{-1 / 2} n^{-1 / 2}\right) \cdot \sum_{i=1}^n \bar{\rho}_{j, r, i}^{(t)}}=\frac{\Theta\left(\sigma_0 \sqrt{d}\right)}{\Theta\left(\sigma_p^{-1} d^{-1 / 2} n^{-1 / 2}\right) \cdot \sum_{i=1}^n \bar{\rho}_{j, r, i}^{(t)}}=O\left(\sigma_0 \sigma_p d n^{-1 / 2}\right)=O(1),
$$
which is by the coefficient scales in Lemma \ref{app:lem:final coefficient}, and the condition for $\sigma_0$ in Condition \ref{Con4.1}. Apply the coefficient order $\sum_{i=1}^n \bar{\rho}_{j, r, i}^{(t)}=\Omega(n)$ to (\ref{eq:12}), we directly have $\left\|\mathbf{w}_{j, r}^{(t)}\right\|_2 \leq \Theta\left(\sigma_p^{-1} d^{-\frac{1}{2}} n^{\frac{1}{2}}\right)$.
\subsection{Order-dependent Sampling (Querying) Analysis}\label{app:Order-dependent Sampling(Querying Analysis)}
Based on the scale of $\mathbf{w}_{j, r}^{(t)}$ and the inner product between it and features, we can now characterize the querying situation of two query criteria-based NAL methods. First, to address the issue of $\Theta (\lvert \mathcal{P} \rvert^2 )$ comparisons in $\mathcal{P}$, we employ a full-order-based technique. We introduce the concepts of Uncertainty Order and Diversity Order in Appendix \ref{App:Order def}. Subsequently, we delve into the order of the samples in $\mathcal{P}$ in the following proposition.
\begin{proposition}
     Under the same conditions of Proposition \ref{prop:sample stage: learning on sample}, there exist $t=\widetilde{O}\left(\eta^{-1} \varepsilon^{-1} m n d^{-1} \sigma_p^{-2}\right)$ that for $\forall \mathbf{x}, \mathbf{x}^\prime \in \mathcal{P} \subsetneq \mathcal{D}$ where $\mathbf{x}$ contains weak feature patch while $\mathbf{x}^\prime$ contains strong feature patch, with probability at least 1-$\delta^\prime$, we have $\mathbf{x}^\prime \preceq^{(t)} \mathbf{x}$. 
\label{prop_App:order_newdata}
\end{proposition}
\textit{Proof of Proposition \ref{prop_App:order_newdata}.}
Firstly, suggest $ \mathbf{x}=[y \cdot  \boldsymbol{\mu}_2, \mathbf{z}_2], \mathbf{x}^{\prime}=[y^{\prime} \cdot \boldsymbol{\mu}_1,\mathbf{z}_1]$, where $\mathbf{z}_1, \mathbf{z}_2 \sim N(\mathbf{0}, \sigma_p^2 \cdot \mathbf{I})$:
\[
\begin{aligned}
&
f\!\left( \mathbf{W}^{(t)}, \mathbf{x} \right)\!=\sum_{j, r} \frac{j}{m}\left[\sigma\!\left( 
 \left\langle \mathbf{w}_{j, r}^{(t)}, y \boldsymbol{\mu}_2\right\rangle\right)\thinspace+ \sigma\!\left( \left\langle\mathbf{w}_{j, r}^{(t)}, \mathbf{z}_2 \right\rangle\right)\! \right],\\
&
f\!\left( \mathbf{W}^{(t)}, \mathbf{x}^{\prime} \right)\!= \sum_{j, r} \frac{j}{m}\left[\sigma\!\left( \left\langle \mathbf{w}_{j, r}^{(t)}, y^{\prime} \boldsymbol{\mu}_1 \right\rangle \right)\!+ \sigma\!\left( \left\langle\mathbf{w}_{j, r}^{(t)}, \mathbf{z}_1 \right\rangle\right)\! \right].
\end{aligned}
\]
By (\ref{Eq:order_absolute}) in Lemma \ref{Lemma:order_absolute} and (\ref{Eq:Diversity_Order}) in Definition\ref{Diversity_Order}, we have the following 
\begin{align*}
\mathbf{x}^\prime \preceq_C^{(t)} \mathbf{x} &\Leftrightarrow \underbrace{\left| f\left(\mathbf{W}^{(t)}, \mathbf{x} \right) \right|< \left| f\left(\mathbf{W}^{(t)}, \mathbf{x}^{\prime} \right) \right|}_{\Omega_C},\\
\mathbf{x}^{\prime} \preceq_{D}^{(t)} \mathbf{x} &\Leftrightarrow \underbrace{ D\left(\mathbf{W}^{(t)}, \mathbf{x}, p \ \mid \mathcal{D}_{n_0}\right) > D\left(\mathbf{W}^{(t)}, \mathbf{x}^\prime, p \ \mid \mathcal{D}_{n_0}\right)}_{\Omega_{D}},\\
\mathbf{x}^{\prime} \preceq^{(t)} \mathbf{x} &\Leftrightarrow \underbrace{ \{ \Omega_C \cap \Omega_{D},\forall p \in \left[1, \infty\right) \}}_{\Omega}
\end{align*}
Denote $\sum_{j} j \cdot \sigma\left(\left\langle\mathbf{w}_{j, r}^{(t)}, \mathbf{z}_1 \right\rangle \right)$, $\sum_{j} j \cdot \sigma\left(\left\langle\mathbf{w}_{j, r}^{(t)}, \mathbf{z}_2\right\rangle \right)$ as $g_r(\mathbf{z}_1)$, $g_r(\mathbf{z}_2)$ respectively, Notice that for $\mathbf{z} \sim N(\mathbf{0}, \sigma_p^2 \cdot \mathbf{I})$: 
\begin{equation}
\begin{aligned}
&\left\langle \mathbf{w}_{j, r}^{(t)}, \mathbf{z}\right\rangle \sim \mathcal{N}\left(0,\left\| \mathbf{w}_{j, r}^{(t)} \right\|_2^2 \sigma_p^2 \cdot \mathbf{I} \right), \\
& \sigma(\left\langle \mathbf{w}_{j, r}^{(t)}, \mathbf{z}\right\rangle) \sim \mathcal{N}^R\left(0,\left\| \mathbf{w}_{j, r}^{(t)} \right\|_2^2 \sigma_p^2 \cdot \mathbf{I} \right).
\end{aligned}
\label{Eq:<w·z>}\end{equation}
Then:
\begin{equation}
\begin{aligned}
P(\Omega_C) & =  P (\left| f\left(\mathbf{W}^{(t)}, \mathbf{x} \right) \right| < \left| f\left(\mathbf{W}^{(t)}, \mathbf{x}^{\prime} \right) \right|)\\
& \geq P(\sum_{l}(\sum_{r} \lvert g_r(\mathbf{z}_l) \rvert)  < \sum_{r}(\Theta(\gamma_{y^{\prime}, r, 1})-\Theta(\gamma_{y, r, 2})) )\\
& \geq P( m \cdot \max_{j, r, l}\{ \left| \left\langle\mathbf{w}_{j, r}^{(t)}, \mathbf{z}_l\right\rangle \right|\} < m (\Theta( \underset{r}{\mathbb{E}} (\gamma_{y^{\prime}, r, 1}))-\Theta(\underset{r}{\mathbb{E}}(\gamma_{y, r, 2}))) )\\
& = P (\underbrace{\max_{j, r, l}\{ \left| \left\langle\mathbf{w}_{j, r}^{(t)}, \mathbf{z}_l\right\rangle \right|\} < \Theta((\underset{r}{\mathbb{E}} (\gamma_{y^{\prime}, r, 1})-\underset{r}{\mathbb{E}}(\gamma_{y, r, 2}))}_{\Omega_{\gamma}}). \\
\end{aligned}
\label{Eq:OmegaC}\end{equation}
The second inequality is by triangle inequality and (\ref{Eq_wt}) in Lemma \ref{app:lemma_wt}; the third inequality is by Lemma \ref{app:lem:final coefficient}.\par
For $\Omega_{D}$, denoting $U_0^l= \{ \mathbf{x} \in \mathcal{D}_0 \mid \mathbf{x}_{\text {signal part }}=\boldsymbol{\mu}_l\}$ as the set of indices of $\mathcal{D}_0$ where the data's feature patch is $\boldsymbol{\mu}_l$, We then take a look at the $r^{th} $ row of the Feature Distance $\mathbf{Z}(\mathbf{x}, t)$, which we denote as $\mathbf{Z}_r(\mathbf{x}, t)$:
\begin{equation}
\begin{aligned}
\mathbf{Z}_r(\mathbf{x}, t)&=\sum_j\left(\sigma \left(\left\langle \mathbf{w}_{j, r}, y \cdot \boldsymbol{\mu}_2 \right\rangle \right) + \sigma\left(\left\langle \mathbf{w}_{j, r}, \mathbf{z}_{r}\right\rangle\right)\right) \\
& =\Theta\left(\gamma_{y, r, 2}\right)+ g_r(\mathbf{z}_2) 
\end{aligned}
\label{Eq:rth row of FD sample}\end{equation}
\begin{equation}
\begin{aligned}
 \sum_{i} \dfrac{\mathbf{Z}_r(  \mathbf{x}^{(i)}, t)}{n_0}&=\sum_{i,j} \frac{\sigma\left(\left\langle \mathbf{w}_{j, r}, y_i \cdot \boldsymbol{\mu}^{(i)} \right\rangle \right)+\sigma\left(\left\langle \mathbf{w}_{j, r}, \boldsymbol{\xi}_i \right\rangle \right)}{n_0} \\
& =\dfrac{\left[\sum_{l} \tau_l \cdot n_0 \cdot \underset{i_l \in U_0^l}{\mathbb{E}} \Theta (\gamma_{y_{i_l}, r, l}) +\sum_{i} \sum_{j} \Theta\left(\bar{\rho}_{j, r, i}\right)\right]}{n_0}
\end{aligned}
\label{Eq:rth row of FD labeledset}\end{equation}
Let (\ref{Eq:rth row of FD sample}) - (\ref{Eq:rth row of FD labeledset}), we have:
\begin{equation}
\mathbf{Z}_r(\mathbf{x}, t) - \sum_{i} \dfrac{\mathbf{Z}_r( \mathbf{x}^{(i)}, t)}{n_0}= \Theta(\gamma_{y, r, 2}) + g_r(\mathbf{z}_2) - \sum_{i} \dfrac{\mathbf{Z}_r( \mathbf{x}^{(i)}, t)}{n_0}
\end{equation}

Now we can estimate $D\left(\mathbf{W}^{(t)}, \mathbf{x}, p \ \mid \mathcal{D}_{n_0}\right)$:
\begin{equation}
\begin{aligned}
     D\left(\mathbf{W}^{(t)}, \mathbf{x}, p \ \mid \mathcal{D}_{n_0}\right) &= \| \mathbf{Z}(\mathbf{x}, t) - \sum_{i=1}^{n_0} \dfrac{\mathbf{Z}(  \mathbf{x}^{(i)}, t)}{n_0} \|_p \\
    & =\left( \sum_{r}  \lvert  \mathbf{Z}_r(\mathbf{x}, t) - \sum_{i} \dfrac{\mathbf{Z}_r( \mathbf{x}^{(i)}, t)}{n_0} \rvert^p \right)^{\frac{1}{p}}\\
    & = \left( \sum_{r}  \lvert  \Theta(\gamma_{y, r, 2}) + g_r(\mathbf{z}_2) - \sum_{i} \dfrac{\mathbf{Z}_r( \mathbf{x}^{(i)}, t)}{n_0}  \rvert^p \right)^{\frac{1}{p}}
\end{aligned}
\label{Eq:D_x}\end{equation}
Similarly, the $D\left(\mathbf{W}^{(t)}, \mathbf{x}^\prime, p \ \mid \mathcal{D}_{n_0}\right)$ could be written as:
\begin{equation}
     D\left(\mathbf{W}^{(t)}, \mathbf{x}^\prime, p \ \mid \mathcal{D}_{n_0}\right) = \left( \sum_{r}  \lvert  \Theta(\gamma_{y, r, 1})  + g_r(\mathbf{z}_1)- \sum_{i} \dfrac{\mathbf{Z}_r( \mathbf{x}^{(i)}, t)}{n_0}  \rvert^p \right)^{\frac{1}{p}}
\label{Eq:D_x'}\end{equation}
To compare $D\left(\mathbf{W}^{(t)}, \mathbf{x}, p \ \mid \mathcal{D}_{n_0}\right)$ and $D\left(\mathbf{W}^{(t)}, \mathbf{x}^\prime, p \ \mid \mathcal{D}_{n_0}\right)$, we first see that both expressions in the $r$-th filter owns 
$$ - \sum_{i} \dfrac{\mathbf{Z}_r( \mathbf{x}^{(i)}, t)}{n_0}  = - \sum_l \tau_l \cdot \Theta(\underset{i_l \in U_0^l}{\mathbb{E}}(\gamma_{y_{i_l}, r, l}) ) - n_{0}^{-1}\sum_{i} \sum_{j} \Theta\left(\bar{\rho}_{j, r, i}\right).
$$
By Condition \ref{Con4.1}, we see that $\sigma_{p}^{2} d /(n_{0}\|\boldsymbol{\mu}_{1}\|_{2}^{2}) = \Omega(\log(T^{*}))$. We see that as $T^{*}$ is the substantially large maximum admissible iterations, collaborating with (\ref{Eq_wt}), (\ref{Eq:rth row of FD labeledset}) and (\ref{Eq:<w·z>}), it holds that the order of $n_{0}^{-1} \sum_{i,j} \sigma\left(\left\langle \mathbf{w}_{j, r}, \boldsymbol{\xi}_i \right\rangle \right)=n_{0}^{-1}\sum_{i} \sum_{j} \Theta\left(\bar{\rho}_{j, r, i}\right)$ in $\sum_{i} \dfrac{\mathbf{Z}_r( \mathbf{x}^{(i)}, t)}{n_0}$ is indeed can dominate $n_{0}^{-1} \sum_{i,j}\sigma\left(\left\langle \mathbf{w}_{j, r}, y_i \cdot \boldsymbol{\mu}^{(i)} \right\rangle \right)=\sum_l \tau_l \cdot \Theta(\underset{i_l \in U_0^l}{\mathbb{E}}(\gamma_{y_{i_l}, r, l}) )$, $\Theta(\gamma_{y, r, 1}) $ and $ g_r(\mathbf{z}_1)$. As $\sum_{i} \dfrac{\mathbf{Z}_r( \mathbf{x}^{(i)}, t)}{n_0}$ is shared by both $D\left(\mathbf{W}^{(t)}, \mathbf{x}, p \ \mid \mathcal{D}_{n_0}\right)$ and $D\left(\mathbf{W}^{(t)}, \mathbf{x}^\prime, p \ \mid \mathcal{D}_{n_0}\right)$ in the $r$-th filter, a sufficient event for $D\left(\mathbf{W}^{(t)}, \mathbf{x}, p \ \mid \mathcal{D}_{n_0}\right) > D\left(\mathbf{W}^{(t)}, \mathbf{x}^\prime, p \ \mid \mathcal{D}_{n_0}\right)$ is that for $\forall r \in [m]$, we have 

\[
\lvert \sum_l \tau_l \cdot \Theta(\underset{i_l \in U_0^l}{\mathbb{E}}(\gamma_{y_{i_l}, r, l}) ) - \Theta(\gamma_{y, r, 2}) - g_r(\mathbf{z}_2) \rvert > \lvert \max\{ \sum_l \tau_l \cdot \Theta(\underset{i_l \in U_0^l}{\mathbb{E}}(\gamma_{y_{i_l}, r, l}) ) - \Theta(\gamma_{y, r, 1}) - g_r(\mathbf{z}_1),  0\}\rvert.
\]

Utilizing those results, we now could estimate the chance of event $\Omega_{D}$:
\begin{equation}
\begin{aligned}
P(\Omega_{D}) & = P (D\left(\mathbf{W}^{(t)}, \mathbf{x}, p \ \mid \mathcal{D}_{n_0}\right) > D\left(\mathbf{W}^{(t)}, \mathbf{x}^\prime, p \ \mid \mathcal{D}_{n_0}\right))\\
& \geq P(m^{\frac{1}{p}}\sum_{l}(\max_{r} \lvert g_r(\mathbf{z}_l) \rvert) < m^{\frac{1}{p}}(\lvert  \Theta(\underset{r}{\mathbb{E}}(\gamma_{y, r, 2}))-\sum_l \tau_l \cdot \Theta(\underset{i_l \in U_0^l, r}{\mathbb{E}}(\gamma_{y_{i_l}, r, l}) )  \rvert \\
& \phantom{\geq P(m^{\frac{1}{p}}\sum_{l}(\sum_{r} \lvert g_r(\mathbf{z}_l) \rvert) < m^{\frac{1}{p}}} - \lvert  \Theta(\underset{r}{\mathbb{E}}(\gamma_{y, r, 1}))-\sum_l \tau_l \cdot \Theta(\underset{i_l \in U_0^l, r}{\mathbb{E}}(\gamma_{y_{i_l}, r, l}) )  \rvert) \\
& \geq P ( m^{\frac{1}{p}} \max_{j, r, l}\{ \left| \left\langle\mathbf{w}_{j, r}^{(t)}, \mathbf{z}_l\right\rangle \right|\} < m^{\frac{1}{p}}\left((\tau_1 - \tau_2) \Theta(\underset{j,r}{\mathbb{E}}(\gamma_{j, r, 1})) - (\tau_1 - \tau_2) \Theta(\underset{j,r}{\mathbb{E}}(\gamma_{j, r, 2}))\right) \\
& = P (m^{\frac{1}{p}} \max_{j, r, l}\{ \left| \left\langle\mathbf{w}_{j, r}^{(t)}, \mathbf{z}_l\right\rangle \right|\} < m^{\frac{1}{p}}\Theta(\dfrac{\tau_1 (\tau_1 - \tau_2)\| \boldsymbol{\mu}_1\|_2^2- \tau_2 (\tau_1 - \tau_2)\| \boldsymbol{\mu}_2\|_2^2}{\sigma_p^2 d/n_0}))\\
& = P (m^{\frac{1}{p}} \max_{j, r, l}\{ \left| \left\langle\mathbf{w}_{j, r}^{(t)}, \mathbf{z}_l\right\rangle \right|\} < m^{\frac{1}{p}}\Theta(\underset{r}{\mathbb{E}} (\gamma_{y^{\prime}, r, 1})-\underset{r}{\mathbb{E}}(\gamma_{y, r, 2})))\\
& = P (\underbrace{\max_{j, r, l}\{ \left| \left\langle\mathbf{w}_{j, r}^{(t)}, \mathbf{z}_l\right\rangle \right|\} < \Theta((\underset{r}{\mathbb{E}} (\gamma_{y^{\prime}, r, 1})-\underset{r}{\mathbb{E}}(\gamma_{y, r, 2}))}_{\Omega_{\gamma}}),
\end{aligned}
\label{Eq:OmegaD}\end{equation}
where the first inequality is by Lemma \ref{app:lem:final coefficient}, triangle inequality, (\ref{Eq_wt}), (\ref{Eq:D_x}) and (\ref{Eq:D_x'}); The forth equality is by (\ref{Eq:<w·z>}). Easy to see that if $p=\infty$, the third equality would be zero, thus our condition $p < \infty$ avoid this case. Now we take a look at the event $\Omega_{\gamma}$:
\begin{equation}
\begin{aligned}
P (\Omega_{\gamma}) & = P (\max_{j, r, l}\{ \left| \left\langle\mathbf{w}_{j, r}^{(t)}, \mathbf{z}_l\right\rangle \right|\} < \Theta((\underset{r}{\mathbb{E}} (\gamma_{y^{\prime}, r, 1})-\underset{r}{\mathbb{E}}(\gamma_{y, r, 2})))\\
& = P (\max_{j, r, l}\{ \left| \left\langle\mathbf{w}_{j, r}^{(t)}, \mathbf{z}_l\right\rangle \right|\} < \Theta\left(\dfrac{\left[\tau_1 \left\|\boldsymbol{\mu}_1\right\|_2^2-\tau_2\left\|\boldsymbol{\mu}_2\right\|_2^2\right]}{\sigma_p^2 d / n_0 }\right) )\\
& \geq P (\bigcup_{j, r, l} \underbrace{\{ \left| \left\langle\mathbf{w}_{j, r}^{(t)}, \mathbf{z}_l\right\rangle - 0 \right| < \Theta\left(\dfrac{\left[\tau_1 \left\|\boldsymbol{\mu}_1\right\|_2^2-\tau_2\left\|\boldsymbol{\mu}_2\right\|_2^2\right]}{\sigma_p^2 d / n_0 }\right) \} }_{\hat{\Omega}_{j, r, l}})\\
& = \sum_{j, r, l} P(\hat{\Omega}_{j, r, l}),
\end{aligned}
\label{eq:Omega}\end{equation}
where the second equality is by the first inference statement of Lemma \ref{app:lem:final coefficient}; the third inequality is by the equivalence property of the union by events; the last equality is by the Union Rule. Then, by Gaussian tail bound, we have:
\[
P(\hat{\Omega}_{j, r, l}) \geq 1- 2 \exp \left\{-\Theta\left(\dfrac{\left[ \tau_1 \left\|\boldsymbol{\mu}_1\right\|_2^2- \tau_2 \left\|\boldsymbol{\mu}_2\right\|_2^2\right]^2}{\sigma_p^{6} d^2 / n_0^2 \left\|w_{j, r}^{(t)}\right\|_2^2}\right)\right\}
\]
Finally, with conditions on $\| \boldsymbol{\mu}_1\|_2^2 - \|\boldsymbol{\mu}_2 \|_2^2$ in Proposition \ref{prop:sample stage: learning on sample},  Lemma \ref{lem:rho n}, (\ref{Eq_wt}) in Lemma \ref{app:lemma_wt} and union bound, we have the conclusion for event $\Omega$:
\begin{equation}
\begin{aligned}
\Rightarrow P(\Omega) \geq P (\Omega_{\gamma})  & \geqslant 1-8 m \exp \left\{-\Theta\left(\frac{\left[ \tau_1 \left\|\boldsymbol{\mu}_1\right\|_2^2- \tau_2 \left\|\boldsymbol{\mu}_2\right\|_2^2\right]^2}{\sigma_p^{4} d / n_0 }\right)\right\}\\
& \geqslant 1-\delta^{\prime},
\end{aligned}
\label{eq:sufficient OMEGA}\end{equation}
for $\forall p \in \left[1, \infty\right)$. \par
\begin{remark}
We can observe that the Uncertainty Order and Diversity Order of samples rely heavily on the model's learning progresss upon them. By Lemma \ref{app:lem:final coefficient}, the learning progresss of samples depend heavily on the feature strength $\| \boldsymbol{\mu}_l \|_2$ and data proportion $\tau_l$. That is to say, in our case, the \hyperlink{msg:perplexing}{\textbf{perplexing samples}} are the samples containing weak feature $\boldsymbol{\mu}_2$. In the next section, we would show that the number of those \hyperlink{msg:perplexing}{\textbf{perplexing samples}} in the labeled set after querying would determine the algorithm's generalization ability.
\end{remark}
From the above proving process, we can deduce some important findings, which can be summarized in the following lemmas. \par
The following lemma shows that Uncertainty Sampling and Diversity Sampling correspond to different comparisons on the model's learning progress over samples in $\mathcal{P}$.
\begin{lemma} 
 (Restatement of Lemma \ref{lemma of comparison}) Under the same conditions in Proposition \ref{prop:sample stage: learning on sample}, with the same notations in Proposition \ref{prop_App:order_newdata}, there exists certain constants $c_1, c_2, c_3, c_4, c_5, c_6 > 0$, such that
 \begin{itemize}
     \item $\mathbf{x} \preceq_{C}^{(t)} \mathbf{x}^{\prime}$ has a sufficient event that
\begin{equation}
 \{ c_1 \underset{r}{\mathbb{E}} (\gamma_{y^{\prime}, r, 1}) - c_2 \underset{r}{\mathbb{E}}(\gamma_{y, r, 2}) > \max_{j, r, l}\{ \left| \left\langle\mathbf{w}_{j, r}^{(t)}, \mathbf{z}_l\right\rangle \right|\}\},
\label{eqapp:Comparision_Uncertainty}\end{equation}
  among which the left side of the inequality corresponds to the comparison of learning progress of samples with different type of feature patch.
    \item $\mathbf{x} \preceq_{D}^{(t)} \mathbf{x}^{\prime}, \forall p \in [1, \infty)$ has a sufficient event that
 \begin{equation}
 \{ \lvert  c_3 \underset{r}{\mathbb{E}}(\gamma_{y, r, 2})-c_4 \sum_l \tau_l \cdot \underset{i_l \in U_0^l, r}{\mathbb{E}}(\gamma_{y_{i_l}, r, l}) \rvert - \lvert  c_5 \underset{r}{\mathbb{E}}(\gamma_{y^{\prime}, r, 1})-c_6 \sum_l \tau_l \cdot \underset{i_l \in U_0^l, r}{\mathbb{E}}(\gamma_{y_{i_l}, r, l}) \rvert > \max_{j, r, l}\{ \left| \left\langle\mathbf{w}_{j, r}^{(t)}, \mathbf{z}_l\right\rangle \right|\} \},
\label{eqapp:Comparision_Diversity}\end{equation}
  among which the left side of the inequality corresponds to the comparison of the disparity between learning toward samples and labeled training set.
 \end{itemize}
\label{app:lemma of comparison}\end{lemma}

\textit{Proof of Lemma \ref{app:lemma of comparison}.} The first bullet point can be easily derived from (\ref{Eq:OmegaC}), while the second bullet point is readily apparent from (\ref{Eq:D_x}), (\ref{Eq:D_x'}), and (\ref{Eq:OmegaD}). \par
During the proving process of Proposition \ref{prop_App:order_newdata}, it is observed that for any $p \in [1, \infty)$, there exists a shared sufficient event for (\ref{eqapp:Comparision_Uncertainty}) and (\ref{eqapp:Comparision_Diversity}). This implies that it is also a shared sufficient event for the events $\Omega_C$ and $\Omega_{D}$, denoted as $\Omega_{\gamma}$:
\[
\Omega_{\gamma} \mathrel{\mathop:}= \{ \max_{j, r, l}\{ \left| \left\langle\mathbf{w}_{j, r}^{(t)}, \mathbf{z}_l\right\rangle \right|\} < \Theta((\underset{r}{\mathbb{E}} (\gamma_{y^{\prime}, r, 1})-\underset{r}{\mathbb{E}}(\gamma_{y, r, 2})) \}.
\]
By the first inference statement of Lemma \ref{app:lem:final coefficient}, we have 
\begin{equation}
\Omega_{\gamma} = \{ \max_{j, r, l}\{ \left| \left\langle\mathbf{w}_{j, r}^{(t)}, \mathbf{z}_l\right\rangle \right|\} < \Theta((\underset{j, r}{\mathbb{E}} (\gamma_{j, r, 1})-\underset{j, r}{\mathbb{E}}(\gamma_{j, r, 2})) \}.
\label{eq:e.24}\end{equation}
Therefore, we can conclude that the significant difference in the model's learning of the feature $\boldsymbol{\mu}_1$ and $\boldsymbol{\mu}_2$ is what causes the sufficient event for both event $\Omega_C$ and $\Omega_{D}$. By (\ref{eq:sufficient OMEGA}), we have:
\begin{equation}
    P (\Omega_{\gamma}) \geq 1-8 m \exp \left\{-\Theta\left( \underset{j, r}{\mathbb{E}}(\gamma_{j, r, 1}) - \underset{j, r}{\mathbb{E}}(\gamma_{j, r, 2})\right)\right\}.
\end{equation}
Based on Lemma \ref{app:lem:final coefficient}, we see that the $\underset{j, r}{\mathbb{E}}(\gamma_{j, r, 1})$ is significant larger than $\underset{j, r}{\mathbb{E}}(\gamma_{j, r, 2})$ under our conditions, which causes the sufficient event $\Omega_{\gamma}$. \par

Based on the above results, we can have a look on the overall order situation of the sampling pool $\mathcal{P}$.
\begin{lemma} (Restatement of Lemma \ref{Lem:order_pool})
    Under Condition \ref{Con4.1}, when the results of Proposition \ref{Prop4.3} and Proposition \ref{prop_App:order_newdata} hold at the initial stage and querying stage at a certain $t \leq T^*$, denoting $\mathbf{X}_{\mathcal{P}}^1 \subsetneqq \mathcal{P}$ as the collection of all the data points with strong feature $\boldsymbol{\mu}_1$ in $\mathcal{P}$, and  $\mathbf{X}_{\mathcal{P}}^2 \subsetneqq \mathcal{P}$ as the collection of data points with weak feature $\boldsymbol{\mu}_2$, we have the conclusion that with probability more than 1-$\Theta(\delta^\prime)$, $\mathbf{X}_{\mathcal{P}}^1 \prec^{(t)} \mathbf{X}_{\mathcal{P}}^2$ holds.
\label{app:Lem:order_pool_app}\end{lemma}
\textit{proof of Lemma \ref{app:Lem:order_pool_app}.} By Proposition \ref{prop_App:order_newdata}, $\forall \mathbf{x}^\prime \in \mathbf{X}_{\mathcal{P}}^1$, and $\forall \mathbf{x} \in \mathbf{X}_{\mathcal{P}}^2$, $\mathbf{x}^\prime \prec^{(t)} \mathbf{x}$ with at least probability $\delta^\prime$. It's natural to see comparing every pairs in $\mathbf{X}_{\mathcal{P}}^1$ and $\mathbf{X}_{\mathcal{P}}^2$ as independent random events. Then given a certain $\mathbf{x}^\prime \in \mathbf{X}_{\mathcal{P}}^1$, the chance that $\forall \mathbf{x} \in \mathbf{X}_{\mathcal{P}}^2$ satisfies $\mathbf{x}^\prime \prec^{(t)} \mathbf{x}$ is $\Theta((1-\delta^\prime)^{| \mathbf{X}_{\mathcal{P}}^2 |})$, therefore, for $\forall \mathbf{x}^\prime \in \mathbf{X}_{\mathcal{P}}^1$, the chance is  $\Theta((1-\delta^\prime)^{|\mathbf{X}_{\mathcal{P}}^2|\cdot|\mathbf{X}_{\mathcal{P}}^1|})=\Theta((1-\delta^\prime)^{p(1-p)\lvert \mathcal{P} \rvert^2})=1-\Theta(\delta^\prime)$ as $\delta^\prime \ll 1$.\par
Based on Lemma \ref{app:Lem:order_pool_app} and (\ref{eq:e.24}), we directly have the following lemma demonstrate that both NAL algorithms would all prioritize those poor learning samples.
\begin{lemma}
    (Restatement of Proposition \ref{prop:sample stage: learning on sample}) Under the same conditions in Proposition \ref{Prop4.3}, the Uncertainty Order and Diversity Order of the samples $[(y\cdot \boldsymbol{\mu}_l)^T, \mathbf{\xi}^T]^T$  in sampling pool $\mathcal{P}$ follows the order of $\displaystyle \underset{j,r}{\mathbb{E}} \gamma_{j,r,l}^{(t)}$.
\label{app:lem:sample stage: learning on sample}\end{lemma}

\subsection{Label Complexity-based Test Error Analysis}\label{App:Label Complexity-based Test Error Analysis}
In this section, we suggest the results in the previous sections all hold with high probability. With the results of the final scale of the coefficients as well as the order situation of the data in sampling pool $\mathcal{P}$, we can now take a look on the test error upper and lower bound under distinct conditions before and after querying.

\begin{lemma} (Partial restatement of Lemma \ref{main:lem for thm})
    Under Condition \ref{Con4.1}, for a test set $\mathcal{D^*} \subseteq \mathcal{D^*}$ with occurrence probability $p^*$ of the $\boldsymbol{\mu}_2$-equipped data, then $\exists\ t=\widetilde{O}\left(\eta^{-1} \varepsilon^{-1} m n_0 d^{-1} \sigma_p^{-2}\right)$, we have the following two situations before and after querying (i.e., $\forall s \in \{0, 1\}$):
\begin{itemize}
    \item If $\forall l \in \{1,2\}, n_{s, l} \geq \dfrac{C_1 \sigma_p^4 d}{\|\boldsymbol{\mu}_l\|_2^4}$ holds, we have the test error:
    \begin{equation}
    L_{\mathcal{D}^*}^{0-1}\left(\mathbf{W}^{(t)}\right) \leq (1-p^*) \cdot \exp \left( \dfrac{-n_{s, 1}\|\boldsymbol{\mu}_1\|_2^4}{C_3 \sigma_p^4 d} \right) + p^* \cdot \exp \left( \dfrac{-n_{s, 2}\|\boldsymbol{\mu}_2 \|_2^4}{C_4 \sigma_p^4 d} \right).
    \label{eq:E.21}\end{equation}
    \item If $\exists l^{\prime} \in \{1, 2\} n_{s, l^{\prime}} \leq \dfrac{C_2 \sigma_p^4 d}{\|\boldsymbol{\mu}_{l^{\prime}}\|_2^4 }$ holds, where $C_1$ is from Condition \ref{Con4.1}, we have the test error 
    \begin{equation}
    L_{\mathcal{D}^*}^{0-1}\left(\mathbf{W}^{(t)}\right)  \geq 0.12 \cdot p^*_{l^{\prime}}. 
    \label{eq:E.22}\end{equation}
\end{itemize}
Here $p^*_{l^\prime}$ denotes the occurrence probability of feature $\boldsymbol{\mu}_{l^\prime}$, $C_1$, $C_2$, $C_3$ and $C_4$ are some positive constants.
\label{app:lemma for thm}\end{lemma}
\textit{Proof of Lemma \ref{app:lemma for thm}.} Recall the test error definition and consider the proportion of different type of data in the testing set $\mathcal{D}^*$, we have:
\begin{equation}
\begin{aligned}
L_{\mathcal{D}^*}^{0-1}(\mathbf{W}) &=\mathbb{P}_{(\mathbf{x}, y) \sim \mathcal{D}^*}[y \cdot f(\mathbf{W}, \mathbf{x}) < 0]\\
&=(1-p^*) \cdot \mathbb{P}_{(\mathbf{x}, y) \sim \mathcal{D}_{\boldsymbol{\mu}_1}^*}[y \cdot f(\mathbf{W}, \mathbf{x}) < 0] +p^* \cdot \mathbb{P}_{(\mathbf{x}, y) \sim \mathcal{D}_{\boldsymbol{\mu}_2}^*}[y \cdot f(\mathbf{W}, \mathbf{x}) < 0],
\end{aligned}
\label{eq:E.23}\end{equation}
where $\mathcal{D}_{\boldsymbol{\mu}_1}^*$  and $\mathcal{D}_{\boldsymbol{\mu}_2}^*$ denotes the collection of data points in $\mathcal{D}$ containing feature $\boldsymbol{\mu}_1$ and $\boldsymbol{\mu}_2$, respectively.\par
First, we seek to prove the first bullet point. We utilize the techniques similar to the proofs of Theorem 1 in \citet{chatterji2021finite}, Lemma 3 in \citet{frei2022benign}, Theorem E.1 in \citet{kou2023benign} and Theorem 3.2 in \citet{meng2023benign}. Denote the feature patch in $\mathbf{x}$ as $\boldsymbol{\mu}_{l_x}$ ($l_x \in \{1, 2\}$), we first take a look at the product 
\begin{equation}
\begin{aligned}
 y \cdot f\left(\mathbf{W}^{(t)}, \mathbf{x}\right)&=\frac{1}{m} \sum_{j, r} y j\left[\sigma\left(\left\langle\mathbf{w}_{j, r}^{(t)}, y \boldsymbol{\mu}_{l_x}\right\rangle\right)+\sigma\left(\left\langle\mathbf{w}_{j, r}^{(t)}, \boldsymbol{\xi}\right\rangle\right)\right] \\
& =\frac{1}{m} \sum_r\left[\sigma\left(\left\langle\mathbf{w}_{y, r}^{(t)}, y \boldsymbol{\mu}_{l_x}\right\rangle\right)+\sigma\left(\left\langle\mathbf{w}_{y, r}^{(t)}, \boldsymbol{\xi}\right\rangle\right)\right]-\frac{1}{m} \sum_r\left[\sigma\left(\left\langle\mathbf{w}_{-y, r}^{(t)}, y \boldsymbol{\mu}_{l_x}\right\rangle\right)+\sigma\left(\left\langle\mathbf{w}_{-y, r}^{(t)}, \boldsymbol{\xi}\right\rangle\right)\right]\\
& \leq \frac{1}{m} \left[ \sum_r \sigma\left(\left\langle\mathbf{w}_{y, r}^{(t)}, y \boldsymbol{\mu}_{l_x}\right\rangle\right) - \sum_r \sigma\left(\left\langle\mathbf{w}_{-y, r}^{(t)}, \boldsymbol{\xi}\right\rangle\right) \right].
\end{aligned}
\label{eq:E.24}\end{equation}
Denote $g(\boldsymbol{\xi})$ as $\sum_r \sigma\left(\left\langle\mathbf{w}_{-y, r}^{(t)}, \boldsymbol{\xi}\right\rangle\right)$. Since $\left\langle\mathbf{w}_{-y, r}^{(t)}, \boldsymbol{\xi}\right\rangle \sim \mathcal{N}\left(0,\left\|\mathbf{w}_{-y, r}^{(t)}\right\|_2^2 \sigma_p^2\right)$, we can get
\begin{equation}
\mathbb{E} g(\boldsymbol{\xi})=\sum_{r=1}^m \mathbb{E} \sigma\left(\left\langle\mathbf{w}_{-y, r}^{(t)}, \boldsymbol{\xi}\right\rangle\right)=\sum_{r=1}^m \frac{\left\|\mathbf{w}_{-y, r}^{(t)}\right\|_2 \sigma_p}{\sqrt{2 \pi}}=\frac{\sigma_p}{\sqrt{2 \pi}} \sum_{r=1}^m\left\|\mathbf{w}_{-y, r}^{(t)}\right\|_2 . 
\label{eq:E.25}\end{equation}
Then we can obtain the following test error upper bound on $\mathcal{D}_{\boldsymbol{\mu}_{l_x}}^*$ by adding $\mathbb{E} g(\boldsymbol{\xi})$ and $\dfrac{\sigma_p}{\sqrt{2 \pi}} \sum_{r=1}^m\left\|\mathbf{w}_{-y, r}^{(t)}\right\|_2$ at two sides of the inequality:
\begin{equation}
\begin{aligned}
\mathbb{P}_{(\mathbf{x}, y) \sim \mathcal{D}_{\boldsymbol{\mu}_{l_x}}^*}\left(y f\left(\boldsymbol{W}^{(t)}, \mathbf{x}\right) \leq 0\right) & \leq \mathbb{P}_{(\mathbf{x}, y) \sim \mathcal{D}}\left(\sum_r \sigma\left(\left\langle\mathbf{w}_{-y, r}^{(t)}, \boldsymbol{\xi}\right\rangle\right) \geq \sum_r \sigma\left(\left\langle\mathbf{w}_{y, r}^{(t)}, y \boldsymbol{\mu}_{l_x}\right\rangle\right)\right) \\
& =\mathbb{P}_{(\mathbf{x}, y) \sim \mathcal{D}}\left(g(\boldsymbol{\xi})-\mathbb{E} g(\boldsymbol{\xi}) \geq \sum_r \sigma\left(\left\langle\mathbf{w}_{y, r}^{(t)}, y \boldsymbol{\mu}_{l_x}\right\rangle\right)-\frac{\sigma_p}{\sqrt{2 \pi}} \sum_{r=1}^m\left\|\mathbf{w}_{-y, r}^{(t)}\right\|_2\right).
\end{aligned}
\label{eq:E.26}\end{equation}
By the results in Lemma \ref{app:lem:final coefficient} and Lemma \ref{app:lemma_wt}, we take a look at the comparison of the two terms at the right side of the inequality:
\begin{equation}
\frac{\sum_r \sigma\left(\left\langle\mathbf{w}_{y, r}^{(t)}, y \boldsymbol{\mu}_{l_x}\right\rangle\right)}{\sigma_p \sum_{r=1}^m\left\|\mathbf{w}_{-y, r}^{(t)}\right\|_2} \geq \frac{\Theta\left(\sum_r \gamma_{y, r, l_x}^{(t)}\right)}{\Theta\left(d^{-1 / 2} n_{s}^{-1 / 2}\right) \cdot \sum_{r, i} \bar{\rho}_{-y, r, i}^{(t)}}=\Theta\left(\tau_{l_x}d^{1 / 2} n_{s}^{1 / 2} \operatorname{SNR}_{l_x}^2\right)=\Theta\left(\tau_{l_x}n_{s}^{1 / 2}\|\boldsymbol{\mu}_{l_x}\|_2^2 / (\sigma_p^2 d^{1 / 2} )\right),
\label{eq:E.27}\end{equation}
where $\tau_{l_x}$ denotes the proportion of feature $\boldsymbol{\mu}_{l_x}$ in current training data set (before or after querying). Worth noting that we have assumption in the first bullet that $\forall l \in \{1,2\}, n_{s, l} \geq \dfrac{C_1 \sigma_p^4 d}{\|\boldsymbol{\mu}_l\|_2^4}$, which means $n_{1,l_x}\|\boldsymbol{\mu}_1\|_2^4 \geq 2C_1 \sigma_p^4 d, \forall l_x \in \{1, 2\}$. Since $C_1$ is a sufficiently large constant, it directly follows that
$$
\sum_r \sigma\left(\left\langle\mathbf{w}_{y, r}^{(t)}, y \boldsymbol{\mu}_{l_x}\right\rangle\right)-\frac{\sigma_p}{\sqrt{2 \pi}} \sum_{r=1}^m\left\|\mathbf{w}_{-y, r}^{(t)}\right\|_2>0 .
$$
By Theorem 5.2.2 in \citet{vershynin2018high}, we know that for any $x \geq 0$, the following holds
\begin{equation}
P(g(\boldsymbol{\xi})-\mathbb{E} g(\boldsymbol{\xi}) \geq x) \leq \exp \left(-\frac{c x^2}{\sigma_p^2\|g\|_{\text {Lip }}^2}\right),
\label{eq:E.28}\end{equation}
where $c$ is a constant. To calculate the Lipschitz norm, we have
$$
\begin{aligned}
\left|g(\boldsymbol{\xi})-g\left(\boldsymbol{\xi}^{\prime}\right)\right| & =\left|\sum_{r=1}^m \sigma\left(\left\langle\mathbf{w}_{-y, r}^{(t)}, \boldsymbol{\xi}\right\rangle\right)-\sum_{r=1}^m \sigma\left(\left\langle\mathbf{w}_{-y, r}^{(t)}, \boldsymbol{\xi}^{\prime}\right\rangle\right)\right| \\
& \leq \sum_{r=1}^m\left|\sigma\left(\left\langle\mathbf{w}_{-y, r}^{(t)}, \boldsymbol{\xi}\right\rangle\right)-\sigma\left(\left\langle\mathbf{w}_{-y, r}^{(t)}, \boldsymbol{\xi}^{\prime}\right\rangle\right)\right| \\
& \leq \sum_{r=1}^m\left|\left\langle\mathbf{w}_{-y, r}^{(t)}, \boldsymbol{\xi}-\boldsymbol{\xi}^{\prime}\right\rangle\right| \\
&\leq \sum_{r=1}^m\left\|\mathbf{w}_{-y, r}^{(t)}\right\|_2 \cdot\left\|\boldsymbol{\xi}-\boldsymbol{\xi}^{\prime}\right\|_2,
\end{aligned}
$$
where the first inequality is by triangle inequality; the second inequality is by the property of ReLU; the last inequality is by Cauchy Schwartz Inequality. Therefore, we have
\begin{equation}
\|g\|_{\text {Lip }} \leq \sum_{r=1}^m\left\|\mathbf{w}_{-y, r}^{(t)}\right\|_2.
\label{eq:E.29}\end{equation}
Utilize (\ref{eq:E.28}) and (\ref{eq:E.29}) in (\ref{eq:E.26}), we have
\begin{equation}
\begin{aligned}
\mathbb{P}_{(\mathbf{x}, y) \sim \mathcal{D}_{\boldsymbol{\mu}_{l_x}}^*}\left(y f\left(\boldsymbol{W}^{(t)}, \mathbf{x}\right) \leq 0\right)
& \leq \exp \left[-\frac{c\left(\sum_r \sigma\left(\left\langle\mathbf{w}_{y, r}^{(t)}, y \boldsymbol{\mu}_{l_x}\right\rangle\right)-\left(\dfrac{\sigma_p}{\sqrt{2 \pi}}\right) \sum_{r=1}^m\left\|\mathbf{w}_{-y, r}^{(t)}\right\|_2\right)^2}{\sigma_p^2\left(\sum_{r=1}^m\left\|\mathbf{w}_{-y, r}^{(t)}\right\|_2\right)^2}\right] \\
&= \exp \left[-c\left(\frac{\sum_r \sigma\left(\left\langle\mathbf{w}_{y, r}^{(t)}, y \boldsymbol{\mu}_{l_x}\right\rangle\right)}{\sigma_p \sum_{r=1}^m \| \mathbf{w}_{-y, r}^{(t)}\|_2} - \dfrac{1}{\sqrt{2 \pi}}\right)^2 \right]\\
& \leq \exp (c / 2 \pi) \exp \left(-0.5 c\left(\frac{\sum_r \sigma\left(\left\langle\mathbf{w}_{y, r}^{(t)}, y \boldsymbol{\mu}_{l_x}\right\rangle\right)}{\sigma_p \sum_{r=1}^m\left\|\mathbf{w}_{-y, r}^{(t)}\right\|_2}\right)^2\right),
\end{aligned}
\label{eq:E.30}\end{equation}
where the third inequality is by $(s-t)^2 \geq s^2 / 2-t^2, \forall s, t \geq 0$.
And then by (\ref{eq:E.27}) and (\ref{eq:E.30}), we can have
\begin{equation}
\begin{aligned}
\mathbb{P}_{(\mathbf{x}, y) \sim \mathcal{D}_{\boldsymbol{\mu}_{l_x}}^*}\left(y f\left(\boldsymbol{W}^{(t)}, \mathbf{x}\right) \leq 0\right) & \leq \exp (c / 2 \pi) \exp \left(-0.5 c\left(\frac{\sum_r \sigma\left(\left\langle\mathbf{w}_{y, r}^{(t)}, y \boldsymbol{\mu}_{l_x}\right\rangle\right)}{\sigma_p \sum_{r=1}^m\left\|\mathbf{w}_{-y, r}^{(t)}\right\|_2}\right)^2\right) \\
& =\exp \left(\frac{c}{2 \pi}-\frac{\tau_{l_x}n_{s, l_x}\|\boldsymbol{\mu}_{l_x}\|_2^4}{C \sigma_p^4 d}\right) \\
& =\exp \left(\frac{c}{2 \pi}-\frac{n_{s, l_x}\|\boldsymbol{\mu}_{l_x}\|_2^4}{C_{l_x} \sigma_p^4 d}\right) \\
& \leq \exp \left(-\frac{n_{s, l_x}\|\boldsymbol{\mu}_{l_x}\|_2^4}{2 C_{l_x} \sigma_p^4 d}\right)
\end{aligned}
\label{eq:E.31}\end{equation}
where $C_{l_x}=C/\tau_{lx}=O(1)$; the last inequality holds if we choose $C_1 \geq c C_{l_x} / \pi$, for any $l_x \in \{1, 2\}$. If we choose $C_3$ as $2 C_{l_1}$ and $C_4$ as $2 C_{l_2}$, by (\ref{eq:E.23}) and (\ref{eq:E.31}) we have
\[
L_{\mathcal{D}^*}^{0-1}\left(\mathbf{W}^{(t)}\right) \leq (1-p^*) \cdot \exp \left( \dfrac{-n_{s, 1}\|\boldsymbol{\mu}_1\|_2^4}{C_3 \sigma_p^4 d} \right) + p^* \cdot \exp \left( \dfrac{-n_{s, 2}\|\boldsymbol{\mu}_2 \|_2^4}{C_4 \sigma_p^4 d} \right).
\]
Next, we serve to prove the second bullet point. We utilize the pigeonhole principle technique in \citet{kou2023benign, meng2023benign}, which is based on the following two lemmas.
\begin{lemma}
For $t \in\left[T_1, T^*\right]$, denote $g(\boldsymbol{\xi})=\sum_{j, r} \sigma\left(\left\langle\mathbf{w}_{j, r}^{(t)}, \boldsymbol{\xi}\right\rangle\right)$. There exists a fixed vector $\mathbf{v}_l$ with $\|\mathbf{v}_l\|_2 \leq 0.02 \sigma_p$ and constant $C_6$ such that
$$
\sum_{j^{\prime} \in\{ \pm 1\}}\left[g\left(j^{\prime} \boldsymbol{\xi}+\mathbf{v}_l\right)-g\left(j^{\prime} \boldsymbol{\xi}\right)\right] \geq 4 C_6 \max _{j, l}\left\{\sum_r \gamma_{j, r, l}^{(t)}\right\},
$$
for all $\boldsymbol{\xi} \in \mathbb{R}^d$.
\label{app:lem5.8:in kou}\end{lemma}
\textit{Proof of Lemma \ref{app:lem5.8:in kou}.} See Lemma 5.8 in \citet{kou2023benign} or Theorem 3.2 in \citet{meng2023benign} for a proof, where we utilize a large enough $C_2$ in the condition given in the second bullet point ($ n_{s, {l^{\prime}}} \leq \dfrac{C_2 \sigma_p^4 d}{\|\boldsymbol{\mu}_{l^{\prime}}\|_2^4 }$) to control the norm of $\mathbf{v}_l$.
\begin{lemma}
(Proposition 2.1 in \citet{devroye2018total}). The TV distance between $\mathcal{N}\left(0, \sigma_p^2 \mathbf{I}_d\right)$ and $\mathcal{N}\left(\mathbf{v}_l, \sigma_p^2 \mathbf{I}_d\right)$ is smaller than $\|\mathbf{v}_l\|_2 / 2 \sigma_p$.
\label{app:leme.2:in devroye}\end{lemma}
\textit{Proof of Lemma \ref{app:leme.2:in devroye}.} See Proposition 2.1 in \citet{devroye2018total} for a proof. \par
Now we take a look at $L_{\mathcal{D}^*}^{0-1}\left(\mathbf{W}^{(t)}\right)$, by (\ref{eq:E.23}) we have:
\begin{equation}
\begin{aligned}
L_{\mathcal{D}^*}^{0-1}\left(\mathbf{W}^{(t)}\right) &=\tau^*_1 \cdot \mathbb{P}_{(\mathbf{x}, y) \sim \mathcal{D}_{\boldsymbol{\mu}_1}^*}\left[y \cdot f(\mathbf{W}, \mathbf{x}) < 0\right] + \tau^*_2) \cdot \mathbb{P}_{(\mathbf{x}, y) \sim \mathcal{D}_{\boldsymbol{\mu}_{2}}} \left[y \cdot f(\mathbf{W}, \mathbf{x}) < 0 \right] \\
& \geq \tau^*_{l^\prime} \cdot \mathbb{P}_{(\mathbf{x}, y) \sim \mathcal{D}_{\boldsymbol{\mu}_{l^{\prime}}}^*}[y \cdot f(\mathbf{W}, \mathbf{x}) < 0] \\
& = \tau^*_{l^\prime} \cdot \mathbb{P}_{(\mathbf{x}, y) \sim \mathcal{D}_{\boldsymbol{\mu}_{l^{\prime}}}^*}\Big(\sum_r \sigma\left(\left\langle\mathbf{w}_{-y, r}^{(t)}, \boldsymbol{\xi}\right\rangle\right)-\sum_r \sigma\left(\left\langle\mathbf{w}_{y, r}^{(t)}, \boldsymbol{\xi}\right\rangle\right)\\
&\phantom{\tau^*_{l^\prime} \cdot \mathbb{P}_{(\mathbf{x}, y) \sim \mathcal{D}_{\boldsymbol{\mu}_{l^{\prime}}}^*}\Big(}\geq \sum_r \sigma\left(\left\langle\mathbf{w}_{y, r}^{(t)}, y \boldsymbol{\mu}_{l^{\prime}}\right\rangle\right)-\sum_r \sigma\left(\left\langle\mathbf{w}_{-y, r}^{(t)}, y \boldsymbol{\mu}_{l^{\prime}}\right\rangle\right)\Big) \\
& \geq 0.5 \tau^*_{l^\prime} \cdot \mathbb{P}_{(\mathbf{x}, y) \sim \mathcal{D}_{\boldsymbol{\mu}_{l^{\prime}}}^*}\left(\left|\sum_r \sigma\left(\left\langle\mathbf{w}_{1, r}^{(t)}, \boldsymbol{\xi}\right\rangle\right)-\sum_r \sigma\left(\left\langle\mathbf{w}_{-1, r}^{(t)}, \boldsymbol{\xi}\right\rangle\right)\right| \geq C_6 \max \left\{\sum_r \gamma_{1, r, {l^{\prime}}}^{(t)}, \sum_r \gamma_{-1, r, {l^{\prime}}}^{(t)}\right\}\right) \\
& = 0.5 \tau^*_{l^\prime} \cdot P(\Omega_{\boldsymbol{\xi}}),
\end{aligned}
\label{eq:E.32}\end{equation}
where $\Omega_{\boldsymbol{\xi}}:=\left\{\boldsymbol{\xi}|| g(\boldsymbol{\xi}) \mid \geq C_6 \max \left\{\sum_r \gamma_{1, r, {l^{\prime}}}^{(t)}, \sum_r \gamma_{-1, r, {l^{\prime}}}^{(t)}\right\}\right\}$. The last inequality holds since we can always have a corresponding $y$ to make a wrong prediction if given $\boldsymbol{\xi}$, the $\left|\sum_r \sigma\left(\left\langle\mathbf{w}_{1, r}^{(t)}, \boldsymbol{\xi}\right\rangle\right)-\sum_r \sigma\left(\left\langle\mathbf{w}_{-1, r}^{(t)}, \boldsymbol{\xi}\right\rangle\right)\right|$ is large enough.

Next, we seek a lower bound of $P(\Omega_{\boldsymbol{\xi}})$. By Lemma \ref{app:lem5.8:in kou}, we have that $\sum_j[g(j \boldsymbol{\xi}+\mathbf{v}_l)-g(j \boldsymbol{\xi})] \geq$ $4 C_6 \max _{j, l}\left\{\sum_r \gamma_{j, r, l}^{(t)}\right\}$. Then by pigeon's hole principle, there must exist one of the $\boldsymbol{\xi}, \boldsymbol{\xi}+\mathbf{v}_l$, $-\boldsymbol{\xi},-\boldsymbol{\xi}+\mathbf{v}_l$ belongs $\Omega_{\boldsymbol{\xi}}$. So we have proved that $\Omega_{\boldsymbol{\xi}} \cup-\Omega_{\boldsymbol{\xi}} \cup \Omega_{\boldsymbol{\xi}}-\{\mathbf{v}_l\} \cup-\Omega_{\boldsymbol{\xi}}-\{\mathbf{v}_l\}=\mathbb{R}^d$. Therefore at least one of $P(\Omega_{\boldsymbol{\xi}}), P(-\Omega_{\boldsymbol{\xi}}), P(\Omega_{\boldsymbol{\xi}}-\{\mathbf{v}_l\}), P(\Omega_{\boldsymbol{\xi}}-\{\mathbf{v}_l\}), P(-\Omega_{\boldsymbol{\xi}}-\{\mathbf{v}_l\})$ is greater than 0.25. By  the definition of TV distance, we have:
$$
\begin{aligned}
|P(\Omega_{\boldsymbol{\xi}})-P(\Omega_{\boldsymbol{\xi}}-\mathbf{v}_l)| & =\left|\mathbb{P}_{\boldsymbol{\xi} \sim \mathcal{N}\left(0, \sigma_p^2 \mathbf{I}_d\right)}(\boldsymbol{\xi} \in \Omega_{\boldsymbol{\xi}})-\mathbb{P}_{\boldsymbol{\xi} \sim \mathcal{N}\left(\mathbf{v}_l, \sigma_p^2 \mathbf{I}_d\right)}(\boldsymbol{\xi} \in \Omega_{\boldsymbol{\xi}})\right| \\
& \leq \operatorname{TV}\left(\mathcal{N}\left(0, \sigma_p^2 \mathbf{I}_d\right), \mathcal{N}\left(\mathbf{v}_l, \sigma_p^2 \mathbf{I}_d\right)\right) \\
& \leq \frac{\|\mathbf{v}_l\|_2}{2 \sigma_p} \\
& \leq 0.02.
\end{aligned}
$$
Also, notice that $P(-\Omega_{\boldsymbol{\xi}})=P(\Omega_{\boldsymbol{\xi}})$, we have $4P(\Omega_{\boldsymbol{\xi}}) \geq 1 - 2 \cdot 0.02$. Thus $ L_{\mathcal{D}^*}^{0-1}\left(\mathbf{W}^{(t)}\right) \geq 0.5 \tau^*_{l^\prime} \cdot 0.24 = 0.12 \cdot \tau^*_{l^\prime}$. The proofs complete.
\par
Based on Lemma \ref{app:lemma for thm}, our focus is to verify whether the NAL algorithms satisfy the condition stated in the first bullet point. On the other hand, it is highly likely that Random Sampling fulfills the condition stated in the second bullet point. The following proposition validates this intuition.
\begin{proposition}
    When Lemma \ref{app:Lem:order_pool_app} holds, and the sampling size of algorithm satisfies $ \dfrac{C_1 \sigma_p^4 d}{\|\boldsymbol{\mu}_2 \|_2^4} - \dfrac{p n_0}{2} \leq n^*=\Theta(\widetilde{n} - n_0) \leq \widetilde{n} - n_0$, we have the following:
    \begin{itemize}
        \item The number of data with strong feature patch $n_{s,1}$ satisfies $ n_{s, 1} \geq \dfrac{C_1 \sigma_p^4 d}{\|\boldsymbol{\mu}_1\|_2^4}, \forall s \in \{0, 1\}$.
        \item The number of data with weak feature patch $n_{s,2}$ before querying and after \textbf{Random Sampling} satisfies $n_{s, 2} \leq \dfrac{C_2 \sigma_p^4 d}{\|\boldsymbol{\mu}_2\|_2^4 }, \forall s \in \{0, 1\}$.
       \item  The total number of data with weak feature patch $n_{1,2}$ after \textbf{Uncertainty Sampling} and \textbf{Diversity Sampling} satisfies $ n_{1, 2} \geq \dfrac{C_1 \sigma_p^4 d}{\|\boldsymbol{\mu}_2\|_2^4} $ .
    \end{itemize}
For the sake of coherence, here $C_1$ and $C_2$ are some constants shared with Theorem \ref{theory main} and Lemma \ref{main:lem for thm}.
\label{app:Prop:num_n12}\end{proposition}
\textit{Proof of Proposition \ref{app:Prop:num_n12}.} 
By conditions in Definition \ref{Def3.1}, we have $ (1-\dfrac{3}{2}p)n_{0}\geq \dfrac{C_1 \sigma_p^4 d}{\|\boldsymbol{\mu}_1\|_2^4}$ for a large constant $C_1$. Then by plugging the results of $n_{p}$ for $n_0$ in Lemma \ref{lem:rho n}, as well as the definition of $n_{s, l}$, we have
\[
n_{1, 1} \geq n_{0, 1} \geq (1-\dfrac{3}{2}p)n_{0} \geq \dfrac{C_1 \sigma_p^4 d}{\|\boldsymbol{\mu}_1\|_2^4}.
\]
For the second bullet, by Lemma \ref{lem:rho n}, Lemma \ref{app:Lem:order_pool_app} and conditions $n^* \geq \dfrac{C_1 \sigma_p^4 d}{\|\boldsymbol{\mu}_2 \|_2^4} - \dfrac{p n_0}{2} $, we have:
\[
n_{1,2} \geq \dfrac{p n_0}{2} + n^* \geq \dfrac{C_1 \sigma_p^4 d}{\|\boldsymbol{\mu}_2\|_2^4} 
\]
Besides, by Lemma \ref{lem:rho n} and the condition $\widetilde{n} \leq \dfrac{2C_2 \sigma_p^4 d}{3p \|\boldsymbol{\mu}_2 \|_2^4} $, the third bullet holds straightforwardly.\par

By the result of Lemma \ref{app:lemma for thm} and Proposition \ref{app:Prop:num_n12}, the results of Proposition \ref{Prop4.3} and Theorem \ref{theory main} holds directly.

\begin{lemma} (Restatement of Corollary \ref{coro:label complexity})
    Under the same conditions as stated in Theorem \ref{theory main}, with a probability of at least $1-\Theta(\delta + \delta^{\prime})$, we observe distinct label complexities for traditional 2-layer ReLU CNN and NAL algorithms in achieving Bayes-optimal generalization ability:
\begin{itemize}
\item For a fully trained neural model, the label complexity $n_{CNN}$ requires $\Omega(p^{-1} \sigma_p^2 d \|\boldsymbol{\mu}_2\|_2^{-4})$.
\item For two NAL algorithms, the maximum label complexity $\widetilde{n}$ only requires $\Omega(\sigma_p^2 d \|\boldsymbol{\mu}_2\|_2^{-4})$.
\end{itemize}
\label{app:lemma:label complexity}\end{lemma}
\textit{Proof of Lemma \ref{app:lemma:label complexity}.} According to Lemma \ref{app:lemma for thm}, to adequately learn the signal $\boldsymbol{\mu}_l$ for any $l \in \{1, 2\}$, one needs at least $\hat{C}1 \sigma_p^4d \| \boldsymbol{\mu}_l \|_2^{-4}$. Since the occurrence probability of $\boldsymbol{\mu}_2$ is low ($p$), Random Sampling without any strategy requires a label complexity of at least $\Omega(p^{-1} \sigma_p^2 d \|\boldsymbol{\mu}_2\|_2^{-4})$ to capture sufficient instances of $\boldsymbol{\mu}_2$ from the training distribution. On the other hand, by leveraging the insights from Lemma \ref{app:Lem:order_pool_app} and Lemma \ref{app:lem:sample stage: learning on sample}, both Uncertainty Sampling and Diversity Sampling can effectively query yet-to-be-learned \hyperlink{msg:perplexing}{\textbf{perplexing samples}}, which are typically samples associated with $\boldsymbol{\mu}_2$ by Lemma \ref{app:lem:final coefficient}. Hence, both querying algorithms only require a label complexity of $\Omega(\sigma_p^2 d \|\boldsymbol{\mu}_2\|_2^{-4})$.

\section{Proofs of Main Results: XOR data version}\label{App:proofs of XOR}
In this section, we first introduce some notations. We denote $n$ as the number of training data in the current labeled training set, which is initially $n_0$ and becomes $n_1$ after sampling (querying). We define $\mathbf{u}_l = \mathbf{a}_l + \mathbf{b}_l$ and $\mathbf{v}_l = \mathbf{a}_l - \mathbf{b}_l$. The proportion of easy-to-learn data $\boldsymbol{\mu}_1 = \pm (\mathbf{a}_1 \pm \mathbf{b}_1)$ in the current labeled set is denoted as $\tau_1$, while $\tau_2$ represents the proportion of hard-to-learn data $\boldsymbol{\mu}_2 = \pm (\mathbf{a}_2 \pm \mathbf{b}_2)$. In a manner similar to the proofs provided in Appendix \ref{app:proofs of main results}, in this section we utilize the techniques employed in \citet{kou2023benign, meng2023benign} to obtain results that are not directly related to our main contribution. For the sake of brevity, we omit most of the proof details of those outcomes, as our setting aligns with the one considered in \cite{meng2023benign}, despite the fact that we examine multiple task-oriented features. Instead, our focus is on providing comprehensive proofs of our primary contributions. \par

First, we claim that all preliminary Lemmas in Appendix \ref{app:preliminary lemmas} hold with high probability. It is evident from Definition \ref{equ:cnndefinition} that $F_{+1}\left(\mathbf{W}_{+1}, \mathbf{x}\right)$ always contributes to the prediction of class $+1$, while $F_{-1}\left(\mathbf{W}_{-1}, \mathbf{x}\right)$ always contributes to the prediction of class $-1$. Therefore, the jobs of $\mathbf{w}_{+1, r}$ and $\mathbf{w}_{-1, r}$ are learning $\pm \mathbf{u}$ and $\pm \mathbf{v}$ respectively. Then, similar to (\ref{app:Def:signol-noise-decomposition}), we take a look at the coefficient updates with \textit{signal-noise decomposition} techniques, specified as the following.
\begin{equation}
\mathbf{w}_{j, r}^{(t)}=\mathbf{w}_{j, r}^{(0)}+ \sum_{l=1}^2 \gamma_{j, r, \mathbf{u}_l}^{(t)} \cdot \dfrac{j \cdot \mathbf{u}_l}{\|\mathbf{u}_l\|_2^{2}} -   \sum_{l=1}^2 \gamma_{j, r, \mathbf{v}_l}^{(t)} \cdot \dfrac{j \cdot\mathbf{v}_l}{\|\mathbf{v}_l\|_2^{2}} +\sum_{i=1}^{n} \bar{\rho}_{j, r, i}^{(t)} \cdot \dfrac{\boldsymbol{\xi}_i}{\|\boldsymbol{\xi}_i\|_2^{2}} +\sum_{i=1}^{n} \underline{\rho}_{j, r, i}^{(t)} \cdot \dfrac{\boldsymbol{\xi}_i}{\|\boldsymbol{\xi}_i\|_2^{2}},
\end{equation}
where we denote $\bar{\rho}_{j, r, i}^{(t)}$ as $\rho_{j, r, i}^{(t)} \mathbb{1}\left(\rho_{j, r, i}^{(t)} \geq 0\right)$, $ \underline{\rho}_{j, r, i}^{(t)}$ as $\rho_{j, r, i}^{(t)} \mathbb{1}\left(\rho_{j, r, i}^{(t)} \leq 0\right)$. Here $\gamma_{j, r, \mathbf{u}_l}^{(t)}$ are mainly contributed by $F_{+1}\left(\mathbf{W}_{+1}, \mathbf{x}\right)$, and $\gamma_{\pm 1, r, \mathbf{u}_l}^{(t)} \approx \left\langle\mathbf{w}_{j, r}^{(t)}, \pm \mathbf{u} \right\rangle$. Similarly $\gamma_{j, r, \mathbf{v}_l}^{(t)}$ are mainly contributed by $F_{-1}\left(\mathbf{W}_{-1}, \mathbf{x}\right)$, and $\gamma_{\pm 1, r, \mathbf{v}_l}^{(t)} \approx \left\langle\mathbf{w}_{j, r}^{(t)}, \pm \mathbf{v} \right\rangle$. Worth noting that $j  \in \{ \pm 1\}$ here denote the signal of $\mathbf{u}_l$ and $\mathbf{v}_l$, but not the signal of $F_{j^\prime}\left(\mathbf{W}_{j^\prime}, \mathbf{x}\right), j^\prime \in \{ \pm 1\}$.\par 
Specifically, the update rule can be written as:
\begin{equation}
\begin{aligned}
\mathbf{w}_{j, r}^{(t+1)}=\mathbf{w}_{j, r}^{(t)} & -\frac{\eta j}{n m} \sum_{i \in S_{+\mathbf{u}_l,+1} \cup S_{-\mathbf{u}_l,-1}} \ell_i^{(t)} \mathbb{1}\left\{\left\langle\mathbf{w}_{j, r}^{(t)}, \boldsymbol{\mu}_i\right\rangle>0\right\} \mathbf{u}_l+\frac{\eta j}{n m} \sum_{i \in S_{-\mathbf{u}_l,+1} \cup S_{+\mathbf{u}_l,-1}} {\ell_i^{\prime}}^{(t)} \mathbb{1}\left\{\left\langle\mathbf{w}_{j, r}^{(t)}, \boldsymbol{\mu}_i\right\rangle>0\right\} \mathbf{u}_l \\
& +\frac{\eta j}{n m} \sum_{i \in S_{+\mathbf{v}_l,-1} \cup S_{-\mathbf{v}_l,+1}} {\ell_i^{\prime}}^{(t)} \mathbb{1}\left\{\left\langle\mathbf{w}_{j, r}^{(t)}, \boldsymbol{\mu}_i\right\rangle>0\right\} \mathbf{v}_l-\frac{\eta j}{n m} \sum_{i \in S_{-\mathbf{v}_l,-1} \cup S_{+\mathbf{v}_l,+1}} \ell_i^{(t)} \mathbb{1}\left\{\left\langle\mathbf{w}_{j, r}^{(t)}, \boldsymbol{\mu}_i\right\rangle>0\right\} \mathbf{v}_l \\
& -\frac{\eta}{n m} \sum_{i=1}^n {\ell_i^{\prime}}^{(t)}\left(j y_i\right) \mathbb{1}\left\{\left\langle\mathbf{w}_{j, r}^{(t)}, \boldsymbol{\xi}_i\right\rangle>0\right\} \boldsymbol{\xi}_i,
\end{aligned} 
\end{equation}
where $S_{\boldsymbol{\mu},j}=\{i\in [n], \boldsymbol{\mu}_i = \boldsymbol{\mu}, y_i =j \}$. Here $\boldsymbol{\mu} \in \{ \pm \mathbf{u}_1, \pm \mathbf{u}_2, \pm \mathbf{v}_1, \pm \mathbf{v}_2 \}, j \in \{ \pm 1 \}$, and we let $\boldsymbol{\mu}_i$ represents the feature in $\mathbf{x}_i$ and $\boldsymbol{\xi}_i$ represents the noise in $\mathbf{x}_i$. \par
The following lemma shows that a specific discrete process can be bounded by its continuous counterpart, which would be useful in bounding the coefficient $\sum_{i=1}^{n} \bar{\rho}_{j, r, i}^{(t)}$ and the derivative of loss function. 
\begin{lemma} (Lemma C.1 in \citet{meng2023benign})
Suppose that a sequence $a_t, t \geq 0$ follows the iterative formula
$$
a_{t+1}=a_t+\frac{c}{1+b e^{a_t}},
$$
for some $1 \geq c \geq 0$ and $b \geq 0$. Then it holds that
$$
x_t \leq a_t \leq \frac{c}{1+b e^{a_0}}+x_t
$$
for all $t \geq 0$. Here, $x_t$ is the unique solution of
$$
x_t+b e^{x_t}=c t+a_0+b e^{a_0} .
$$
\end{lemma}
\subsection{Coefficient Ratio and Scale Analysis: XOR data version}\label{appXOR:Coefficient Ratio and Scale analysis}
Similar to the processes in Appendix \ref{app:proofs of main results}, we assume the results in the previous section hold with high probability. Meanwhile, let $T^*=$ $\eta^{-1}$ poly $\left(\varepsilon^{-1}, d, n, m\right)$ be the maximum admissible iteration. We adopt similar notations as those in (\ref{eq:alpha,belta,snr}):
\begin{equation}
\begin{aligned}
& \alpha\mathrel{\mathop:}=4 \log \left(T^*\right), \\
& \beta\mathrel{\mathop:}=2 \max _{l, i, j, r}\left\{\left|\left\langle\mathbf{w}_{j, r}^{(0)}, \boldsymbol{\mu}_l\right\rangle\right|,\left|\left\langle\mathbf{w}_{j, r}^{(0)}, \boldsymbol{\xi}_i\right\rangle\right|\right\}, \\
& \operatorname{SNR}_l\mathrel{\mathop:}=\dfrac{\|\boldsymbol{\mu}_l\|_2}{ \sigma_p \sqrt{d}}, \\
& \kappa=56 \sqrt{\frac{\log \left(6 n^2 / \delta\right)}{d}} n \log \left(T^*\right)+10 \sqrt{\log (16 m n / \delta)} \cdot \sigma_0 \sigma_p \sqrt{d}+ \sum_{l=1}^264 \tau_l n \cdot \operatorname{SNR}_l^2 \log \left(T^*\right) .
\end{aligned}
\label{eqXOR:alpha,belta,snr}
\end{equation}
Then, similar to our results in Proposition \ref{prop:E.8}, we here also have the coefficient scale as below.

\begin{proposition} If Condition \ref{condition_XOR} holds, then for any $0 \leq t \leq T^*, j \in\{ \pm 1\}, r \in[m]$ and $i \in[n]$, it holds that
$$
\begin{aligned}
& 0 \leq \lvert \left\langle\mathbf{w}_{+1, r}^{(t)}, \mathbf{u}_l\right\rangle \rvert, \lvert \left\langle\mathbf{w}_{-1, r}^{(t)}, \mathbf{v}_l\right\rangle \rvert = \Theta(\gamma_{j, r, \mathbf{u}_l}^{(t)}), \Theta(\gamma_{j, r, \mathbf{v}_l}^{(t)}) \leq 32 \tau_l n \cdot \operatorname{SNR}_l^2 \alpha, \\
& 0 \leq \bar{\rho}_{j, r, i}^{(t)} \leq 4 \alpha, \quad 0 \geq \underline{\rho}_{j, r, i}^{(t)} \geq-\beta-32 \sqrt{\frac{\log \left(6 n^2 / \delta\right)}{d}} n \alpha,\\
& -\frac{\kappa}{2}+\frac{1}{m} \sum_{r=1}^m \bar{\rho}_{y_i, r, i}^{(t)}  \leq y_i f\left(\mathbf{W}^{(t)}, \mathbf{x}_i\right) \leq \frac{\kappa}{2}+\frac{1}{m} \sum_{r=1}^m \bar{\rho}_{y_i, r, i}^{(t)} .
\end{aligned}
$$

Moreover, define $\bar{c}=\dfrac{2 \eta \sigma_p^2 d}{n m}, \underline{c}=\dfrac{\eta \sigma_p^2 d}{3 n m}, \bar{b}=e^{-\kappa}$ and $\underline{b}=e^\kappa$, and let $\bar{x}_t, \underline{x}_t$ be the unique solution of
$$
\begin{aligned}
& \bar{x}_t+\bar{b} e^{\bar{x}_t}=\bar{c} t+\bar{b}, \\
& \underline{x}_t+\underline{b} e^{\underline{x}_t}=\underline{c} t+\underline{b},
\end{aligned}
$$
it holds that
\begin{equation}
\underline{x}_t \leq \frac{1}{m} \sum_{r=1}^m \bar{\rho}_{y_i, r, i}^{(t)} \leq \bar{x}_t+\bar{c} /(1+\bar{b}), \quad  \log \left(\frac{\eta \sigma_p^2 d}{8 n m} t+2/3\right) \leq \bar{x}_t, \underline{x}_t \leq \log \left(\frac{2 \eta \sigma_p^2 d}{n m} t+1\right)
\label{eqXOR:scale of barunderline_x}\end{equation}
for all $r \in[m]$ and $i \in[n]$.
\label{prop:F.2}\end{proposition}
\textit{Proof of Proposition \ref{prop:F.2}.} Please refer to Proposition C.2, Proposition C.8 and Lemma C.9 in \citet{meng2023benign} for a proof. Regardless of the variations in data settings, it is feasible to obtain the result through inductive techniques \cite{cao2022benign, frei2022benign, kou2023benign, lu2023benign}. \par
Building upon Proposition \ref{prop:F.2}, we can further analyze the convergence of the training dynamics by examining the extent of feature learning and noise memorization in the subsequent section.

\subsection{Feature Learning and Noise Memorization Analysis: XOR data version}\label{appXOR:feature learning and Noise Memorization Analysis}
Similar to Lemma \ref{lem:E.16} and Lemma \ref{app:lemma_wt} for linearly separable data, we can also determine the scale of coefficients and inner products as follows.

\begin{proposition} Under Condition \ref{condition_XOR}, the following points hold ($n>n_0$) for $\forall l \in \{1, 2\}$:
\begin{enumerate}
    \item For any $r \in[m]$, $\left\langle\mathbf{w}_{+1, r}^{(t)}, \mathbf{u}_l\right\rangle$ $(\text{or }\left\langle\mathbf{w}_{-1, r}^{(t)}, \mathbf{v}_l\right\rangle)$ increases if $\left\langle\mathbf{w}_{+1, r}^{(0)}, \mathbf{u}_l\right\rangle>0 (\text{ or }\left\langle\mathbf{w}_{-1, r}^{(t)}, \mathbf{v}_l\right\rangle < 0)$, $\left\langle\mathbf{w}_{+1, r}^{(t)}, \mathbf{u}_l\right\rangle$ $(\text{ or }\left\langle\mathbf{w}_{-1, r}^{(t)}, \mathbf{v}_l\right\rangle)$ decreases if $\left\langle\mathbf{w}_{+1, r}^{(0)}, \mathbf{u}_l\right\rangle<0$ $(\text{ or }\left\langle\mathbf{w}_{-1, r}^{(t)}, \mathbf{v}_l\right\rangle)>0$. Moreover, it holds that
    \begin{equation}
    \begin{aligned}
    &\gamma_{j, r, \mathbf{u}_l}^{(t)},\gamma_{j, r, \mathbf{v}_l}^{(t)} = \Theta( \dfrac{\tau_l n\|\boldsymbol{\mu}_l\|_2^2}{\sigma_p^2 d} \cdot \log \left(\dfrac{\eta \sigma_p^2 d t}{n m}\right) ),  \lvert \left\langle\mathbf{w}_{+1, r}^{(t)}, \mathbf{u}_l\right\rangle \rvert, \lvert \left\langle\mathbf{w}_{-1, r}^{(t)}= \Theta( \dfrac{\tau_l n\|\boldsymbol{\mu}_l\|_2^2}{\sigma_p^2 d} \cdot \log \left(\dfrac{\eta \sigma_p^2 d t}{n m}\right) ), \mathbf{v}_l\right\rangle \rvert,\\
    & \lvert \left\langle\mathbf{w}_{-1, r}^{(t)}, \mathbf{u}_l\right\rangle \rvert \leq \lvert \left\langle\mathbf{w}_{-1, r}^{(0)}, \mathbf{u}_l\right\rangle \rvert + \eta \| \boldsymbol{\mu}_l \|_2^2/m, \lvert \left\langle\mathbf{w}_{+1, r}^{(t)}, \mathbf{v}_l\right\rangle \rvert \leq \lvert \left\langle\mathbf{w}_{+1, r}^{(0)}, \mathbf{v}_l\right\rangle \rvert + \eta \| \boldsymbol{\mu}_l \|_2^2/m.
    \end{aligned}
    \label{eq:xor_ub-of-mu-same}\end{equation}
     \item Let $\underline{x}_t$ defined in Proposition \ref{prop:F.2}, we have
    \begin{equation}
    \Omega(n) \leq \frac{n}{5} \cdot\left(\bar{x}_{t-1}-\bar{x}_1\right) \leq \sum_{i=1}^n \bar{\rho}_{j, r, i}^{(t)} \leq 3 n \underline{x}_t \leq 3n \cdot \log \left(\frac{2 \eta \sigma_p^2 d}{n m} t+1\right) = \Theta (n \cdot \log \left(\dfrac{\eta \sigma_p^2 d t}{n m}\right) ),
    \label{eqXOR:f.6}\end{equation}
    for all $t \in\left[T^*\right]$ and $r \in[m]$. Moreover, we have: 
    $$
     \sum_{i=1}^n \bar{\rho}_{j, r, i}^{(t)}/ \gamma_{\boldsymbol{\mu}_l, j^{\prime}, r^{\prime}, l}^{(t)} = \Theta\left(\tau_l^{-1}\cdot \operatorname{SNR}_l^{-2} \right)  = \sum_{i=1}^n \bar{\rho}_{j, r, i}^{(t)}/ \lvert 
    \left\langle\mathbf{w}_{\pm 1, r^{\prime}}^{(t)}, \boldsymbol{\mu}_l \right\rangle \rvert,
    $$
    for all $j, j^{\prime} \in\{ \pm 1\}, r, r^{\prime} \in[m]$. 
    \item For $t=\Omega\left(n m /\left(\eta \sigma_p^2 d\right)\right)$, the bound for $\left\|\mathbf{w}_{j, r}^{(t)}\right\|_2$ is given by:
    \begin{equation}
    \left\|\mathbf{w}_{j, r}^{(t)}\right\|_2 =  \Theta\left(\sigma_p^{-1} d^{-1 / 2} n^{1 / 2} \cdot \log \left(\dfrac{\eta \sigma_p^2 d t}{n m}\right) \right).
    \end{equation}
\end{enumerate}
\label{prop:F.3}\end{proposition}
\textit{Proof of Proposition \ref{prop:F.2}.} The basic techniques are the same as Lemma \ref{lem:E.16} and Lemma \ref{app:lemma_wt} despite variation in data settings. Please refer to Proposition 4.2, Proposition D.3-5 in \citet{meng2023benign} for a comprehensive proof.

\subsection{Order-dependent Sampling (Querying) Analysis: XOR data version}\label{appXOR:Order-dependent Sampling(Querying Analysis)}
Based on the scale of $\mathbf{w}_{j, r}^{(t)}$ and the inner product between it and features, we can now characterize the querying situation of the two NAL methods based on the query criteria. Similar to the order-dependent analysis techniques utilized in Appendix \ref{app:Order-dependent Sampling(Querying Analysis)}, we employ a full-order-based technique to tackle the problem of $\Theta (\lvert \mathcal{P} \rvert^2)$ comparisons in $\mathcal{P}$. The concepts of Uncertainty Order and Diversity Order are introduced in Appendix \ref{App:Order def}. We then proceed to examine the order of the samples in $\mathcal{P}$ in the following proposition.
\begin{proposition}
     Under the same conditions of Proposition \ref{app:propXOR:sample stage:learning on sample}, there exist $t=\widetilde{O}\left(\eta^{-1} \varepsilon^{-1} m n d^{-1} \sigma_p^{-2}\right)$ that for $\forall \mathbf{x}, \mathbf{x}^\prime \in \mathcal{P} \subsetneq \mathcal{D}$ where $\mathbf{x}$ contains hard-to-learn feature patch while $\mathbf{x}^\prime$ contains easy-to-learn feature patch, with probability at least 1-$\delta^\prime$, we have $\mathbf{x}^\prime \preceq^{(t)} \mathbf{x}$. 
\label{prop_App_XOR:order_newdata}
\end{proposition}
\textit{Proof of Proposition \ref{prop_App_XOR:order_newdata}.}
Firstly, suggest $ \mathbf{x}=[y \cdot  \boldsymbol{\mu}_2, \mathbf{z}_2], \mathbf{x}^{\prime}=[y^{\prime} \cdot \boldsymbol{\mu}_1,\mathbf{z}_1]$, where $\boldsymbol{\mu}_1 \in \{ \mathbf{u}_1, \mathbf{v}_1\}, \boldsymbol{\mu}_2 \in \{ \mathbf{u}_2, \mathbf{v}_2\}, y, y^\prime \in [\pm 1], \mathbf{z}_1, \mathbf{z}_2 \sim N(\mathbf{0}, \sigma_p^2 \cdot \mathbf{I})$:
\[
\begin{aligned}
&
f\!\left( \mathbf{W}^{(t)}, \mathbf{x} \right)\!=\sum_{j, r} \frac{j}{m}\left[\sigma\!\left( 
 \left\langle \mathbf{w}_{j, r}^{(t)}, y \boldsymbol{\mu}_2\right\rangle\right)\thinspace+ \sigma\!\left( \left\langle\mathbf{w}_{j, r}^{(t)}, \mathbf{z}_2 \right\rangle\right)\! \right],\\
&
f\!\left( \mathbf{W}^{(t)}, \mathbf{x}^{\prime} \right)\!= \sum_{j, r} \frac{j}{m}\left[\sigma\!\left( \left\langle \mathbf{w}_{j, r}^{(t)}, y^{\prime} \boldsymbol{\mu}_1 \right\rangle \right)\!+ \sigma\!\left( \left\langle\mathbf{w}_{j, r}^{(t)}, \mathbf{z}_1 \right\rangle\right)\! \right].
\end{aligned}
\]
By (\ref{Eq:order_absolute}) in Lemma \ref{Lemma:order_absolute} and (\ref{Eq:Diversity_Order}) in Definition \ref{Diversity_Order}, we have the following 
\begin{align*}
\mathbf{x}^\prime \preceq_C^{(t)} \mathbf{x} &\Leftrightarrow \underbrace{\left| f\left(\mathbf{W}^{(t)}, \mathbf{x} \right) \right|< \left| f\left(\mathbf{W}^{(t)}, \mathbf{x}^{\prime} \right) \right|}_{\Omega_C},\\
\mathbf{x}^{\prime} \preceq_{D}^{(t)} \mathbf{x} &\Leftrightarrow \underbrace{ D\left(\mathbf{W}^{(t)}, \mathbf{x}, p \ \mid \mathcal{D}_{n_0}\right) > D\left(\mathbf{W}^{(t)}, \mathbf{x}^\prime, p \ \mid \mathcal{D}_{n_0}\right)}_{\Omega_{D}},\\
\mathbf{x}^{\prime} \preceq^{(t)} \mathbf{x} &\Leftrightarrow \underbrace{ \{ \Omega_C \cap \Omega_{D},\forall p \in \left[1, \infty\right) \}}_{\Omega}
\end{align*}
Denote $\sum_{j} j \cdot \sigma\left(\left\langle\mathbf{w}_{j, r}^{(t)}, \mathbf{z}_1 \right\rangle \right)$, $\sum_{j} j \cdot \sigma\left(\left\langle\mathbf{w}_{j, r}^{(t)}, \mathbf{z}_2\right\rangle \right)$ as $g_r(\mathbf{z}_1)$, $g_r(\mathbf{z}_2)$ respectively, Notice that for $\mathbf{z} \sim N(\mathbf{0}, \sigma_p^2 \cdot \mathbf{I})$: 
\begin{equation}
\begin{aligned}
&\left\langle \mathbf{w}_{j, r}^{(t)}, \mathbf{z}\right\rangle \sim \mathcal{N}\left(0,\left\| \mathbf{w}_{j, r}^{(t)} \right\|_2^2 \sigma_p^2 \cdot \mathbf{I} \right), \\
& \sigma(\left\langle \mathbf{w}_{j, r}^{(t)}, \mathbf{z}\right\rangle) \sim \mathcal{N}^R\left(0,\left\| \mathbf{w}_{j, r}^{(t)} \right\|_2^2 \sigma_p^2 \cdot \mathbf{I} \right).
\end{aligned}
\label{EqXOR:<w·z>}\end{equation}
Then:
\begin{equation}
\begin{aligned}
P(\Omega_C) & =  P (\left| f\left(\mathbf{W}^{(t)}, \mathbf{x} \right) \right| < \left| f\left(\mathbf{W}^{(t)}, \mathbf{x}^{\prime} \right) \right|)\\
& \geq P(\sum_{l}(\sum_{r} \lvert g_r(\mathbf{z}_l) \rvert)  < \sum_{r}(\Theta(\gamma_{y^{\prime}, r, \boldsymbol{\mu}_1})-\Theta(\gamma_{y, r, \boldsymbol{\mu}_2})) )\\
& \geq P( m \cdot \max_{j, r, l}\{ \left| \left\langle\mathbf{w}_{j, r}^{(t)}, \mathbf{z}_l\right\rangle \right|\} < m (\Theta( \underset{r}{\mathbb{E}} (\gamma_{y^{\prime}, r, \boldsymbol{\mu}_1}))-\Theta(\underset{r}{\mathbb{E}}(\gamma_{y, r, \boldsymbol{\mu}_2}))) )\\
& = P (\underbrace{\max_{j, r, l}\{ \left| \left\langle\mathbf{w}_{j, r}^{(t)}, \mathbf{z}_l\right\rangle \right|\} < \Theta((\underset{r}{\mathbb{E}} (\gamma_{y^{\prime}, r, \boldsymbol{\mu}_1})-\underset{r}{\mathbb{E}}(\gamma_{y, r, \boldsymbol{\mu}_2}))}_{\Omega_{\gamma}}). \\
\end{aligned}
\label{EqXOR:OmegaC}\end{equation}
The second inequality is by triangle inequality and (\ref{eq:xor_ub-of-mu-same}) in Proposition \ref{prop:F.3}; the third inequality is by (\ref{EqXOR:<w·z>}).\par
For $\Omega_{D}$, denoting $U_0^l= \{ \mathbf{x} \in \mathcal{D}_0 \mid \mathbf{x}_{\text {signal part }}=\boldsymbol{\mu}_l\}$ as the set of indices of $\mathcal{D}_0$ where the data's feature patch is $\boldsymbol{\mu}_l$, We then take a look at the $r^{th} $ row of the Feature Distance $\mathbf{Z}(\mathbf{x}, t)$, which we denote as $\mathbf{Z}_r(\mathbf{x}, t)$:
\begin{equation}
\begin{aligned}
\mathbf{Z}_r(\mathbf{x}, t)&=\sum_j\left(\sigma \left(\left\langle \mathbf{w}_{j, r}, y \cdot \boldsymbol{\mu}_2 \right\rangle \right) + \sigma\left(\left\langle \mathbf{w}_{j, r}, \mathbf{z}_{r}\right\rangle\right)\right) \\
& =\Theta\left(\gamma_{y, r, \boldsymbol{\mu}_2}\right)+ g_r(\mathbf{z}_2) 
\end{aligned}
\label{EqXOR:rth row of FD sample}\end{equation}
\begin{equation}
\begin{aligned}
 \sum_{i} \dfrac{\mathbf{Z}_r(  \mathbf{x}^{(i)}, t)}{n_0}&=\sum_{i,j} \frac{\sigma\left(\left\langle \mathbf{w}_{j, r}, y_i \cdot \boldsymbol{\mu}^{(i)} \right\rangle \right)+\sigma\left(\left\langle \mathbf{w}_{j, r}, \boldsymbol{\xi}_i \right\rangle \right)}{n_0} \\
& =\dfrac{\left[\sum_{l} \tau_l \cdot n_0 \cdot \underset{i_l \in U_0^l}{\mathbb{E}} \Theta (\gamma_{y_{i_l}, r, \boldsymbol{\mu}_l}) +\sum_{i} \sum_{j} \Theta\left(\bar{\rho}_{j, r, i}\right)\right]}{n_0}
\end{aligned}
\label{EqXOR:rth row of FD labeledset}\end{equation}
Let (\ref{EqXOR:rth row of FD sample}) - (\ref{EqXOR:rth row of FD labeledset}), we have:
\begin{equation}
\mathbf{Z}_r(\mathbf{x}, t) - \sum_{i} \dfrac{\mathbf{Z}_r( \mathbf{x}^{(i)}, t)}{n_0}= \Theta(\gamma_{y, r, \boldsymbol{\mu}_2}) + g_r(\mathbf{z}_2) - \sum_{i} \dfrac{\mathbf{Z}_r( \mathbf{x}^{(i)}, t)}{n_0}
\end{equation}
Now we can estimate $D\left(\mathbf{W}^{(t)}, \mathbf{x}, p \ \mid \mathcal{D}_{n_0}\right)$:
\begin{equation}
\begin{aligned}
     D\left(\mathbf{W}^{(t)}, \mathbf{x}, p \ \mid \mathcal{D}_{n_0}\right) &= \| \mathbf{Z}(\mathbf{x}, t) - \sum_{i=1}^{n_0} \dfrac{\mathbf{Z}(  \mathbf{x}^{(i)}, t)}{n_0} \|_p \\
    & =\left( \sum_{r}  \lvert  \mathbf{Z}_r(\mathbf{x}, t) - \sum_{i} \dfrac{\mathbf{Z}_r( \mathbf{x}^{(i)}, t)}{n_0} \rvert^p \right)^{\frac{1}{p}}\\
    & = \left( \sum_{r}  \lvert  \Theta(\gamma_{y, r, \boldsymbol{\mu}_2})  + g_r(\mathbf{z}_2) - \sum_{i} \dfrac{\mathbf{Z}_r( \mathbf{x}^{(i)}, t)}{n_0} \rvert^p \right)^{\frac{1}{p}}
\end{aligned}
\label{EqXOR:D_x}\end{equation}
Similarly, the $D\left(\mathbf{W}^{(t)}, \mathbf{x}^\prime, p \ \mid \mathcal{D}_{n_0}\right)$ could be written as:
\begin{equation}
     D\left(\mathbf{W}^{(t)}, \mathbf{x}^\prime, p \ \mid \mathcal{D}_{n_0}\right) = \left( \sum_{r}  \lvert  \Theta(\gamma_{y, r, \boldsymbol{\mu}_1}) + g_r(\mathbf{z}_1)- \sum_{i} \dfrac{\mathbf{Z}_r( \mathbf{x}^{(i)}, t)}{n_0}  \rvert^p \right)^{\frac{1}{p}}
\label{EqXOR:D_x'}\end{equation}

To compare $D\left(\mathbf{W}^{(t)}, \mathbf{x}, p \ \mid \mathcal{D}_{n_0}\right)$ and $D\left(\mathbf{W}^{(t)}, \mathbf{x}^\prime, p \ \mid \mathcal{D}_{n_0}\right)$, we first see that both expressions in the $r$-th filter owns 
$$ - \sum_{i} \dfrac{\mathbf{Z}_r( \mathbf{x}^{(i)}, t)}{n_0}=- \sum_l \tau_l \cdot \Theta(\underset{i_l \in U_0^l}{\mathbb{E}}(\gamma_{y_{i_l}, r, \boldsymbol{\mu}_l}) ) - n_{0}^{-1}\sum_{i} \sum_{j} \Theta\left(\bar{\rho}_{j, r, i}\right).
$$
By Condition \ref{condition_XOR}, it holds that $\sigma_{p}^{2} d /(n_{0}\|\boldsymbol{\mu}_{1}\|_{2}^{2}) = \Omega(\log(T^{*}))$. We see that as $T^{*}$ is the substantially large maximum admissible iterations, collaborating with (\ref{eq:xor_ub-of-mu-same}), (\ref{EqXOR:rth row of FD labeledset}) and (\ref{EqXOR:<w·z>}), it holds that the order of $n_{0}^{-1} \sum_{i,j} \sigma\left(\left\langle \mathbf{w}_{j, r}, \boldsymbol{\xi}_i \right\rangle \right)=n_{0}^{-1}\sum_{i} \sum_{j} \Theta\left(\bar{\rho}_{j, r, i}\right)$ in $\sum_{i} \dfrac{\mathbf{Z}_r( \mathbf{x}^{(i)}, t)}{n_0}$ is indeed can dominate $n_{0}^{-1} \sum_{i,j}\sigma\left(\left\langle \mathbf{w}_{j, r}, y_i \cdot \boldsymbol{\mu}^{(i)} \right\rangle \right)=\sum_l \tau_l \cdot \Theta(\underset{i_l \in U_0^l}{\mathbb{E}}(\gamma_{y_{i_l}, r, \boldsymbol{\mu}_l}) )$, $\Theta(\gamma_{y, r, \boldsymbol{\mu}_1}) $ and $ g_r(\mathbf{z}_1)$. As $\sum_{i} \dfrac{\mathbf{Z}_r( \mathbf{x}^{(i)}, t)}{n_0}$ is shared by both $D\left(\mathbf{W}^{(t)}, \mathbf{x}, p \ \mid \mathcal{D}_{n_0}\right)$ and $D\left(\mathbf{W}^{(t)}, \mathbf{x}^\prime, p \ \mid \mathcal{D}_{n_0}\right)$ in the $r$-th filter, a sufficient event for $D\left(\mathbf{W}^{(t)}, \mathbf{x}, p \ \mid \mathcal{D}_{n_0}\right) > D\left(\mathbf{W}^{(t)}, \mathbf{x}^\prime, p \ \mid \mathcal{D}_{n_0}\right)$ is that for $\forall r \in [m]$, it holds that 

\[
\lvert \sum_l \tau_l \cdot \Theta(\underset{i_l \in U_0^l}{\mathbb{E}}(\gamma_{y_{i_l}, r, \boldsymbol{\mu}_l}) ) - \Theta(\gamma_{y, r, \boldsymbol{\mu}_2}) - g_r(\mathbf{z}_2) \rvert > \lvert \max\{ \sum_l \tau_l \cdot \Theta(\underset{i_l \in U_0^l}{\mathbb{E}}(\gamma_{y_{i_l}, r, \boldsymbol{\mu}_l}) ) - \Theta(\gamma_{y, r, 1}) - g_r(\mathbf{z}_1),0 \} \rvert.
\]

Utilizing those results, we now could estimate the chance of event $\Omega_{D}$:
\begin{equation}
\begin{aligned}
P(\Omega_{D}) & = P (D\left(\mathbf{W}^{(t)}, \mathbf{x}, p \ \mid \mathcal{D}_{n_0}\right) > D\left(\mathbf{W}^{(t)}, \mathbf{x}^\prime, p \ \mid \mathcal{D}_{n_0}\right))\\
& \geq P(m^{\frac{1}{p}}\sum_{l}(\max_{r} \lvert g_r(\mathbf{z}_l) \rvert) < m^{\frac{1}{p}}(\lvert  \Theta(\underset{r}{\mathbb{E}}(\gamma_{y, r, \boldsymbol{\mu}_2}))-\sum_l \tau_l \cdot \Theta(\underset{i_l \in U_0^l, r}{\mathbb{E}}(\gamma_{y_{i_l}, r, \boldsymbol{\mu}_l}) )  \rvert \\
& \phantom{\geq P(m^{\frac{1}{p}}\sum_{l}(\sum_{r} \lvert g_r(\mathbf{z}_l) \rvert) < m^{\frac{1}{p}}} - \lvert  \Theta(\underset{r}{\mathbb{E}}(\gamma_{y, r, \boldsymbol{\mu}_1}))-\sum_l \tau_l \cdot \Theta(\underset{i_l \in U_0^l, r}{\mathbb{E}}(\gamma_{y_{i_l}, r, \boldsymbol{\mu}_l}) )  \rvert) \\
& \geq P ( m^{\frac{1}{p}} \max_{j, r, l}\{ \left| \left\langle\mathbf{w}_{j, r}^{(t)}, \mathbf{z}_l\right\rangle \right|\} < m^{\frac{1}{p}}\left((\tau_1 - \tau_2) \Theta(\underset{j,r}{\mathbb{E}}(\gamma_{j, r, \boldsymbol{\mu}_1})) - (\tau_1 - \tau_2) \Theta(\underset{j,r}{\mathbb{E}}(\gamma_{j, r, \boldsymbol{\mu}_2}))\right) \\
& = P (m^{\frac{1}{p}} \max_{j, r, l}\{ \left| \left\langle\mathbf{w}_{j, r}^{(t)}, \mathbf{z}_l\right\rangle \right|\} < m^{\frac{1}{p}}\Theta(\dfrac{\tau_1 (\tau_1 - \tau_2)\| \boldsymbol{\mu}_1\|_2^2- \tau_2 (\tau_1 - \tau_2)\| \boldsymbol{\mu}_2\|_2^2}{\sigma_p^2 d/n_0})\cdot \log \left(\dfrac{\eta \sigma_p^2 d t}{n m}\right))\\
& = P (m^{\frac{1}{p}} \max_{j, r, l}\{ \left| \left\langle\mathbf{w}_{j, r}^{(t)}, \mathbf{z}_l\right\rangle \right|\} < m^{\frac{1}{p}}\Theta(\underset{r}{\mathbb{E}} (\gamma_{y^{\prime}, r, \boldsymbol{\mu}_1})-\underset{r}{\mathbb{E}}(\gamma_{y, r, \boldsymbol{\mu}_2})))\\
& = P (\underbrace{\max_{j, r, l}\{ \left| \left\langle\mathbf{w}_{j, r}^{(t)}, \mathbf{z}_l\right\rangle \right|\} < \Theta((\underset{r}{\mathbb{E}} (\gamma_{y^{\prime}, r, \boldsymbol{\mu}_1})-\underset{r}{\mathbb{E}}(\gamma_{y, r, \boldsymbol{\mu}_2}))}_{\Omega_{\gamma}}),
\end{aligned}
\label{EqXOR:OmegaD}\end{equation}
where the first inequality is by triangle inequality, (\ref{EqXOR:D_x}) and (\ref{EqXOR:D_x'}); The forth equality is by (\ref{EqXOR:<w·z>}). Easy to see that if $p=\infty$, the third equality would be zero, thus our condition $p < \infty$ avoid this case. Now we take a look at the event $\Omega_{\gamma}$:
\begin{equation}
\begin{aligned}
P (\Omega_{\gamma}) & = P (\max_{j, r, l}\{ \left| \left\langle\mathbf{w}_{j, r}^{(t)}, \mathbf{z}_l\right\rangle \right|\} < \Theta((\underset{r}{\mathbb{E}} (\gamma_{y^{\prime}, r, \boldsymbol{\mu}_1})-\underset{r}{\mathbb{E}}(\gamma_{y, r, \boldsymbol{\mu}_2})))\\
& = P (\max_{j, r, l}\{ \left| \left\langle\mathbf{w}_{j, r}^{(t)}, \mathbf{z}_l\right\rangle \right|\} < \Theta\left(\dfrac{\left[\tau_1 \left\|\boldsymbol{\mu}_1\right\|_2^2-\tau_2\left\|\boldsymbol{\mu}_2\right\|_2^2\right]}{\sigma_p^2 d / n_0 } \cdot \log \left(\dfrac{\eta \sigma_p^2 d t}{n m}\right)\right) )\\
& \geq P (\bigcup_{j, r, l} \underbrace{\{ \left| \left\langle\mathbf{w}_{j, r}^{(t)}, \mathbf{z}_l\right\rangle - 0 \right| < \Theta\left(\dfrac{\left[\tau_1 \left\|\boldsymbol{\mu}_1\right\|_2^2-\tau_2\left\|\boldsymbol{\mu}_2\right\|_2^2\right]}{\sigma_p^2 d / n_0 }\cdot \log \left(\dfrac{\eta \sigma_p^2 d t}{n m}\right) \right) \} }_{\hat{\Omega}_{j, r, l}})\\
& = \sum_{j, r, l} P(\hat{\Omega}_{j, r, l}),
\end{aligned}
\label{eqXOR:Omega}\end{equation}
where the second equality is by the first inference statement of Lemma \ref{app:lem:final coefficient}; the third inequality is by the equivalence property of the union by events; the last equality is by the Union Rule. Then, by Gaussian tail bound, we have:
\[
P(\hat{\Omega}_{j, r, l}) \geq 1- 2 \exp \left\{-\Theta\left(\dfrac{\left[ \tau_1 \left\|\boldsymbol{\mu}_1\right\|_2^2- \tau_2 \left\|\boldsymbol{\mu}_2\right\|_2^2\right]^2}{\sigma_p^{6} d^2 / n_0^2 \left\|w_{j, r}^{(t)}\right\|_2^2} \cdot \log \left(\dfrac{\eta \sigma_p^2 d t}{n m}\right)\right)\right\}
\]
Finally, with conditions in Proposition \ref{app:propXOR:sample stage:learning on sample}, Lemma \ref{lem:rho n}, Proposition \ref{prop:F.3} and union bound, we have the conclusion for event $\Omega$:
\begin{equation}
\begin{aligned}
\Rightarrow P(\Omega) \geq P (\Omega_{\gamma})  & \geqslant 1-8 m \exp \left\{-\Theta\left(\frac{\left[ \tau_1 \left\|\boldsymbol{\mu}_1\right\|_2^2- \tau_2 \left\|\boldsymbol{\mu}_2\right\|_2^2\right]^2}{\sigma_p^{4} d / n_0 }\right)\right\}\\
& \geqslant 1-\delta^{\prime},
\end{aligned}
\label{eqXOR:sufficient OMEGA}\end{equation}
for $\forall p \in \left[1, \infty\right)$. \par
\begin{remark}
The proof process is nearly identical to that of the linearly separable case (i.e., the proof of Proposition \ref{prop_App:order_newdata}). The only differences lie in the scale of $\|w_{j, r}^{(t)}\|_2$ and $\gamma_{\pm 1, r, \boldsymbol{\mu}}$, but the conditions required are the same. 
\end{remark}
Similar to Lemma \ref{app:lemma of comparison} in Appendix \ref{app:Order-dependent Sampling(Querying Analysis)}, we have the following lemma.
\begin{lemma} 
 Under the same conditions in Proposition \ref{prop:sample stage: learning on sample}, with the same notations in Proposition \ref{prop_App_XOR:order_newdata}, there exists certain constants $ c_1, c_2, c_3, c_4, c_5, c_6 >0 $, such that
 \begin{itemize}
     \item $\mathbf{x} \preceq_{C}^{(t)} \mathbf{x}^{\prime}$ has a sufficient event that
\begin{equation}
 \{ c_1 \underset{r}{\mathbb{E}} (\gamma_{y^{\prime}, r, \boldsymbol{\mu}_1}) - c_2 \underset{r}{\mathbb{E}}(\gamma_{y, r, \boldsymbol{\mu}_2}) > \max_{j, r, l}\{ \left| \left\langle\mathbf{w}_{j, r}^{(t)}, \mathbf{z}_l\right\rangle \right|\}\},
\label{eqXOR:Comparision_Uncertainty}\end{equation}
  among which the left side of the inequality corresponds to the comparison of learning progress of samples with different type of feature patch.
    \item $\mathbf{x} \preceq_{D}^{(t)} \mathbf{x}^{\prime}, \forall p \in [1, \infty)$ has a sufficient event that
 \begin{equation}
 \{ \lvert  c_3 \underset{r}{\mathbb{E}}(\gamma_{y, r, \boldsymbol{\mu}_2})-c_4 \sum_l \tau_l \cdot \underset{i_l \in U_0^l, r}{\mathbb{E}}(\gamma_{y_{i_l}, r, \boldsymbol{\mu}_l}) \rvert - \lvert  c_5 \underset{r}{\mathbb{E}}(\gamma_{y^{\prime}, r, \boldsymbol{\mu}_1})-c_6 \sum_l \tau_l \cdot \underset{i_l \in U_0^l, r}{\mathbb{E}}(\gamma_{y_{i_l}, r, \boldsymbol{\mu}_l}) \rvert > \max_{j, r, l}\{ \left| \left\langle\mathbf{w}_{j, r}^{(t)}, \mathbf{z}_l\right\rangle \right|\} \},
\label{eqXOR:Comparision_Diversity}\end{equation}
  among which the left side of the inequality corresponds to the comparison of the disparity between learning toward samples and labeled training set.
 \end{itemize}
\label{appXOR:lemma of comparison}\end{lemma}

\textit{Proof of Lemma \ref{appXOR:lemma of comparison}.} The first bullet point can be easily derived from (\ref{EqXOR:OmegaC}), while the second bullet point is readily apparent from (\ref{EqXOR:D_x}), (\ref{EqXOR:D_x'}), and (\ref{EqXOR:OmegaD}). \par

Similar to the discussions in Appendix \ref{app:Order-dependent Sampling(Querying Analysis)}, it is observed that for any $p \in [1, \infty)$, there exists a shared sufficient event for (\ref{eqXOR:Comparision_Uncertainty}) and (\ref{eqXOR:Comparision_Diversity}). This implies that it is also a shared sufficient event for the events $\Omega_C$ and $\Omega_{D}$, denoted as $\Omega_{\gamma}$:
\[
\Omega_{\gamma} \mathrel{\mathop:}= \{ \max_{j, r, l}\{ \left| \left\langle\mathbf{w}_{j, r}^{(t)}, \mathbf{z}_l\right\rangle \right|\} < \Theta((\underset{r}{\mathbb{E}} (\gamma_{y^{\prime}, r, \boldsymbol{\mu}_1})-\underset{r}{\mathbb{E}}(\gamma_{y, r, \boldsymbol{\mu}_2})) \}.
\]
By the first inference statement of Proposition \ref{prop:F.3}, we have 
\begin{equation}
\Omega_{\gamma} = \{ \max_{j, r, \boldsymbol{\mu}_l}\{ \left| \left\langle\mathbf{w}_{j, r}^{(t)}, \mathbf{z}_l\right\rangle \right|\} < \Theta((\underset{j, r}{\mathbb{E}} (\gamma_{j, r, \boldsymbol{\mu}_1})-\underset{j, r}{\mathbb{E}}(\gamma_{j, r, \boldsymbol{\mu}_2})) \}.
\label{eqXOR:F.20}\end{equation}
Therefore, we can conclude that the significant difference in the model's learning of the feature $\boldsymbol{\mu}_1$ and $\boldsymbol{\mu}_2$ is what causes the sufficient event for both event $\Omega_C$ and $\Omega_{D}$. By (\ref{eqXOR:sufficient OMEGA}), we have:
\begin{equation}
    P (\Omega_{\gamma}) \geq 1-8 m \exp \left\{-\Theta\left( \underset{j, r}{\mathbb{E}}(\gamma_{j, r, \boldsymbol{\mu}_1}) - \underset{j, r}{\mathbb{E}}(\gamma_{j, r, \boldsymbol{\mu}_2})\right)\right\}.
\end{equation}
Based on Proposition \ref{prop:F.3}, we see that the $\underset{j, r}{\mathbb{E}}(\gamma_{j, r, \boldsymbol{\mu}_1})$ is significant larger than $\underset{j, r}{\mathbb{E}}(\gamma_{j, r, \boldsymbol{\mu}_2})$ under our conditions, which causes the sufficient event $\Omega_{\gamma}$. \par

Similar to Lemma \ref{Lem:order_pool} for linearly separable XOR data, we also have conclusions regarding the order of pool for XOR data. 
\begin{lemma} 
    Under Condition \ref{condition_XOR}, when the results of Proposition \ref{Prop4.3} and Proposition \ref{prop_App_XOR:order_newdata} hold at the initial stage and querying stage at a certain $t \leq T^*$, denoting $\mathbf{X}_{\mathcal{P}}^1 \subsetneqq \mathcal{P}$ as the collection of all the data points with strong feature $\boldsymbol{\mu}_1$ in $\mathcal{P}$, and  $\mathbf{X}_{\mathcal{P}}^2 \subsetneqq \mathcal{P}$ as the collection of data points with weak feature $\boldsymbol{\mu}_2$, we have the conclusion that with probability more than 1-$\Theta(\delta^\prime)$, $\mathbf{X}_{\mathcal{P}}^1 \prec^{(t)} \mathbf{X}_{\mathcal{P}}^2$ holds.
\label{appXOR:Lem:order_pool_app}\end{lemma}
\textit{proof of Lemma \ref{appXOR:Lem:order_pool_app}.} See Lemma \ref{app:Lem:order_pool_app} for a proof.\par
Similar to Lemma \ref{app:lem:sample stage: learning on sample}, we directly have the following lemma demonstrate that both NAL algorithms would all prioritize those \hyperlink{msg:perplexing}{\textbf{perplexing samples}}.
\begin{lemma}
    (Formal Restatement of Proposition \ref{app:propXOR:sample stage:learning on sample}) Under the same conditions in Proposition \ref{app:propXOR:sample stage:learning on sample}, the Uncertainty Order and Diversity Order of the samples $[(y\cdot \boldsymbol{\mu}_l)^T, \mathbf{\xi}^T]^T$  in sampling pool $\mathcal{P}$ follows the order of $\displaystyle \underset{j,k,l}{\mathbb{E}} \gamma_{j,k,\boldsymbol{\mu}_l}^{(t)}$.
\label{appXOR:lem:sample stage: learning on sample}\end{lemma}

\subsection{Label Complexity-based Test Error Analysis: XOR data version}\label{AppXOR:Label Complexity-based Test Error Analysis}
The underlying philosophy in this section is the same as that in Appendix \ref{App:Label Complexity-based Test Error Analysis} for the theory regarding linearly separable data. We propose that the results obtained in the previous section hold with high probability. By considering the scale of the coefficients, inner products, and the order of the data in the sampling pool $\mathcal{P}$, we can now examine the upper and lower bounds of the test error under different conditions before and after querying.

\begin{lemma} 
    Under Condition \ref{condition_XOR}, for a test set $\mathcal{D^*} \subseteq \mathcal{D^*}$ with occurrence probability $p^*$ of the $\boldsymbol{\mu}_2$-equipped data, then $\exists\ t=\widetilde{O}\left(\eta^{-1} \varepsilon^{-1} m \widetilde{n} d^{-1} \sigma_p^{-2}\right)$, we have the following two situations before and after querying (i.e., $\forall s \in \{0, 1\}$):
\begin{itemize}
    \item For $t=\Omega\left(\widetilde{n} m /\left(\eta \sigma_p^2 d \varepsilon \right)\right)$, the training loss converges $L_S\left(\mathbf{W}^{(t)}\right) \leq \varepsilon$.
    \item If $\forall l \in \{1,2\}, n_{s, l} \geq \dfrac{\hat{C}_1 \sigma_p^4 d}{\|\boldsymbol{\mu}_l\|_2^4}$ holds, we have the test error:
    \begin{equation}
    L_{\mathcal{D}^*}^{0-1}\left(\mathbf{W}^{(t)}\right) \leq (1-p^*) \cdot \exp \left( \dfrac{-n_{s, 1}\|\boldsymbol{\mu}_1\|_2^4}{\hat{C}_3 \sigma_p^4 d} \right) + p^* \cdot \exp \left( \dfrac{-n_{s, 2}\|\boldsymbol{\mu}_2 \|_2^4}{\hat{C}_4 \sigma_p^4 d} \right).
    \label{eqXOR:E.21}\end{equation}
    \item If $\exists l^{\prime} \in \{1, 2\} n_{s, l^{\prime}} \leq \dfrac{\hat{C}_2 \sigma_p^4 d}{\|\boldsymbol{\mu}_{l^{\prime}}\|_2^4 }$ holds, where $\hat{C}_1$ is from Condition \ref{Con4.1}, we have the test error 
    \begin{equation}
    L_{\mathcal{D}^*}^{0-1}\left(\mathbf{W}^{(t)}\right)  \geq 0.12 \cdot \tau^*_{l^{\prime}}. 
    \label{eqXOR:E.22}\end{equation}
\end{itemize}
Here $\tau^*_{l^\prime}$ denotes the occurrence probability of feature $\boldsymbol{\mu}_{l^\prime}$, $\hat{C}_1$, $\hat{C}_2$, $\hat{C}_3$ and $\hat{C}_4$ are some positive constants.
\label{appXOR:lemma for thm}\end{lemma}
\textit{Proof of Lemma \ref{appXOR:lemma for thm}.} The proof flow follows Theorem 3.2 in \citet{meng2023benign} despite that we consider two features. For the training convergence, by Proposition \ref{prop:F.2} we have 
$$
\begin{aligned}
y_i f\left(\mathbf{W}^{(t)}, \mathbf{x}_i\right) & \geq-\frac{\kappa}{2}+\frac{1}{m} \sum_{r=1}^m \bar{\rho}_{y_i, r, i}^{(t)} \\
& \geq-\frac{\kappa}{2}+\underline{x}_t \\
& \geq-\kappa+\log \left(\Theta(\frac{\eta \sigma_p^2 d}{ n_{s} m}) t+\frac{2}{3}\right).
\end{aligned}
$$
Recall $\kappa$ is defined in (\ref{eqXOR:alpha,belta,snr}). Here, the first inequality is by the conclusion in Proposition \ref{prop:F.2} and the second inequality is by (\ref{eqXOR:scale of barunderline_x}) Proposition \ref{prop:F.2}, and last inequality are by (\ref{eqXOR:scale of barunderline_x}). Then we have
$$
L\left(\mathbf{W}^{(t)}\right) \leq \log \left(1+\exp \{\kappa\} /\left(\Theta(\frac{\eta \sigma_p^2 d}{n_{s} m}) t+\frac{2}{3}\right)\right) \leq \frac{e^\kappa}{\Theta(\frac{\eta \sigma_p^2 d}{ n_{s} m}) t+\frac{2}{3}} \leq \frac{e^\kappa}{2 / \varepsilon+\frac{2}{3}} \leq \varepsilon 
$$
The last inequality is by $\log (1+x) \leq x$, $t \geq \Omega\left(\frac{\widetilde{n} m}{\eta \sigma_p^2 d \varepsilon}\right)$ and $exp\{ \kappa \} \leq 1.5$. \par
For evaluating test error, same as techniques in Lemma \ref{app:lemma for thm}, we have
\begin{equation}
\begin{aligned}
L_{\mathcal{D}^*}^{0-1}(\mathbf{W}) &=\mathbb{P}_{(\mathbf{x}, y) \sim \mathcal{D}^*}[y \cdot f(\mathbf{W}, \mathbf{x}) < 0]\\
&=(1-p^*) \cdot \mathbb{P}_{(\mathbf{x}, y) \sim \mathcal{D}_{\boldsymbol{\mu}_1}^*}[y \cdot f(\mathbf{W}, \mathbf{x}) < 0] +p^* \cdot \mathbb{P}_{(\mathbf{x}, y) \sim \mathcal{D}_{\boldsymbol{\mu}_2}^*}[y \cdot f(\mathbf{W}, \mathbf{x}) < 0],
\end{aligned}
\label{eqXOR:f.22}\end{equation}
where $\mathcal{D}_{\boldsymbol{\mu}_1}^*$  and $\mathcal{D}_{\boldsymbol{\mu}_2}^*$ denotes the collection of data points in $\mathcal{D}$ containing feature $\boldsymbol{\mu}_1$ and $\boldsymbol{\mu}_2$, respectively. Notably, $\mathbb{P}_{(\mathbf{x}, y) \sim \mathcal{D}_{\boldsymbol{\mu}_l}^*}[y \cdot f(\mathbf{W}, \mathbf{x}) < 0]$ is equal to
$$
\sum_{\boldsymbol{\mu} \in \{ \pm \mathbf{u}_l, \pm \mathbf{v}_l \}} P\left(y f\left(\mathbf{W}^{(t)}, \mathbf{x}\right)>0 \mid \mathbf{x}_{\text {signal part }}=\boldsymbol{\mu}\right) \cdot \frac{1}{4},
$$
then without loss of generality, we can only investigate
$$
P\left(1 \cdot f\left(\mathbf{W}^{(t)}, \mathbf{x}\right)>0 \mid \mathbf{x}_{\text {signal part }}=\mathbf{u}_l \right), \forall l \in \{1, 2\}
$$
and the proofs for other cases (i.e., $\boldsymbol{\mu} \in \{-\mathbf{u}_1, -\mathbf{u}_2, \pm \mathbf{v}_1, \pm \mathbf{v}_2\}$) are the same. Denote the feature patch in $\mathbf{x}$ as $\mathbf{u}_{l_x}$ ($l_x \in \{1, 2\}$), when $\mathbf{x}=\left(\mathbf{u}_{l_x}^{\top}, \boldsymbol{\xi}^{\top}\right)^{\top}$, the true label $y=+1$. Considering this case, we have
$$
\begin{aligned}
1 \cdot f\left(\mathbf{W}^{(t)}, \mathbf{x}\right) &=\frac{1}{m} \sum_{r=1}^m F_{+1, r}\left(\mathbf{W}^{(t)}, \mathbf{u}_{l_x}\right)+F_{+1, r}\left(\mathbf{W}^{(t)}, \boldsymbol{\xi}\right)-\frac{1}{m} \sum_{r=1}^m\left(F_{-1, r}\left(\mathbf{W}^{(t)}, \mathbf{u}_{l_x}\right)+F_{-1, r}\left(\mathbf{W}^{(t)}, \boldsymbol{\xi}\right)\right) \\
& \leq \frac{1}{m} \left[ \sum_r \sigma\left(\left\langle\mathbf{w}_{+1, r}^{(t)},  \mathbf{u}_{l_x} \right\rangle\right) - \sum_r \sigma\left(\left\langle\mathbf{w}_{-1, r}^{(t)}, \boldsymbol{\xi}\right\rangle\right) \right].
\end{aligned}
$$
Then we can adopt the exact same techniques in Lemma \ref{app:lemma for thm}. Recall $g(\boldsymbol{\xi})$ is denoted as $\sum_r \sigma\left(\left\langle\mathbf{w}_{-y, r}^{(t)}, \boldsymbol{\xi}\right\rangle\right)$, also (\ref{eq:E.25}):
\begin{equation}
\mathbb{E} g(\boldsymbol{\xi})=\sum_{r=1}^m \mathbb{E} \sigma\left(\left\langle\mathbf{w}_{-y, r}^{(t)}, \boldsymbol{\xi}\right\rangle\right)=\sum_{r=1}^m \frac{\left\|\mathbf{w}_{-y, r}^{(t)}\right\|_2 \sigma_p}{\sqrt{2 \pi}}=\frac{\sigma_p}{\sqrt{2 \pi}} \sum_{r=1}^m\left\|\mathbf{w}_{-y, r}^{(t)}\right\|_2 . 
\label{eq:F.25}\end{equation}
Then we can obtain the following test error upper bound on $\mathcal{D}_{\mathbf{u}_{l_x}}^*$ by adding $\mathbb{E} g(\boldsymbol{\xi})$ and $\dfrac{\sigma_p}{\sqrt{2 \pi}} \sum_{r=1}^m\left\|\mathbf{w}_{-y, r}^{(t)}\right\|_2$ at two sides of the inequality:
\begin{equation}
\begin{aligned}
\mathbb{P}_{(\mathbf{x}, +1) \sim \mathcal{D}_{\mathbf{u}_{l_x}}^*}\left(1 \cdot  f\left(\boldsymbol{W}^{(t)}, \mathbf{x}\right) \leq 0\right) & \leq P \left(\sum_r \sigma\left(\left\langle\mathbf{w}_{-1, r}^{(t)}, \boldsymbol{\xi}\right\rangle\right) \geq \sum_r \sigma\left(\left\langle\mathbf{w}_{1, r}^{(t)},  \mathbf{u}_{l_x}\right\rangle\right)\right) \\
& =P\left(g(\boldsymbol{\xi})-\mathbb{E} g(\boldsymbol{\xi}) \geq \sum_r \sigma\left(\left\langle\mathbf{w}_{1, r}^{(t)},  \mathbf{u}_{l_x}\right\rangle\right)-\frac{\sigma_p}{\sqrt{2 \pi}} \sum_{r=1}^m\left\|\mathbf{w}_{-1, r}^{(t)}\right\|_2\right).
\end{aligned}
\label{eq:F.26}\end{equation}
By the results in Proposition \ref{prop:F.3}, we take a look at the comparison of the two terms at the right side of the inequality:
\begin{equation}
\frac{\sum_r \sigma\left(\left\langle\mathbf{w}_{y, r}^{(t)}, y \mathbf{u}_{l_x}\right\rangle\right)}{\sigma_p \sum_{r=1}^m\left\|\mathbf{w}_{-1, r}^{(t)}\right\|_2} \geq \frac{\Theta\left(\sum_r \gamma_{1, r, \mathbf{u}_{l_x}}^{(t)}\right)}{\Theta\left(d^{-1 / 2} n_{s}^{-1 / 2}\right) \cdot \sum_{r, i} \bar{\rho}_{-1, r, i}^{(t)}}=\Theta\left(\tau_{l_x}d^{1 / 2} n_{s}^{1 / 2} \operatorname{SNR}_{l_x}^2\right)=\Theta\left(\tau_{l_x}n_{s}^{1 / 2}\|\mathbf{u}_{l_x}\|_2^2 / (\sigma_p^2 d^{1 / 2})\right),
\label{eq:F.27}\end{equation}
where $\tau_{l_x}$ denotes the proportion of feature $\mathbf{u}_{l_x}$ in current training data set (before or after querying). Worth noting that we have assumption in the first bullet that $\forall l \in \{1,2\}, n_{s, l} \geq \dfrac{\hat{C}_1 \sigma_p^4 d}{\|\mathbf{u}_l\|_2^4}$, which means $n_{1,l_x}\|\mathbf{u}_1\|_2^4 \geq 2\hat{C}_1 \sigma_p^4 d, \forall l_x \in \{1, 2\}$. Since $\hat{C}_1$ is a sufficiently large constant, it directly follows that
$$
\sum_r \sigma\left(\left\langle\mathbf{w}_{1, r}^{(t)},  \mathbf{u}_{l_x}\right\rangle\right)-\frac{\sigma_p}{\sqrt{2 \pi}} \sum_{r=1}^m\left\|\mathbf{w}_{-1, r}^{(t)}\right\|_2>0 .
$$
Same as (\ref{eq:F.28}), we adopt the techniques of Theorem 5.2.2 in \citet{vershynin2018high}:
\begin{equation}
P(g(\boldsymbol{\xi})-\mathbb{E} g(\boldsymbol{\xi}) \geq x) \leq \exp \left(-\frac{c x^2}{\sigma_p^2\|g\|_{\text {Lip }}^2}\right),
\label{eq:F.28}\end{equation}
where $c$ is a constant. To calculate the Lipschitz norm, we have
$$
\begin{aligned}
\left|g(\boldsymbol{\xi})-g\left(\boldsymbol{\xi}^{\prime}\right)\right| & =\left|\sum_{r=1}^m \sigma\left(\left\langle\mathbf{w}_{-1, r}^{(t)}, \boldsymbol{\xi}\right\rangle\right)-\sum_{r=1}^m \sigma\left(\left\langle\mathbf{w}_{-y, r}^{(t)}, \boldsymbol{\xi}^{\prime}\right\rangle\right)\right| \\
& \leq \sum_{r=1}^m\left|\sigma\left(\left\langle\mathbf{w}_{-1, r}^{(t)}, \boldsymbol{\xi}\right\rangle\right)-\sigma\left(\left\langle\mathbf{w}_{-1, r}^{(t)}, \boldsymbol{\xi}^{\prime}\right\rangle\right)\right| \\
& \leq \sum_{r=1}^m\left|\left\langle\mathbf{w}_{-1, r}^{(t)}, \boldsymbol{\xi}-\boldsymbol{\xi}^{\prime}\right\rangle\right| \\
&\leq \sum_{r=1}^m\left\|\mathbf{w}_{-1, r}^{(t)}\right\|_2 \cdot\left\|\boldsymbol{\xi}-\boldsymbol{\xi}^{\prime}\right\|_2,
\end{aligned}
$$
where the first inequality is by triangle inequality; the second inequality is by the property of ReLU; the last inequality is by Cauchy Schwartz Inequality. Therefore, we have
\begin{equation}
\|g\|_{\text {Lip }} \leq \sum_{r=1}^m\left\|\mathbf{w}_{-1, r}^{(t)}\right\|_2.
\label{eq:F.29}\end{equation}
Utilize (\ref{eq:F.28}) and (\ref{eq:F.29}) in (\ref{eq:F.26}), we have
\begin{equation}
\begin{aligned}
\mathbb{P}_{(\mathbf{x}, +1) \sim \mathcal{D}_{\mathbf{u}_{l_x}}^*}\left(1 \cdot  f\left(\boldsymbol{W}^{(t)}, \mathbf{x}\right) \leq 0\right)
& \leq \exp \left[-\frac{c\left(\sum_r \sigma\left(\left\langle\mathbf{w}_{1, r}^{(t)},  \mathbf{u}_{l_x}\right\rangle\right)-\left(\dfrac{\sigma_p}{\sqrt{2 \pi}}\right) \sum_{r=1}^m\left\|\mathbf{w}_{-1, r}^{(t)}\right\|_2\right)^2}{\sigma_p^2\left(\sum_{r=1}^m\left\|\mathbf{w}_{-1, r}^{(t)}\right\|_2\right)^2}\right] \\
&= \exp \left[-c\left(\frac{\sum_r \sigma\left(\left\langle\mathbf{w}_{1, r}^{(t)},  \mathbf{u}_{l_x}\right\rangle\right)}{\sigma_p \sum_{r=1}^m \| \mathbf{w}_{-1, r}^{(t)}\|_2} - \dfrac{1}{\sqrt{2 \pi}}\right)^2 \right]\\
& \leq \exp (c / 2 \pi) \exp \left(-0.5 c\left(\frac{\sum_r \sigma\left(\left\langle\mathbf{w}_{1, r}^{(t)},  \mathbf{u}_{l_x}\right\rangle\right)}{\sigma_p \sum_{r=1}^m\left\|\mathbf{w}_{-1, r}^{(t)}\right\|_2}\right)^2\right),
\end{aligned}
\label{eq:F.30}\end{equation}
where the third inequality is by $(s-t)^2 \geq s^2 / 2-t^2, \forall s, t \geq 0$.
And then by (\ref{eq:F.27}) and (\ref{eq:F.30}), we can have
\begin{equation}
\begin{aligned}
\mathbb{P}_{(\mathbf{x}, +1) \sim \mathcal{D}_{\mathbf{u}_{l_x}}^*}\left(1 \cdot  f\left(\boldsymbol{W}^{(t)}, \mathbf{x}\right) \leq 0\right) & \leq \exp (c / 2 \pi) \exp \left(-0.5 c\left(\frac{\sum_r \sigma\left(\left\langle\mathbf{w}_{1, r}^{(t)},  \mathbf{u}_{l_x}\right\rangle\right)}{\sigma_p \sum_{r=1}^m\left\|\mathbf{w}_{-1, r}^{(t)}\right\|_2}\right)^2\right) \\
& =\exp \left(\frac{c}{2 \pi}-\frac{\tau_{l_x}n_{s, l_x}\|\mathbf{u}_{l_x}\|_2^4}{\hat{C} \sigma_p^4 d}\right) \\
& =\exp \left(\frac{c}{2 \pi}-\frac{n_{s, l_x}\|\mathbf{u}_{l_x}\|_2^4}{\hat{C}_{l_x} \sigma_p^4 d}\right) \\
& \leq \exp \left(-\frac{n_{s, l_x}\|\mathbf{u}_{l_x}\|_2^4}{2 \hat{C}_{l_x} \sigma_p^4 d}\right)
\end{aligned}
\label{eq:F.31}\end{equation}
where $\hat{C}_{l_x}=\hat{C}/\tau_{lx}=O(1)$; the last inequality holds if we choose $\hat{C}_1 \geq c \hat{C}_{l_x} / \pi$, for any $l_x \in \{1, 2\}$. If we choose $\hat{C}_3$ as $2 \hat{C}_{l_1}$ and $\hat{C}_4$ as $2 \hat{C}_{l_2}$, by (\ref{eqXOR:f.22}) and (\ref{eq:F.31}) we have
\[
L_{\mathcal{D}^*}^{0-1}\left(\mathbf{W}^{(t)}\right) \leq (1-p^*) \cdot \exp \left( \dfrac{-n_{s, 1}\|\mathbf{u}_1\|_2^4}{\hat{C}_3 \sigma_p^4 d} \right) + p^* \cdot \exp \left( \dfrac{-n_{s, 2}\|\mathbf{u}_2 \|_2^4}{\hat{C}_4 \sigma_p^4 d} \right).
\]
Next, we serve to prove the test error upper bound. Same as the proof in Lemma \ref{app:lemma for thm}, we utilize the pigeonhole principle technique in \citet{kou2023benign, meng2023benign}, which is based on the following two lemmas.
\begin{lemma}
For $t \in\left[T_1, T^*\right]$, denote $g(\boldsymbol{\xi})=\sum_{j, r} \sigma\left(\left\langle\mathbf{w}_{j, r}^{(t)}, \boldsymbol{\xi}\right\rangle\right)$. There exists a fixed vector $\mathbf{v}_l$ with $\|\mathbf{v}_l\|_2 \leq 0.01 \sigma_p$ and constant $\hat{C}_6$ such that
$$
\sum_{j^{\prime} \in\{ \pm 1\}}\left[g\left(j^{\prime} \boldsymbol{\xi}+\mathbf{v}_l\right)-g\left(j^{\prime} \boldsymbol{\xi}\right)\right] \geq 4 \hat{C}_6 \max _{j, l}\left\{\sum_r \gamma_{j, r, \boldsymbol{\mu}_l}^{(t)} \right\},
$$
for all $\boldsymbol{\xi} \in \mathbb{R}^d$.
\label{appXOR:lem5.8:in kou}\end{lemma}
\textit{Proof of Lemma \ref{appXOR:lem5.8:in kou}.} See Lemma 5.8 in \citet{kou2023benign} or Theorem 3.2 in \citet{meng2023benign} for a proof, where we utilize a large enough $\hat{C}_2$ in the condition given in the second bullet point ($ n_{s, {l^{\prime}}} \leq \dfrac{\hat{C}_2 \sigma_p^4 d}{\|\boldsymbol{\mu}_{l^{\prime}}\|_2^4 }$) to control the norm of $\mathbf{v}_l$.
\begin{lemma}
(Proposition 2.1 in \citet{devroye2018total}). The TV distance between $\mathcal{N}\left(0, \sigma_p^2 \mathbf{I}_d\right)$ and $\mathcal{N}\left(\mathbf{v}_l, \sigma_p^2 \mathbf{I}_d\right)$ is smaller than $\|\mathbf{v}_l\|_2 / 2 \sigma_p$.
\label{appXOR:leme.2:in devroye}\end{lemma}
\textit{Proof of Lemma \ref{appXOR:leme.2:in devroye}.} See Proposition 2.1 in \citet{devroye2018total} for a proof. \par
Now we take a look at $L_{\mathcal{D}^*}^{0-1}\left(\mathbf{W}^{(t)}\right)$, by (\ref{eqXOR:f.22}) we have:
\begin{equation}
\begin{aligned}
L_{\mathcal{D}^*}^{0-1}\left(\mathbf{W}^{(t)}\right) &=\tau^*_1 \cdot \mathbb{P}_{(\mathbf{x}, y) \sim \mathcal{D}_{\boldsymbol{\mu}_1}^*}\left[y \cdot f(\mathbf{W}, \mathbf{x}) < 0\right] + \tau^*_2 \cdot \mathbb{P}_{(\mathbf{x}, y) \sim \mathcal{D}_{\boldsymbol{\mu}_{2}}} \left[y \cdot f(\mathbf{W}, \mathbf{x}) < 0 \right] \\
& \geq \tau^*_{l^\prime} \cdot \mathbb{P}_{(\mathbf{x}, y) \sim \mathcal{D}_{\boldsymbol{\mu}_{l^{\prime}}}^*}[y \cdot f(\mathbf{W}, \mathbf{x}) < 0] \\
& \geq 0.5 \tau^*_{l^\prime} \cdot \mathbb{P}_{(\mathbf{x}, y) \sim \mathcal{D}_{\boldsymbol{\mu}_{l^{\prime}}}^*}\Big(\left|\sum_r \sigma\left(\left\langle\mathbf{w}_{1, r}^{(t)}, \boldsymbol{\xi}\right\rangle\right)-\sum_r \sigma\left(\left\langle\mathbf{w}_{-1, r}^{(t)}, \boldsymbol{\xi}\right\rangle\right)\right| \\
& \phantom{\geq 0.5 \tau^*_{l^\prime} \cdot \mathbb{P}_{(\mathbf{x}, y) \sim \mathcal{D}_{\boldsymbol{\mu}_{l^{\prime}}}^*}\Big(}\geq \hat{C}_6 \max \left\{\sum_r \gamma_{1, r, \boldsymbol{\mu}_{l^{\prime}}}^{(t)}, \sum_r \gamma_{-1, r, \boldsymbol{\mu}_{l^{\prime}}}^{(t)}\right\}\Big) \\
& = 0.5 \tau^*_{l^\prime} \cdot P(\Omega_{\boldsymbol{\xi}}),
\end{aligned}
\label{eqXOR:E.32}\end{equation}
where $\Omega_{\boldsymbol{\xi}}:=\left\{\boldsymbol{\xi}|| g(\boldsymbol{\xi}) \mid \geq \hat{C}_6 \max \left\{\sum_r \gamma_{1, r, \boldsymbol{\mu}_{l^{\prime}}}^{(t)}, \sum_r \gamma_{-1, r, \boldsymbol{\mu}_{l^{\prime}}}^{(t)}\right\}\right\}$. The last inequality holds since we can always have a corresponding $y$ to make a wrong prediction if given $\boldsymbol{\xi}$, the $\left|\sum_r \sigma\left(\left\langle\mathbf{w}_{1, r}^{(t)}, \boldsymbol{\xi}\right\rangle\right)-\sum_r \sigma\left(\left\langle\mathbf{w}_{-1, r}^{(t)}, \boldsymbol{\xi}\right\rangle\right)\right|$ is large enough.

Next, we seek a lower bound of $P(\Omega_{\boldsymbol{\xi}})$. By Lemma \ref{appXOR:lem5.8:in kou}, we have that $\sum_j[g(j \boldsymbol{\xi}+\mathbf{v}_l)-g(j \boldsymbol{\xi})] \geq$ $4 \hat{C}_6 \max _{j, l}\left\{\sum_r \gamma_{j, r, \boldsymbol{\mu}_l}^{(t)}\right\}$. Then by pigeon's hole principle, there must exist one of the $\boldsymbol{\xi}, \boldsymbol{\xi}+\mathbf{v}_l$, $-\boldsymbol{\xi},-\boldsymbol{\xi}+\mathbf{v}_l$ belongs $\Omega_{\boldsymbol{\xi}}$. So we have proved that $\Omega_{\boldsymbol{\xi}} \cup-\Omega_{\boldsymbol{\xi}} \cup \Omega_{\boldsymbol{\xi}}-\{\mathbf{v}_l\} \cup-\Omega_{\boldsymbol{\xi}}-\{\mathbf{v}_l\}=\mathbb{R}^d$. Therefore at least one of $P(\Omega_{\boldsymbol{\xi}}), P(-\Omega_{\boldsymbol{\xi}}), P(\Omega_{\boldsymbol{\xi}}-\{\mathbf{v}_l\}), P(\Omega_{\boldsymbol{\xi}}-\{\mathbf{v}_l\}), P(-\Omega_{\boldsymbol{\xi}}-\{\mathbf{v}_l\})$ is greater than 0.25. By  the definition of TV distance, we have:
$$
\begin{aligned}
|P(\Omega_{\boldsymbol{\xi}})-P(\Omega_{\boldsymbol{\xi}}-\mathbf{v}_l)| & =\left|\mathbb{P}_{\boldsymbol{\xi} \sim \mathcal{N}\left(0, \sigma_p^2 \mathbf{I}_d\right)}(\boldsymbol{\xi} \in \Omega_{\boldsymbol{\xi}})-\mathbb{P}_{\boldsymbol{\xi} \sim \mathcal{N}\left(\mathbf{v}_l, \sigma_p^2 \mathbf{I}_d\right)}(\boldsymbol{\xi} \in \Omega_{\boldsymbol{\xi}})\right| \\
& \leq \operatorname{TV}\left(\mathcal{N}\left(0, \sigma_p^2 \mathbf{I}_d\right), \mathcal{N}\left(\mathbf{v}_l, \sigma_p^2 \mathbf{I}_d\right)\right) \\
& \leq \frac{\|\mathbf{v}_l\|_2}{2 \sigma_p} \\
& \leq 0.02.
\end{aligned}
$$
Also, notice that $P(-\Omega_{\boldsymbol{\xi}})=P(\Omega_{\boldsymbol{\xi}})$, we have $4P(\Omega_{\boldsymbol{\xi}}) \geq 1 - 2 \cdot 0.02$. Thus $ L_{\mathcal{D}^*}^{0-1}\left(\mathbf{W}^{(t)}\right) \geq 0.5 \tau^*_{l^\prime} \cdot 0.24 = 0.12 \cdot \tau^*_{l^\prime}$. The proofs of Lemma \ref{appXOR:lemma for thm} complete.
\par

Similar to the proof process in Appendix \ref{App:Label Complexity-based Test Error Analysis}, our main focus is to verify whether the NAL algorithms satisfy the condition stated in the first bullet point of Lemma \ref{appXOR:lemma for thm}. Conversely, it is highly probable that Random Sampling satisfies the condition stated in the second bullet point. The following proposition validates this intuition.
\begin{proposition}
    When Lemma \ref{appXOR:Lem:order_pool_app} holds, and the sampling size of algorithm satisfies $ \dfrac{\hat{C}_1 \sigma_p^4 d}{\|\boldsymbol{\mu}_2 \|_2^4} - \dfrac{p n_0}{2} \leq n^*=\Theta(\widetilde{n} - n_0) \leq \widetilde{n} - n_0$, we have the following:
    \begin{itemize}
        \item The number of data with strong feature patch $n_{s,1}$ satisfies $ n_{s, 1} \geq \dfrac{\hat{C}_1 \sigma_p^4 d}{\|\boldsymbol{\mu}_1\|_2^4}, \forall s \in \{0, 1\}$.
        \item The number of data with weak feature patch $n_{s,2}$ before querying and after \textbf{Random Sampling} satisfies $n_{s, 2} \leq \dfrac{\hat{C}_2 \sigma_p^4 d}{\|\boldsymbol{\mu}_2\|_2^4 }, \forall s \in \{0, 1\}$.
       \item  The total number of data with weak feature patch $n_{1,2}$ after \textbf{Uncertainty Sampling} and \textbf{Diversity Sampling} satisfies $ n_{1, 2} \geq \dfrac{\hat{C}_1 \sigma_p^4 d}{\|\boldsymbol{\mu}_2\|_2^4} $ .
    \end{itemize}
For the sake of coherence, here $\hat{C}_1$ and $\hat{C}_2$ are some constants shared with Theorem \ref{theory XOR}.
\label{appXOR:Prop:num_n12}\end{proposition}
\textit{Proof of Proposition \ref{appXOR:Prop:num_n12}.} 
According to the conditions stated in Definition \ref{def_XOR}, we have $(1-\dfrac{3}{2}p)n_{0} \geq \frac{\hat{C}_1 \sigma_p^4 d}{\|\boldsymbol{\mu}_1\|_2^4}$ for a large constant $\hat{C}_1$. By substituting the results of $n_{p}$ for $n_0$ from Lemma \ref{lem:rho n}, as well as the definition of $n_{s, l}$, we obtain the following:
\[
n_{1, 1} \geq n_{0, 1} \geq (1-\dfrac{3}{2}p)n_{0} \geq \dfrac{\hat{C}_1 \sigma_p^4 d}{\|\boldsymbol{\mu}_1\|_2^4}.
\]
For the second bullet, by Lemma \ref{lem:rho n}, Lemma \ref{appXOR:Lem:order_pool_app} and conditions $n^* \geq \dfrac{\hat{C}_1 \sigma_p^4 d}{\|\boldsymbol{\mu}_2 \|_2^4} - \dfrac{p n_0}{2} $, we have:
\[
n_{1,2} \geq \dfrac{p n_0}{2} + n^* \geq \dfrac{\hat{C}_1 \sigma_p^4 d}{\|\boldsymbol{\mu}_2\|_2^4} 
\]
Furthermore, by using Lemma \ref{lem:rho n} and the condition $\widetilde{n} \leq \dfrac{2\hat{C}_2 \sigma_p^4 d}{3p \|\boldsymbol{\mu}_2 \|_2^4}$, the third bullet point is satisfied straightforwardly. 

Based on the results of Lemma \ref{appXOR:lemma for thm} and Proposition \ref{appXOR:Prop:num_n12}, the conclusions of Proposition \ref{prop_XOR_ini} and Theorem \ref{theory XOR} follow directly.

\section{Attribution of Lion Images}
In Figure \ref{fig:lions}, a collection of various lion images found on Google is presented. Due to the challenge of accurately determining the copyright attribution of these images, specific acknowledgments to individual websites or sources cannot be provided here. However, we express our gratitude to all creators, and sincerely hope that they do not find any offense in the use of their work for illustrative purposes in our paper.


\end{document}